\documentclass[conference]{IEEEtran}
\IEEEoverridecommandlockouts

\usepackage{times}
\usepackage{enumitem}
\usepackage[hidelinks]{hyperref}

\usepackage{amsmath}
\usepackage{graphicx}
\usepackage{booktabs}

\usepackage{caption}
\usepackage{subcaption}
\usepackage[dvipsnames]{xcolor}
\usepackage{siunitx}
\usepackage{adjustbox}

\usepackage{listings}
\usepackage{float}
\usepackage{placeins}
\usepackage{amssymb}
\usepackage{gensymb}
\usepackage{amsthm}
\usepackage{makecell}

\newtheorem{theorem}{Theorem}
\newtheorem{lemma}{Lemma}
\newtheorem{definition}{Definition}

\definecolor{CodeBlue}{rgb}{0.25,0.5,0.85}
\definecolor{CodeGreen}{rgb}{0.15,0.6,0.25}
\definecolor{CodePurple}{rgb}{0.5,0,0.35}
\definecolor{CodeString}{rgb}{0.6,0.1,0.1}
\definecolor{CodeGray}{rgb}{0.5,0.5,0.5}
\definecolor{CodeBackground}{rgb}{0.98,0.98,0.98}

\lstdefinestyle{HighImpactPython}{
    language=Python,
    backgroundcolor=\color{CodeBackground},
    commentstyle=\itshape\color{CodeGreen},
    keywordstyle=\bfseries\color{CodeBlue},
    numberstyle=\tiny\color{CodeGray},
    stringstyle=\color{CodeString},
    basicstyle=\ttfamily\footnotesize,
    breakatwhitespace=false,
    breaklines=true,
    captionpos=b,
    keepspaces=true,
    numbers=left,
    numbersep=8pt,
    showspaces=false,
    showstringspaces=false,
    showtabs=false,
    tabsize=4,
    frame=single,
    rulecolor=\color{CodeGray!50},
    emph={MarsHiRISE, DataLoader, HiRISEGeoSampler, pathlib, logger},
    emphstyle=\color{CodePurple}\bfseries
}

\setlist[itemize]{leftmargin=1.15em,itemsep=0.05em,topsep=0em}

\def\BibTeX{{\rm B\kern-.05em{\sc i\kern-.025em b}\kern-.08em
    T\kern-.1667em\lower.7ex\hbox{E}\kern-.125emX}}
\begin{document}

\title{MarsFM: Shading-Regularized Flow Matching for Martian Relief Estimation
}

\author{\IEEEauthorblockN{Marius~F.~R.~Juston}
\IEEEauthorblockA{\textit{Department of Industrial and Enterprise Systems Engineering} \\
\textit{The Grainger College of Engineering} \\
\textit{University of Illinois Urbana-Champaign, Urbana}\\
IL 61801-3080, USA \\
mjuston2@illinois.edu}
}

\maketitle
\begin{abstract}
We present \textbf{MarsFM}, an image-conditioned latent flow-matching model for local Martian relief estimation from single-band HiRISE RED orthoimagery. The method combines a pretrained generative prior with stereo-derived geometric supervision and a differentiable Lunar--Lambert shading objective. Relief, normal, gradient, curvature, and ordinal terms constrain complementary aspects of terrain structure, while a positive-affine-invariant image comparison constrains rendered appearance. An evaluation comprising 2024 gathered patch records per integration-step count yields mean affine-aligned RMSE between 0.0935 and 0.0957 in normalized signed-log relief space for one to twenty Euler steps. These scores measure agreement with VAE-reconstructed references on positive-reference support. Their narrow range supports low-step inference under this protocol. Spatial, differential, and spectral diagnostics show broad terrain correspondence alongside smoothing, amplitude compression, and boundary mismatch. MarsFM provides a framework for combining learned terrain priors with image-based constraints; establishing improved physical terrain resolution requires matched baselines and independent high-resolution reference data. Data: \url{https://huggingface.co/datasets/SuperComputer/mars_hirise_dtm_processed-6aa9b66ba461e07f}; code: \url{https://github.com/Marius-Juston/MarsRecon}.

\end{abstract}
\begin{IEEEkeywords}
Digital Elevation Models (DEM), Flow Matching, HiRISE, Mars Reconnaissance Orbiter, Photoclinometry.
\end{IEEEkeywords}
\section{Introduction}

Quantitative analysis of the Martian surface---from the migration of transverse aeolian ridges (TARs)~\cite{Athmania2015AutomatedData} to slope activity, polar deformation, and candidate mud-volcanic landforms---requires topography at scales appropriate to the features under study. The Mars Orbiter Laser Altimeter (MOLA)~\cite{Smith2001MarsMars} provides global elevation context, but its global grids do not resolve many local landforms visible in orbital imagery. Higher-resolution topography comes from orbital photogrammetry~\cite{Jiang2017FusionImagery}, including the High-Resolution Imaging Science Experiment (HiRISE) and Context Camera (CTX) onboard the Mars Reconnaissance Orbiter (MRO)~\cite{MROScience,Malin2007ContextOrbiter,McEwen2007MarsHiRISE}.

HiRISE acquires RED, Blue--Green, and Near-Infrared pushbroom imagery, with RED products commonly posted at 25--50\,cm/pixel~\cite{McEwen2007MarsHiRISE}. Suitable off-nadir acquisitions permit stereo reconstruction through pipelines such as the NASA Ames Stereo Pipeline and SOCET Set~\cite{Beyer2018TheData,Hepburn2019CreatingPipeline}. Reconstruction requires overlapping views, suitable acquisition geometry, and reliable image correspondence. A single-image estimator offers a complementary source of local terrain information where a usable stereo terrain product is unavailable. Appendices~\ref{sec:dataset} and~\ref{sec:dtm_dataset} characterize the imagery and terrain catalogs separately.

Machine learning offers a complementary route, although single-image elevation prediction remains an ill-posed inverse problem~\cite{Chung2024DiffusionProblems,Rout2023SolvingModels}. Mars3DNet and MADNet demonstrate learned planetary reconstruction~\cite{Chen2021Mars3DNet:Images,Tao2021MADNetLearning}. Generative approaches repurpose image priors for depth prediction~\cite{Saxena2023MonocularModels,Ke2025Marigold:Analysis}, and flow matching enables short integration paths at inference~\cite{Lipman2022FlowModeling,Liu2022FlowFlow,Gui2024DepthFM:Matching}. Diffusion-based Martian terrain super-resolution has also been investigated~\cite{Cao2026MartianDiffusion}. Our focus is the combination of a few-step, image-conditioned flow model with terrain supervision and differentiable shading regularization; improved sharpness, accuracy, and uncertainty calibration must each be evaluated separately.

A central difficulty is the difference between image sampling and terrain posting: a representative HiRISE pair provides a $\sim0.25$\,m image grid and a $\sim1$\,m DTM grid~\cite{McEwen2007MarsHiRISE,Kirk2009UltrahighSites}. Shading can supply an additional constraint, although albedo and shadows also generate image structure. Our contributions are: \textbf{(i)} an adaptation of DepthFM~\cite{Gui2024DepthFM:Matching} to local Martian relief, using a frozen Stable-Diffusion 1.5 VAE~\cite{Rombach2021High-ResolutionModels}; \textbf{(ii)} a training objective coupling stereo-derived relief and differential geometry with Lunar--Lambert shading; and \textbf{(iii)} a preprocessing and diagnostic framework covering patch sampling, relief normalization, illumination fitting, and exposure invariance. Evaluation across six integration-step counts characterizes predictive agreement and inference sensitivity. The appendices develop the mathematics and examine the method's spatial, geometric, and photometric behavior.

\section{Method}
\label{sec:method}

\subsection{Latent Flow Matching}
\label{sec:method_fm}

Conditional Flow Matching (CFM)~\cite{Lipman2022FlowModeling,Liu2022FlowFlow,Tong2023ImprovingTransport} learns a velocity field that transports a source representation toward a target along prescribed training paths. Following DepthFM~\cite{Gui2024DepthFM:Matching}, a frozen Stable-Diffusion 1.5 VAE~\cite{Rombach2021High-ResolutionModels} encodes the image and normalized relief. HiRISE RED is replicated across three image channels. Let $\mathbf z_{\mathrm{img}}=\mathcal E(\mathbf I)$ and $\mathbf z_1=\mathcal E(\mathbf d)$. The source $\mathbf z_0$ is the image latent, optionally perturbed using the checkpoint's noising schedule. Training samples the path
\begin{equation}
\mathbf z_t=(1-t)\mathbf z_0+t\mathbf z_1+\sigma_{\min}\boldsymbol\epsilon,
\qquad \boldsymbol\epsilon\sim\mathcal N(0,I),
\end{equation}
and supervises the velocity with $\mathbf z_1-\mathbf z_0$. Time is sampled from a logit-normal distribution and $\sigma_{\min}=10^{-4}$. The UNet receives the current state and clean image latent by channel concatenation ($4+4=8$ channels), with frozen null-text conditioning. At inference, Euler integration advances the learned velocity field and the VAE decodes its endpoint into relief. Appendix~\ref{sec:fm_loss} gives the source perturbation and endpoint formulation. The straight conditional training path does not require learned inference trajectories to be straight.

\subsection{Multi-Objective Loss}
\label{sec:method_loss}

The stereo reference (denoted GT for consistency with the figures) has coarser posting than the source image. Matching an interpolated reference alone supplies no independent evidence for finer-scale geometry. We combine three classes of supervision, with full formulations and diagnostic panels in Appendix~\ref{sec:losses}. Their intended roles motivate the design; the contribution of each term requires a controlled ablation.

\textbf{Relief and differential-geometry supervision.} Four terms constrain the reference geometry: a Huber relief loss $\mathcal{L}_{\text{huber}}$~\cite{Huber1964RobustParameter} with $\delta=0.1$; a cosine loss $\mathcal{L}_{\text{norm}}$ on Sobel-derived normals~\cite{Fan2021Three-Filters-to-Normal:Estimator}; a multi-scale $\ell_1$ gradient loss $\mathcal{L}_{\text{grad}}$~\cite{Li2018MegaDepth:Photos} at $s\in\{1,2,4\}$; and a five-point Laplacian loss $\mathcal{L}_{\text{lap}}$. Huber and latent velocity matching constrain relief values, while the differential terms constrain orientation and spatial variation. Constant offsets vanish under differentiation; arbitrary rescaling does not.

\textbf{Robust and ordinal supervision.} Stereo references can contain registration errors, shadow-related gaps, and triangulation seams~\cite{Beyer2018TheData,Kirk2009UltrahighSites}. Alongside Huber regression, we use an ordinal loss $\mathcal{L}_{\text{ord}}$ inspired by relative-depth supervision~\cite{ChenSingle-ImageWild,Xian2018MonocularSupervision}. The reported configuration samples $10^4$ confidence-weighted pixel pairs per image and excludes reference differences below $\tau=0.02$. This threshold selects pairs; it is not a positive prediction margin. Ordering can tolerate errors that preserve the sign of a height difference, but it remains vulnerable to artifacts that reverse that sign.

\textbf{Image-based shading supervision.} The photoclinometric loss $\mathcal{L}_{\text{photo}}$ compares a render of the predicted relief with the co-registered orthoimage on the processed image grid. A Lunar--Lambert surrogate~\cite{McEwen1991PhotometricApplications,mcewen1996reflectance,Kirk2003High-resolutionImages} blends Lambertian and Lommel--Seeliger responses using a learned scalar. This compact model is motivated by planetary photometry~\cite{Hapke2012TheorySpectroscopy}, but does not model spatially varying albedo, cast shadows, or the complete bidirectional reflectance function. Its contribution is an additional image-derived constraint, whose ability to improve physical terrain resolution remains to be established.


The total objective is $\mathcal{L}_{\text{total}}=\sum_k\lambda_k\mathcal{L}_k$. All nonzero losses are active from the first optimization step. Learning-rate warmup lasts 1000 steps and does not delay their activation. The focal-frequency term has zero weight in this experiment; its formulation is provided for completeness in Appendix~\ref{sec:ffl_loss}.

\section{Implementation}
\label{sec:impl}

\paragraph{Dataset} A TorchGeo-compatible pipeline reads PDS metadata, applies radiometric scale and offset, and indexes imagery and terrain in a Mars IAU 2000 geographic CRS. The reported corpus characterization comprises approximately 400 stereo-derived products within $-120^{\circ}\le\lambda\le150^{\circ}$ and $|\varphi|\le30^{\circ}$, with approximately $1.6\times10^9$ valid pixel pairs. Training, validation, and testing use an 80/10/10 geographic split along longitude. Patches must overlap at least 50\% of the product's convex footprint approximation; observed-pixel validity is tracked separately. The experimental preprocessing removes a fitted plane, applies a signed logarithm with a residual-magnitude 98th-percentile scale, and clips to $[-1,1]$. GMRF infilling supplies finite VAE inputs while the observation mask distinguishes measured and imputed values. Appendix~\ref{sec:dtm_dataset} provides the full preprocessing and scale-selection analysis.

\paragraph{Training} The UNet velocity head is initialized from the public DepthFM checkpoint~\cite{Gui2024DepthFM:Matching} and fine-tuned for 10K steps with AdamW (learning rate $2\times10^{-4}$, weight decay 0.01, cosine decay, and 1000 warmup steps), bf16 precision, gradient checkpointing, and EMA decay 0.999. Training uses two NVIDIA RTX 6000 Ada GPUs, four samples per device, and four accumulation steps, giving an effective batch of 32. The run takes approximately 27 hours including validation. Each $0.018^{\circ}$ geographic window is resized to $512\times512$; its equatorial width is approximately 1.067\,km, corresponding to 2.08\,m/pixel, with longitude spacing varying by latitude. Loss weights are $\lambda_{\mathrm{FM}}=\lambda_{\mathrm{huber}}=\lambda_{\mathrm{norm}}=\lambda_{\mathrm{lap}}=\lambda_{\mathrm{ord}}=1$, $\lambda_{\mathrm{grad}}=\lambda_{\mathrm{photo}}=2$, and $\lambda_{\mathrm{FFL}}=0$.

\section{Experimental Results}
\label{sec:results}

\paragraph{Protocol} The primary evaluation uses checkpoint 9476, selected by validation photo-consistency, and 2024 gathered patch records at each of $\{1,2,4,8,10,20\}$ Euler steps. Each prediction uses one realization. Predictions and references are the first channels of the decoded predicted and encoded-target latents, respectively; the reference therefore includes VAE reconstruction. Scores follow unconstrained least-squares scale/shift alignment in normalized signed-log space~\cite{Eigen2014DepthNetwork,Ke2025Marigold:Analysis}. The evaluator selects finite pixels with positive reference values and does not receive the observation-confidence mask. It can therefore include positive infilled values and excludes nonpositive relief. These scores characterize agreement on that support, rather than error over all observed terrain. The 2024 entries are distributed evaluation records; a site-level manifest is needed to verify unique geographic support. A separate 1012-record, 16-step evaluation against the normalized input reference is retained in Table~\ref{tab:metrics_legacy}.
\begin{table*}[t]
\centering
\caption{Primary evaluation: mean $\pm$ sample standard deviation over 2024 gathered records per Euler-step count, checkpoint 9476. Predictions are aligned to VAE-reconstructed references in normalized signed-log space on positive-reference support. Angular and curvature scores describe the represented grid, not physical terrain geometry. Ratio and log scores are legacy diagnostics with the restricted domains described in the text. Photo-consistency is not computed in this sweep.}
\label{tab:metrics_main}
\scriptsize
\setlength{\tabcolsep}{3pt}
\begin{tabular}{@{}lrrrrrr@{}}
\toprule
Euler steps & 1 & 2 & 4 & 8 & 10 & 20 \\
\midrule
RMSE & $0.0935 \pm 0.0819$ & $0.0944 \pm 0.0834$ & $0.0938 \pm 0.0819$ & $0.0948 \pm 0.0834$ & $0.0957 \pm 0.0849$ & $0.0951 \pm 0.0837$ \\
Normal error ($^{\circ}$) & $0.363 \pm 0.187$ & $0.364 \pm 0.187$ & $0.366 \pm 0.189$ & $0.368 \pm 0.191$ & $0.368 \pm 0.191$ & $0.369 \pm 0.192$ \\
Slope RMSE ($^{\circ}$) & $0.480 \pm 0.293$ & $0.485 \pm 0.302$ & $0.482 \pm 0.296$ & $0.483 \pm 0.294$ & $0.477 \pm 0.288$ & $0.478 \pm 0.283$ \\
Curvature RMSE ($\times10^3$) & $2.76 \pm 1.80$ & $2.80 \pm 1.87$ & $2.79 \pm 1.84$ & $2.78 \pm 1.80$ & $2.75 \pm 1.75$ & $2.75 \pm 1.69$ \\
Boundary F-score & $0.072 \pm 0.135$ & $0.083 \pm 0.153$ & $0.085 \pm 0.154$ & $0.079 \pm 0.148$ & $0.086 \pm 0.157$ & $0.085 \pm 0.160$ \\
Topographic MS-SSIM & $0.552 \pm 0.163$ & $0.547 \pm 0.163$ & $0.546 \pm 0.164$ & $0.542 \pm 0.159$ & $0.540 \pm 0.160$ & $0.538 \pm 0.160$ \\
PSD slope ratio & $1.080 \pm 0.118$ & $1.064 \pm 0.107$ & $1.060 \pm 0.103$ & $1.074 \pm 0.119$ & $1.073 \pm 0.120$ & $1.076 \pm 0.124$ \\
Patch SWD & $0.028 \pm 0.061$ & $0.030 \pm 0.065$ & $0.028 \pm 0.061$ & $0.031 \pm 0.071$ & $0.030 \pm 0.066$ & $0.031 \pm 0.065$ \\
Absolute relative error & $7.214 \pm 4.261$ & $7.338 \pm 4.489$ & $7.315 \pm 4.575$ & $7.335 \pm 4.524$ & $7.393 \pm 4.639$ & $7.418 \pm 4.536$ \\
Squared relative error & $0.837 \pm 0.950$ & $0.891 \pm 1.401$ & $0.846 \pm 0.974$ & $0.864 \pm 0.973$ & $0.891 \pm 1.061$ & $0.887 \pm 1.038$ \\
Log-RMSE & $1.163 \pm 0.236$ & $1.169 \pm 0.240$ & $1.168 \pm 0.240$ & $1.171 \pm 0.234$ & $1.174 \pm 0.235$ & $1.176 \pm 0.240$ \\
Scale-invariant log score & $108.78 \pm 18.30$ & $109.19 \pm 18.54$ & $109.16 \pm 18.50$ & $109.40 \pm 18.04$ & $109.67 \pm 18.05$ & $109.87 \pm 18.62$ \\
$\delta_1$ (\%) & $23.56 \pm 9.08$ & $23.18 \pm 9.16$ & $23.39 \pm 9.26$ & $23.12 \pm 8.58$ & $22.92 \pm 8.62$ & $22.90 \pm 8.77$ \\
$\delta_2$ (\%) & $44.04 \pm 12.94$ & $43.71 \pm 13.01$ & $43.77 \pm 12.95$ & $43.71 \pm 12.65$ & $43.33 \pm 12.59$ & $43.38 \pm 12.74$ \\
$\delta_3$ (\%) & $59.42 \pm 13.11$ & $59.25 \pm 13.23$ & $59.29 \pm 13.19$ & $59.13 \pm 12.81$ & $58.94 \pm 12.93$ & $58.94 \pm 13.18$ \\
\bottomrule
\end{tabular}
\end{table*}

\paragraph{Quantitative results} Table~\ref{tab:metrics_main} reports all computed diagnostics across the six step counts. Mean RMSE spans 0.0935--0.0957, indicating limited sensitivity to integration count under this protocol. At one step, topographic MS-SSIM is 0.552 and the boundary F-score is 0.072: broad structural agreement coexists with weak edge agreement. Low grid-coordinate angular errors do not establish accurate physical slopes, because the evaluator uses normalized relief and unit-pixel differences. Spectral slopes and patch statistics complement these spatial diagnostics but cannot independently establish feature localization or resolved terrain bandwidth. The sweep does not compute photo-consistency; its stored zero is a default value and is not a measured result.

\paragraph{Metric interpretation} Errors are dimensionless unless explicitly identified as grid-coordinate angles. Ratio inputs are clipped below $10^{-6}$, and log-RMSE uses pixels where both fields are positive. These conventions are poorly suited to general signed relief and are reported as implementation diagnostics, not as conventional absolute-depth accuracy. Affine alignment can admit a negative scale and removes calibration information. The supporting 16-step export (Table~\ref{tab:metrics_legacy}) has RMSE $0.098\pm0.085$, topographic MS-SSIM 0.580, boundary F-score 0.075, and photo-consistency 0.220 under its different reference and aggregation protocol.

\paragraph{Euler-step sensitivity.} Table~\ref{tab:steps_ablation} summarizes the primary evaluation's compute--accuracy tradeoff. Across one to twenty steps, the rounded RMSE range is 0.0022, approximately 2.35\% of the lowest value. One step attains the lowest observed RMSE. This supports low-step inference for the evaluated model, but does not establish trajectory straightness or statistical equivalence. Fresh source noise is drawn at each step count, so the observed differences combine integration effects with realization variability. A paired-seed experiment would be needed to isolate the integration effect.
\begin{table}[t]
\centering
\caption{Compact summary of the primary 2024-record evaluation in Table~\ref{tab:metrics_main}. One step gives the lowest observed RMSE. Fresh noise is sampled at each step count; these values do not constitute a paired integration-error experiment.}
\label{tab:steps_ablation}
\small
\setlength{\tabcolsep}{2pt} 
\begin{tabular}{@{}ccccccc@{}} 
\toprule
Steps        & 1     & 2     & 4     & 8     & 10    & 20    \\
\midrule
RMSE         & 0.0935 & 0.0944 & 0.0938 & 0.0948 & 0.0957 & 0.0951 \\
$\delta_1$\,(\%) & 23.56  & 23.18  & 23.39  & 23.12  & 22.92  & 22.90 \\
\bottomrule
\end{tabular}
\end{table}

\paragraph{Qualitative results} Figure~\ref{fig:qualitative_main} compares three archived scenes using imagery, predicted relief, reference relief, and associated diagnostics. The predictions capture portions of the broad terrain structure while retaining visible amplitude, spatial, and shading discrepancies. These examples complement the aggregate metrics by showing where their agreement breaks down. An improvement attributable specifically to photoclinometric supervision would require the same scenes from a matched run without that loss.
\begin{figure}[htbp]
\centering
\includegraphics[width=\linewidth,height=0.80\textheight,keepaspectratio]{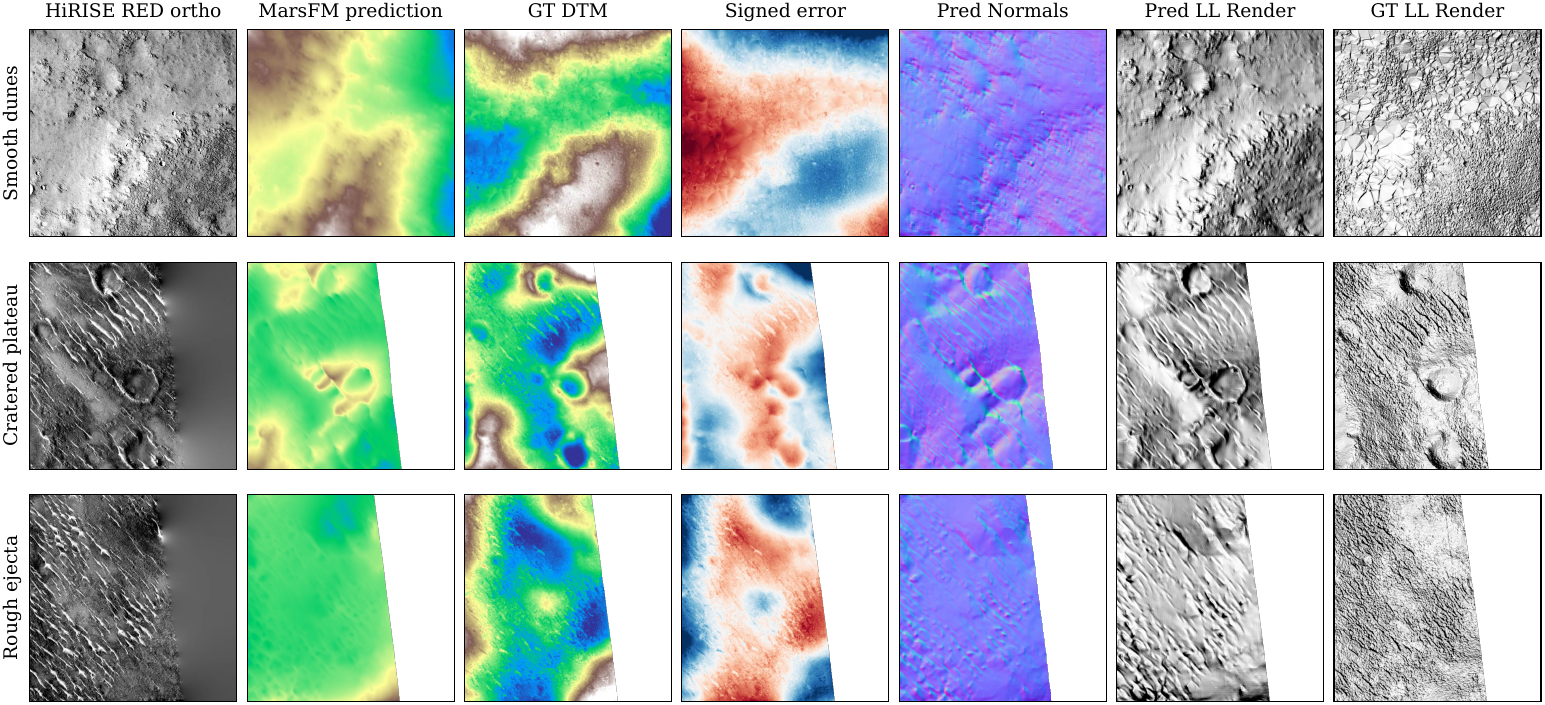}
\caption{\textbf{Qualitative reconstruction diagnostics for three scenes.} The complete archived panel is retained to compare imagery, predicted and reference relief, and their diagnostic responses. Broad shape correspondence coexists with smoothing and local discrepancies; visual texture alone does not verify additional terrain resolution.}
\label{fig:qualitative_main}
\end{figure}

\paragraph{Limitations} The experiments characterize a single trained system and do not isolate the contribution of shading or establish superiority to a matched baseline. Evaluation uses VAE-reconstructed normalized references, positive-reference support, and reference-derived affine alignment; it therefore does not establish physical height accuracy or performance over all observed signed relief. Geographic independence requires a product-level split manifest. Plane detrending removes regional elevation, and clipping loses extreme residuals. Reference-conditioned illumination and approximate reflectance can confound geometry with albedo and shadows; the implemented illumination update also limits claims of robust fitting (Appendix~\ref{sec:sun_irls}). Only RED imagery is used. Independent high-resolution references and controlled comparisons are required to establish improved terrain resolution and uncertainty calibration. The catalog and rationale-text analyses document possible extensions, rather than additional model inputs.

\section{Conclusion}
\label{sec:conclusion}

MarsFM combines an image-conditioned flow model with stereo-derived geometric supervision and differentiable shading for local Martian relief estimation. Its central contribution is a formulation in which learned terrain priors, differential structure, and rendered image agreement can be optimized together. Evaluation on 2024 gathered records per step count yields normalized RMSE of 0.0935--0.0957 across one to twenty Euler steps, supporting low-step prediction under the stated protocol. The full diagnostics show broad terrain correspondence together with smoothing, amplitude compression, and boundary and rendering errors. The next empirical requirement is to establish whether shading improves reconstruction relative to matched baselines under physically calibrated evaluation. BG/IR conditioning, CTX transfer, and consistency across overlapping patches offer further directions.

\bibliographystyle{ieeetr}
\bibliography{references.bib}

\onecolumn
\appendices
\FloatBarrier
\section{RDR Dataset} \label{sec:dataset}

The data pipeline supports the broader HiRISE RDR catalog as well as the paired DTM experiments. This appendix preserves the RDR infrastructure, multispectral statistics, metadata analysis, and implementation examples because they document the corpus and its selection biases. The trained model uses RED imagery; BG/IR channels and rationale text are available for future extensions. A Mars-specific \emph{coordinate reference system} (CRS), product-label handling, and patch sampler adapt the Earth-oriented TorchGeo workflow to planetary data.

We use a Mars IAU 2000 geographic CRS as a common indexing system and retain each raster's product-defined projection. The reference ellipsoid parameters documented for Mars mapping~\cite{DSMAP} are:
\begin{align}
    \text{Equatorial radius (a)} = 3,396,190 m, \quad \text{Polar radius (b)} = 3,376,200 m
\end{align}
For a spherical product projection using the local ellipsoid radius at planetocentric latitude $\varphi$, the geometric radius is:
\begin{align}
R(\varphi) = \frac{ab}{\sqrt{b^2\cos^2\varphi+a^2\sin^2\varphi}}
\end{align}
Product metadata determines the appropriate projection, center latitude, radii, and longitude convention. Many nonpolar products use an equirectangular projection with observation-specific parameters; a single projected CRS should not be imposed on every file. The loader reads those parameters before transforming query bounds to the common geographic CRS. Physical distances and gradients must still be evaluated in metric coordinates, not directly in degrees.

At the equator, $\pi a/180\approx59{,}274.70$\,m per degree. A 0.25\,m pixel therefore corresponds to approximately 237,099 pixels per degree, and a 0.5\,m pixel to 118,549 pixels per degree. The earlier catalog-derived value of 118,502.26 pixels per degree is consistent with an approximately 0.5\,m grid using a slightly different local radius; it does not describe 0.25\,m sampling. Product-specific map scales take precedence over these approximations.

\FloatBarrier
\subsection{Radiometric calibration}

For products that encode calibrated radiance factor as scaled digital numbers (DN), physical I/F is recovered using the scaling factor and offset in that product's label. DN is an encoding whose interpretation depends on the product; applying another observation's coefficients can introduce a radiometric mismatch. The conversion is:
\begin{align}
    \mathrm{I/F} = \mathrm{DN} \times \mathrm{SCALING\_FACTOR} + \mathrm{OFFSET}
\end{align}
Each image is normalized using this radiometric calibration before proceeding through the rest of the processing pipeline.

\FloatBarrier
\subsection{Instrument characteristics and spectral bandpasses}

HiRISE uses a 0.5-meter primary mirror and a pushbroom arrangement of 14 charge-coupled devices (CCDs)~\cite{McEwen2007MarsHiRISE,Sutton2022RevealingModels}. The three spectral channels support complementary geological observations, summarized below.
\begin{table}[htbp]
    \centering
    \caption{Spectral band characteristics and primary geological applications}
    \label{tab:sensor_characteristics}
    \begin{tabular}{@{} l c p{8.5cm} @{}}
        \toprule
        \textbf{Spectral Band} & \textbf{Wavelength Range (nm)} & \textbf{Primary Geological Applications} \\
        \midrule
        Blue--Green (BG) & 400--600 & Detection of frost, ice, and atmospheric scattering phenomena \\[2pt]
        Panchromatic (RED) & 550--850 & High-resolution topographic mapping, structural geomorphology analysis, and digital elevation model (DEM) extraction \\[2pt]
        Near-Infrared (IR) & 800--1000 & Mineralogical differentiation and sub-pixel spectral unmixing (e.g., integration with CRISM datasets) \\
        \bottomrule
    \end{tabular}
\end{table}
After performing the Radiometric calibration for each image in the dataset, we calculate the statistical per-pixel information for the entire dataset, as illustrated in the histograms in Figure \ref{fig:pixel_histograms}. With this statistical information, the three channels are tabulated in Table \ref{tab:dataset_stats_extended}.
\begin{figure}[htbp]
     \centering
     \begin{subfigure}[b]{0.32\textwidth}
         \centering
         \includegraphics[width=\textwidth,height=0.80\textheight,keepaspectratio]{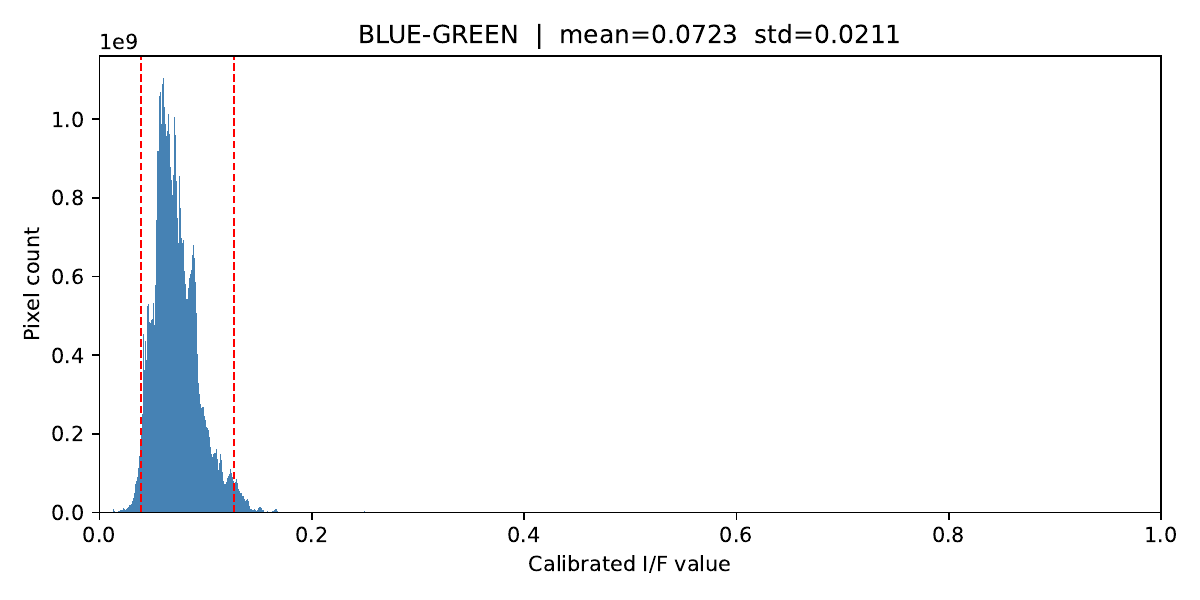}
         \caption{Blue-green spectral band}
         \label{fig:histogram_BLUE-GREEN}
     \end{subfigure}
     \hfill
     \begin{subfigure}[b]{0.32\textwidth}
         \centering
         \includegraphics[width=\textwidth,height=0.80\textheight,keepaspectratio]{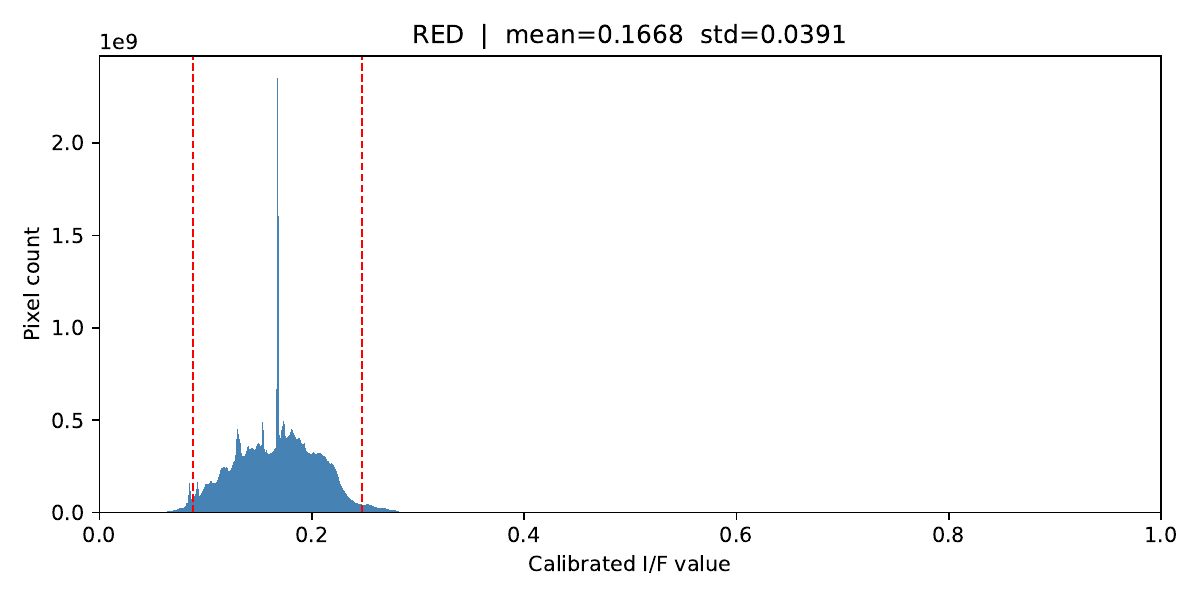}
         \caption{Red spectral band}
         \label{fig:histogram_RED}
     \end{subfigure}
     \hfill
     \begin{subfigure}[b]{0.32\textwidth}
         \centering
         \includegraphics[width=\textwidth,height=0.80\textheight,keepaspectratio]{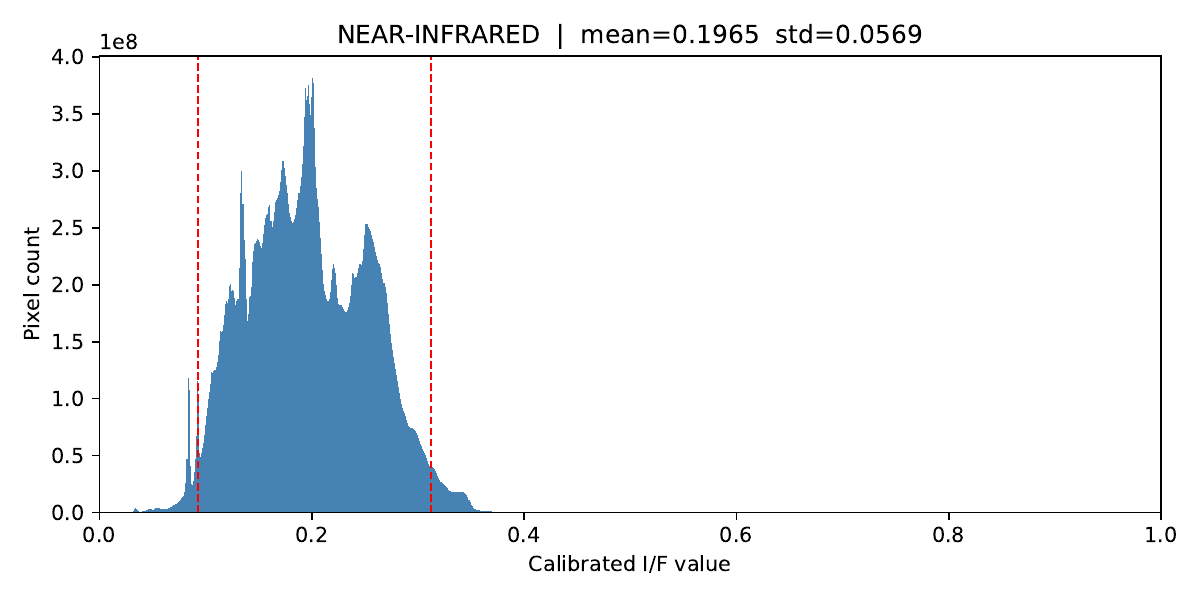}
         \caption{Infrared spectral band}
         \label{fig:histogram_NEAR-INFRARED}
     \end{subfigure}
        \caption{Per-pixel radiometrically calibrated distribution of 149,921 valid patches with a patch size of 0.005 degrees over 422 HiRISE measurement readings. These patches are derived from the HiRISE dataset and subsampled using the target-region keyword "Olympus". The histograms are unnormalized frequency histograms that provide aggregated information on approximately 46 billion pixels per channel. The red lines in the histograms represent the 2nd and 98th percentiles of the data.}
        \label{fig:pixel_histograms}
\end{figure}
\begin{table}[htbp]
\centering
\caption{Summary statistics of the dataset across spectral channels. P2 and P98 represent the 2\% and 98\% percentiles. CV denotes the coefficient of variation ($\sigma/\mu$). DRCR denotes dynamic range compression ratio ($(P_{98}-P_{2})/\mu$).}
\label{tab:dataset_stats_extended}

\begin{tabular}{
l
S[table-format=1.4e2]
S[table-format=1.4]
S[table-format=1.4]
}
\toprule
\textbf{Channel} & {\textbf{$N_{\text{pixels}}$}} & {\textbf{Mean}} & {\textbf{Std}} \\
\midrule
Near-infrared & 4.5949e+10 & 0.1965 & 0.0569 \\
Red & 4.9983e+10 & 0.1668 & 0.0391 \\
Blue-green & 4.6060e+10 & 0.0723 & 0.0211 \\
\bottomrule
\end{tabular}
\vspace{1.5em}

\begin{tabular}{
l
S[table-format=1.4]
S[table-format=1.4]
S[table-format=1.4]
S[table-format=1.4]
S[table-format=1.3]
S[table-format=1.3]
}
\toprule
\textbf{Channel} & {\textbf{Min}} & {\textbf{Max}} & {\textbf{P2}} & {\textbf{P98}} & {\textbf{CV}} & {\textbf{DRCR}} \\
\midrule
Near-infrared & 0.0097 & 0.7700 & 0.0928 & 0.3125 & 0.290 & 1.118 \\
Red & 0.0097 & 0.3600 & 0.0879 & 0.2471 & 0.235 & 0.954 \\
Blue-green & 0.0097 & 0.3184 & 0.0389 & 0.1270 & 0.292 & 1.218 \\
\bottomrule
\end{tabular}
\end{table}
\FloatBarrier
\subsection{Dataset subset}

The archived catalog snapshot contains 76 releases and 198,460 RDR records. These counts refer to catalog entries, which can include multiple products from an observation, rather than necessarily distinct image strips. Figure~\ref{fig:toptographic_map} shows catalog locations and the selected Olympus-related subset. At a reported mean JPEG2000 size of 652.495\,MB per record, downloading one file per entry would require approximately 130\,TB; this is a storage estimate for that snapshot, not a measured unique-observation total. The loader supports selected subsets so that processing need not require a complete mirror.
\begin{figure}[htbp]
    \centering
    \includegraphics[width=1\linewidth,height=0.80\textheight,keepaspectratio]{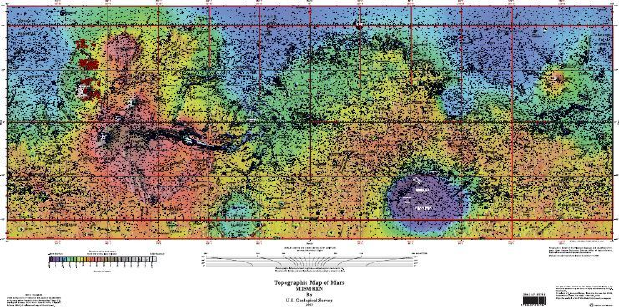}
    \caption{Topographic Map of Mars \cite{Topographic2782} through the view of Mercator Projection with longitude range of $-180^{\circ}$ to $180^{\circ}$ and latitude range of approximately $-57^{\circ}$ to $57^{\circ}$. The map was generated by the U.S. Geological Survey, and on the map were overlaid all the HiRISE sample locations in \textcolor{blue}{blue} within that topographical range, and in \textcolor{red}{red} is the subsampled dataset used as the scope of the project. The \textcolor{red}{sub-sampled} dataset was selected by filtering samples with the rationale description "Olympus", rather than a bounding-box selection approach.}
    \label{fig:toptographic_map}
\end{figure}
Instead of handling the full 198,460 HiRISE instrument measurements, we handle only a small subset. To do this, our pipeline can filter regions of interest in two ways: \textbf{bounding-box} filtering based on specified longitude and latitude, or \textbf{text filtering}.

For each HiRISE instrument measurement, there is an associated metadata file that contains the minimum and maximum latitudes and longitudes, along with other information, and a rationale description text field. This rationale description text field explains why the HiRISE measurement was taken or the name of the region of interest. Such descriptions could be "Becquerel Crater stratigraphy", "Ancient cratered surface northeast of Martz Crater, "Polar escarpment with prominent marker layer", among others. The second type of filter iterates over the rationale description and maintains only the list of entries that contain the target text.

\FloatBarrier
\subsection{Metadata}

As previously stated, each HiRISE instrument comes with a metadata file. This file contains additional useful information, such as the map scale, center location, viewing angle parameters, radiometry calibration parameters, and more. Some of the important parameters are listed and explained in Table \ref{tab:pds_metadata_complete}.
\begin{table}[htbp]
\centering
\tiny
\caption{Comprehensive PDS Metadata and Viewing Parameters for HiRISE Observation}
\label{tab:pds_metadata_complete}
\renewcommand{\arraystretch}{1.2}
\begin{tabular}{@{}llp{8cm}@{}}
\toprule
\textbf{Parameter} & \textbf{Example Value} & \textbf{Scientific Representation / Description} \\ \midrule
\textbf{OBSERVATION\_ID} & TRA\_000873\_2015 & Unique identifier for this orbital capture sequence, documenting a transition-phase observation of Becquerel Crater stratigraphy. \\
\textbf{PRODUCT\_ID} & TRA\_000873\_2015\_COLOR & Specific identifier for this data product, indicating it is the synthesized color composite image. \\
\textbf{OBSERVATION\_START\_TIME} & 2006-10-03T12:51:33.015 & Precise Universal Time Coordinated (UTC) timestamp when the instrument began recording the observation. \\
\textbf{MAP\_PROJECTION\_TYPE} & EQUIRECTANGULAR & Cartographic mapping method applied. Utilizes a spherical approximation (local radius $R = 3393.83$ km) where latitude and longitude grid lines intersect at right angles. \\
\textbf{MAP\_SCALE} & 0.25 meters/pixel & Spatial resolution of the projected map. Each pixel represents a 0.25 m $\times$ 0.25 m area on the Martian surface at the center latitude. \\
\textbf{CENTER\_LATITUDE / LONGITUDE} & 20.000$^\circ$, 180.000$^\circ$ & Planetocentric coordinates defining the center of the map projection. \\
\textbf{INCIDENCE\_ANGLE} & 47.099670$^\circ$ & Solar zenith angle relative to the reference-surface normal at the reported location; this is distinct from incidence on each resolved terrain facet. \\
\textbf{EMISSION\_ANGLE} & 0.282987$^\circ$ & Angle between the reference-surface normal and the instrument line-of-sight at the reported location. A value near 0$^\circ$ indicates a nearly perfect nadir viewing geometry. \\
\textbf{PHASE\_ANGLE} & 47.047614$^\circ$ & Angle formed between the light source (Sun), the target surface, and the observer (MRO). Used to define the bidirectional reflectance distribution function (BRDF). \\
\textbf{LOCAL\_TIME} & 15.40208 & The true solar local time at the target center, expressed in decimal hours (approximately 15:24). Essential for diurnal thermal inertia modeling. \\
\textbf{SOLAR\_LONGITUDE} & 115.346344$^\circ$ & Areocentric longitude of the Sun ($L_s$). A value of $\sim$115$^\circ$ falls in northern Martian summer. \\
\textbf{SUB\_SOLAR\_AZIMUTH} & 191.465135$^\circ$ & The clockwise angle from the reference vector to the sub-solar point as viewed from the target center. Requires the stated image/map reference convention before use as a local illumination direction. \\
\textbf{NORTH\_AZIMUTH} & 270.000000$^\circ$ & The azimuth angle indicating the direction of the Martian North Pole relative to the center of the image. A value of 270$^\circ$ indicates that North is to the left of the unrotated image frame. \\
\textbf{SCALING\_FACTOR / OFFSET} & $2.1639 \times 10^{-4}$ / 0.0461 & Linear coefficients required to convert raw 16-bit integer Digital Numbers (DN) into physical reflectance values ($I/F$). \\ \bottomrule
\end{tabular}
\end{table}
\FloatBarrier
\subsection{Pipeline workflow}

The full pipeline workflow showing how the dataset reader and sampler work together is visualized in Figure \ref{fig:dataset_architecture}.
\begin{figure}[htbp]
    \centering
    \includegraphics[height=0.80\textheight,width=\linewidth,keepaspectratio]{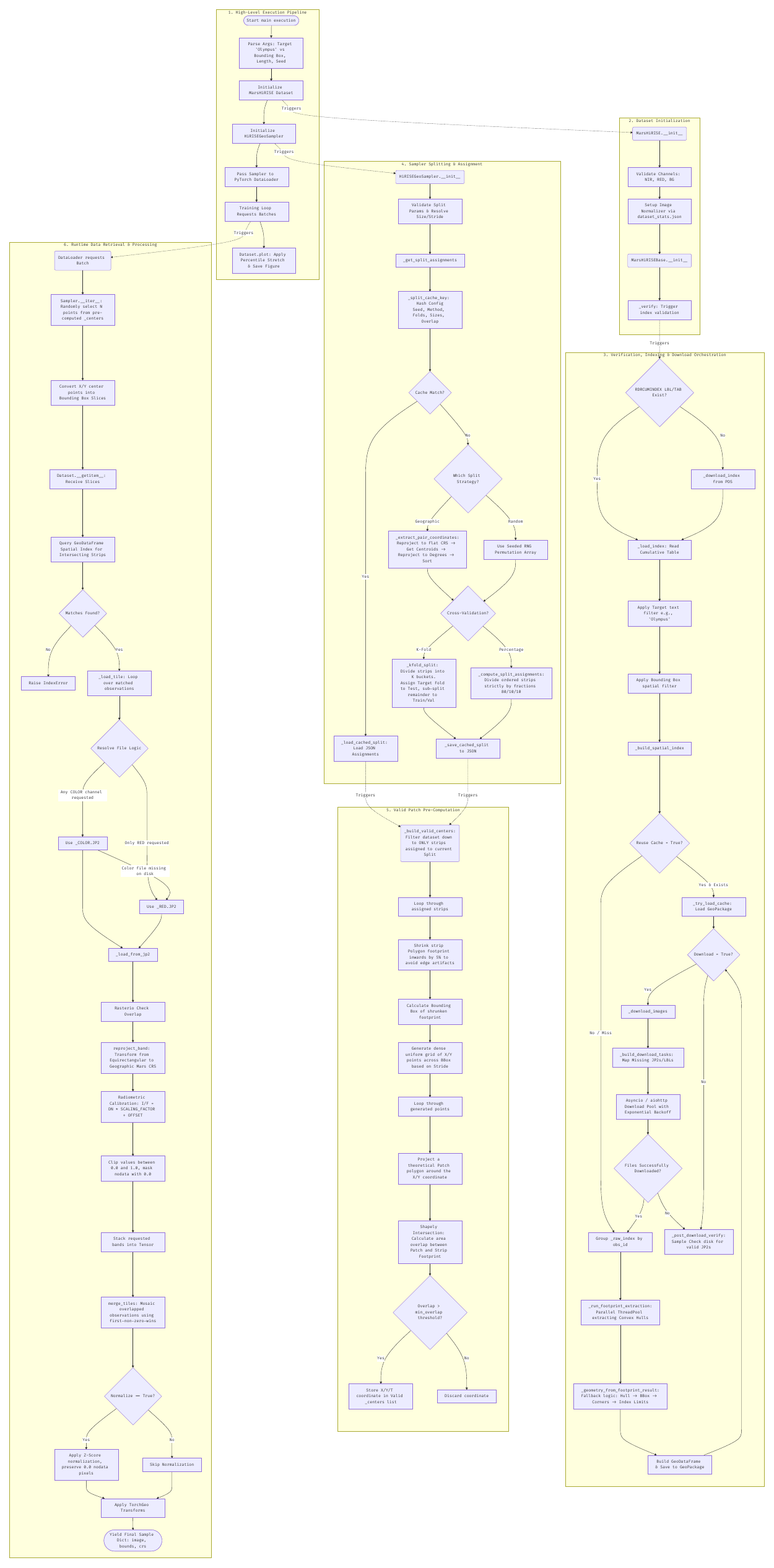}
    \caption{Full pipeline architecture overview of the Mars HiRISE dataset loader and sampler}
    \label{fig:dataset_architecture}
\end{figure}
JPEG2000 stores a wavelet-coded representation whose random-window cost depends on tiling, precincts, resolution levels, caching, and decoder implementation. Repeated patch reads can therefore incur substantial decode overhead. The original preprocessing workflow is retained in Figures~\ref{fig:dataset_architecture} and~\ref{fig:dataset_preprocessing}; a numerical speedup requires a benchmark on the actual product and storage configuration.

We convert the selected JPEG2000 products to Cloud-Optimized GeoTIFFs with 512$\times$512 internal tiles and overviews. Rasterio can then request the tiles intersecting a patch. This makes the access pattern predictable and supports repeated training reads; the cost scales with the requested tiles, decompression, and I/O, rather than a universal asymptotic distinction between the two formats.
\begin{figure}[htbp]
    \centering
    \includegraphics[height=0.80\textheight,width=\linewidth,keepaspectratio]{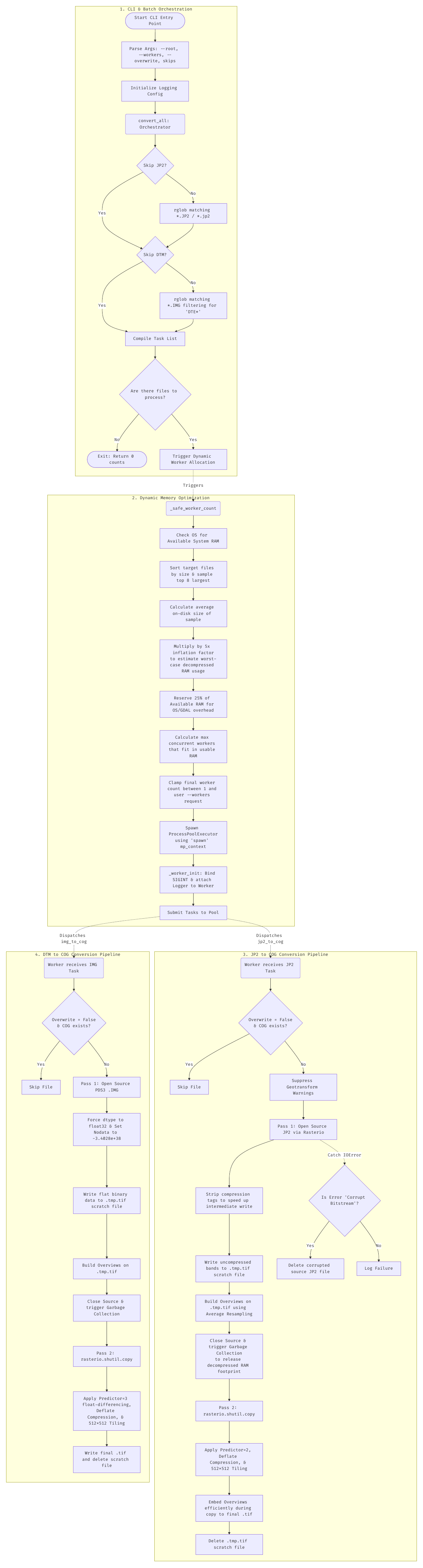}
    \caption{Full pipeline architecture overview of the Mars HiRISE dataset preprocessing to convert the JPEG2000 images into Cloud-Optimized GeoTIFF}
    \label{fig:dataset_preprocessing}
\end{figure}
\FloatBarrier
\subsection{Code example}

Listing~\ref{lst:hirise_pipeline} illustrates the dataset and sampler interface. It assumes the project classes, plotting modules, logger, and TorchGeo units have been imported. On platforms using spawned workers, the executable entry point must also be protected by the usual Python main guard.
\begin{figure}[htbp]
\begin{lstlisting}
# optimal normalization statistics based on the radio-calibrated samples
normalization_path = "dataset_stats/image/dataset_stats.json"

olympus = False
global_coverage = True

dataset = MarsHiRISE(
    target="Olympus" if olympus else None,
    bbox=None if olympus else (-136, 12, -124, 24),
    channels=["NEAR-INFRARED", "RED", "BLUE-GREEN"],
    download=True,
    reuse_cache=True,
    normalize=True,
    normalization_path=normalization_path,
)

output_path = pathlib.Path("Figures")
output_path.mkdir(parents=True, exist_ok=True)

if global_coverage:
    dataset.plot_global_coverage(output_path / "global_coverage.pdf")
    logger.info("Saved global coverage fig")

sampler = HiRISEGeoSampler(
    dataset,
    size=0.005,
    length=None,  # use all valid sample locations
    units=Units.CRS,
)

logger.info("Number of samples: %d", len(sampler))
dataloader = DataLoader(
    dataset, sampler=sampler,
    num_workers=10, multiprocessing_context="spawn", prefetch_factor=4,
)
logger.info("Number of data-loader: %d", len(dataloader))

for i, sample in enumerate(dataloader):
    output_path.mkdir(parents=True, exist_ok=True)
    fig = dataset.plot(sample)
    fig.savefig(output_path / f"output{i}.png")
    logger.info("Saved fig output%d.png", i)
    plt.close(fig)
\end{lstlisting}
\caption{Geospatial dataloading and patch extraction pipeline for the Mars HiRISE dataset. Note the use of the specialized coordinate reference system (CRS) sampler to efficiently map the continuous spatial bounds of Olympus Mons to discrete, normalizable tensors.}
\label{lst:hirise_pipeline}
\end{figure}
\FloatBarrier
\subsection{Rationale Text Analysis} \label{sec:ratioanl_text}

The HiRISE index contains a \texttt{RATIONALE\_DESC} field explaining why each observation was collected. The analysis script parses the official PDS fixed-width index, with a maximum of 75 characters, deduplicates repeated RED/COLOR products by observation ID, computes observation centers from latitude/longitude bounds, tokenizes the rationale text, removes non-scientific stopwords, preserves hyphenated geological terms, and assigns science themes through keyword sets tuned to HiRISE vocabulary.

\begin{figure}[htbp]
    \centering
    \includegraphics[width=\linewidth,height=0.80\textheight,keepaspectratio]{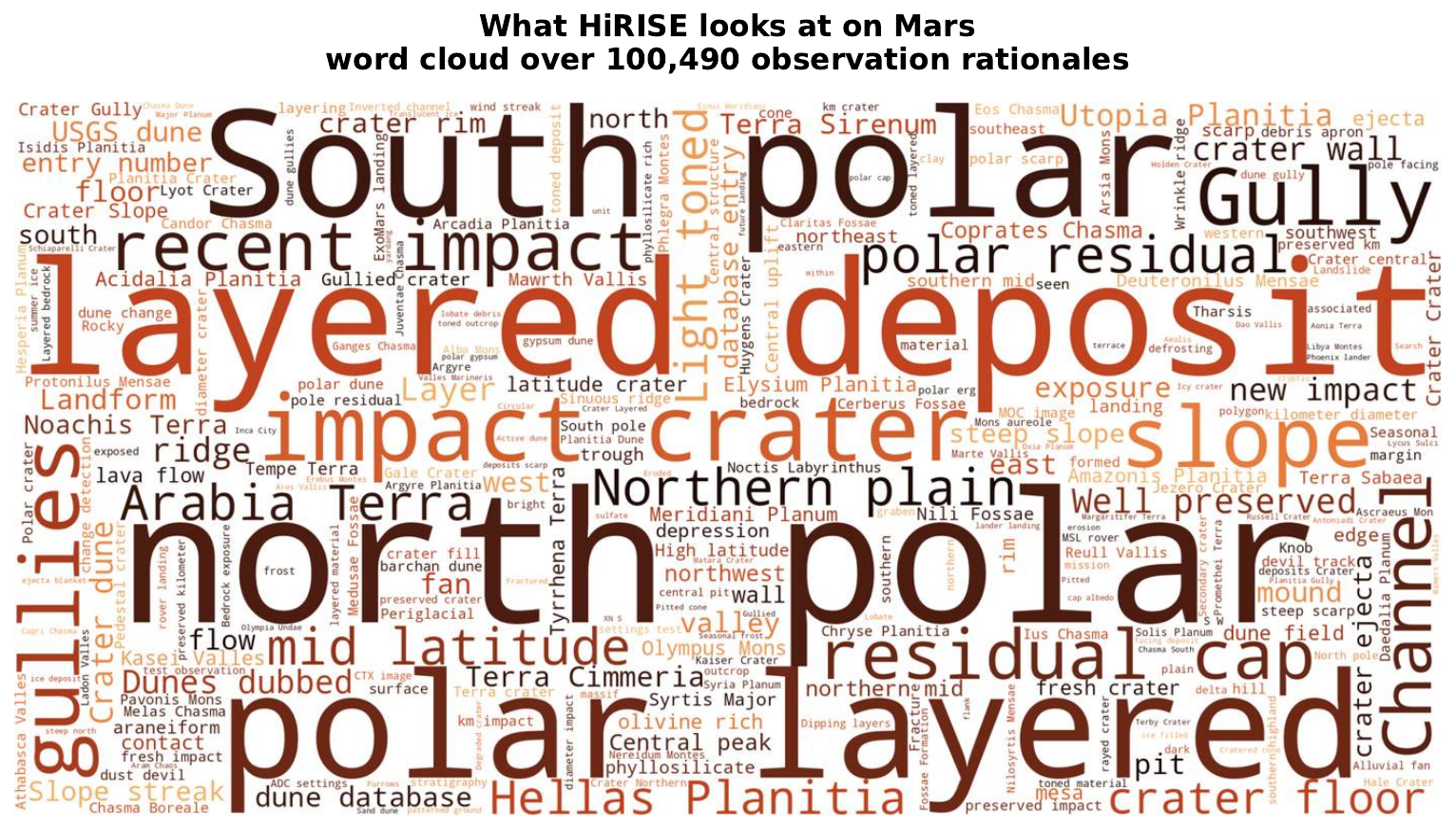}
    \caption{\textbf{HiRISE targeting vocabulary.} The word clouds summarize the most prominent terms in observation rationales after stopword filtering. They show that HiRISE text is strongly geologic rather than generic metadata, with repeated concepts such as craters, dunes, layers, polar processes, slopes, deposits, and channels.}
    \label{fig:rationale_wordclouds_appendix}
\end{figure}

\begin{figure}[htbp]
    \centering
    \includegraphics[width=\linewidth,height=0.80\textheight,keepaspectratio]{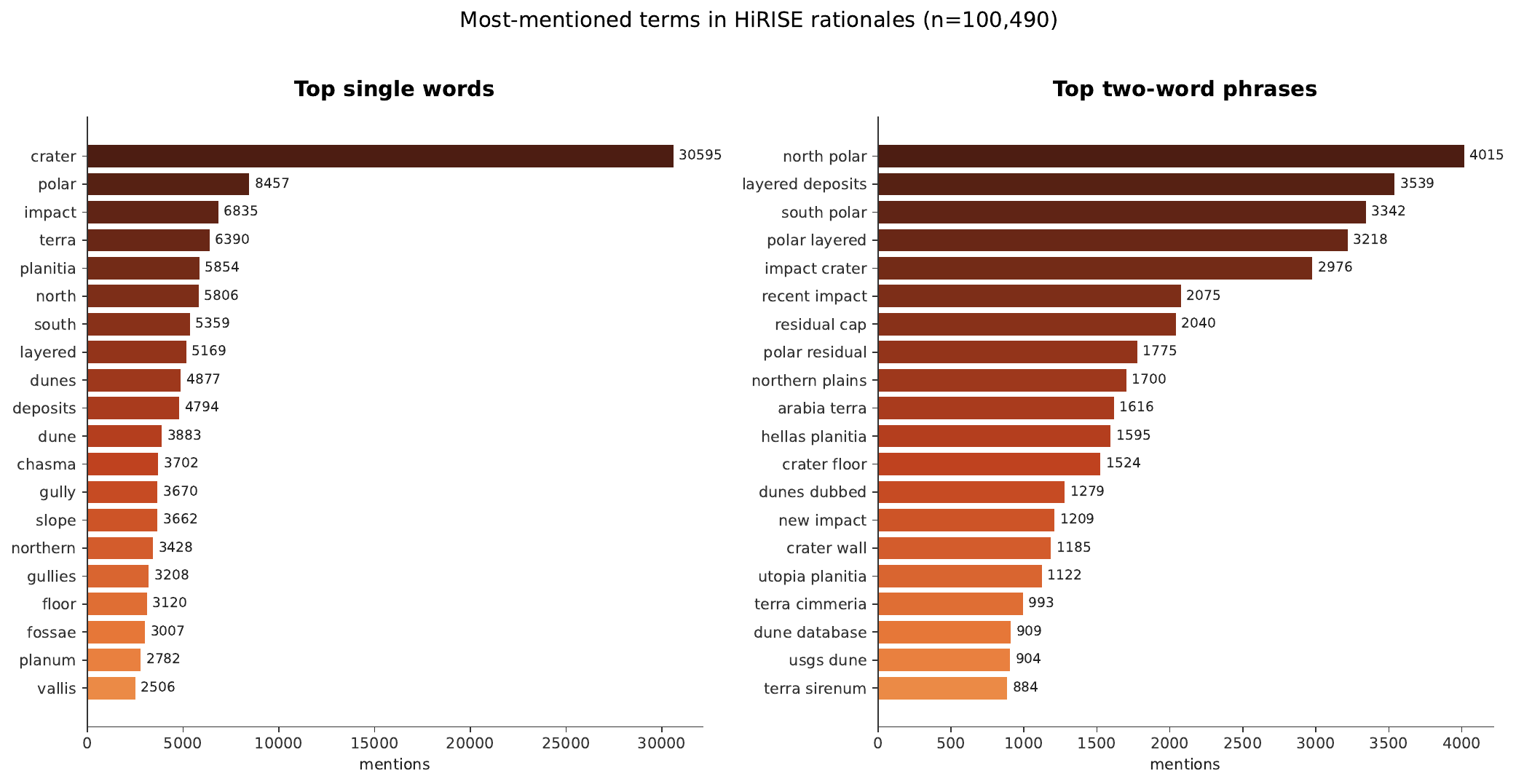}
    \caption{\textbf{Top rationale words and phrases.} The script counts filtered unigrams and adjacent token bigrams, making the text distribution auditable rather than relying only on the word-cloud visualization.}
    \label{fig:rationale_top_words_appendix}
\end{figure}

\begin{figure}[htbp]
    \centering
    \begin{subfigure}[b]{0.49\linewidth}
        \centering
        \includegraphics[width=\linewidth,height=0.80\textheight,keepaspectratio]{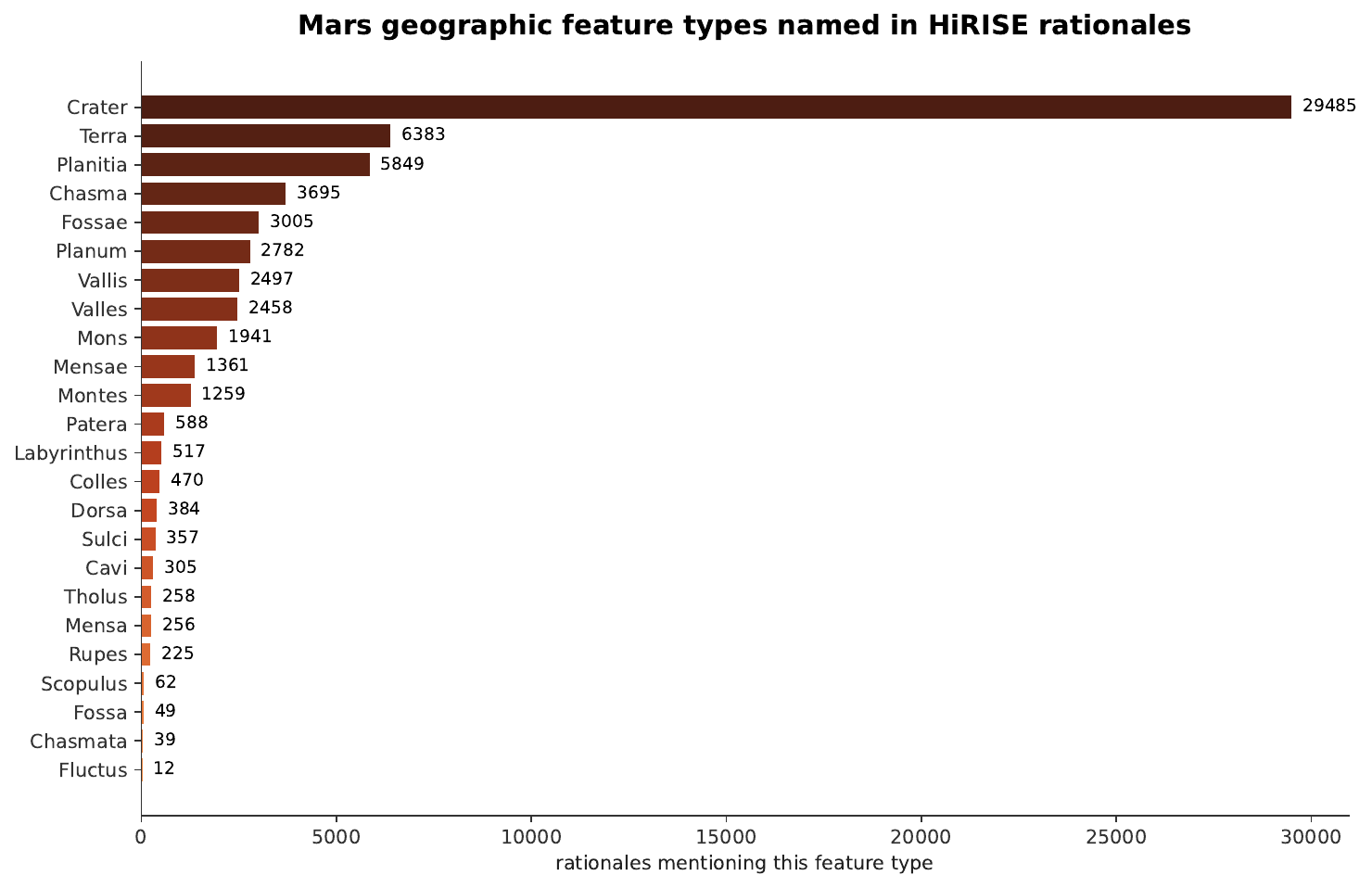}
        \caption{Named geographic feature types}
    \end{subfigure}
    \hfill
    \begin{subfigure}[b]{0.49\linewidth}
        \centering
        \includegraphics[width=\linewidth,height=0.80\textheight,keepaspectratio]{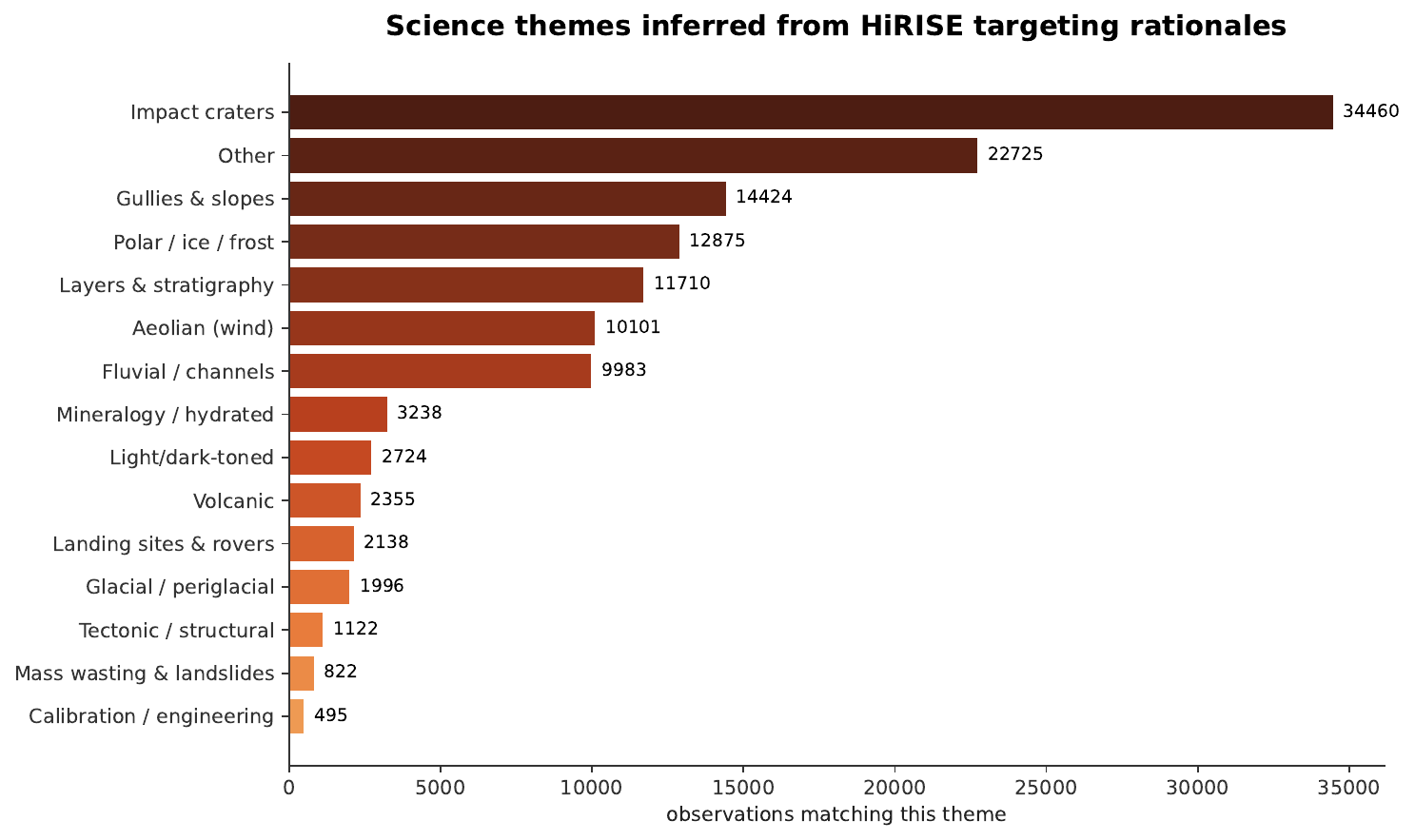}
        \caption{Keyword-derived science themes}
    \end{subfigure}
    \caption{\textbf{From words to geological categories.} Geographic-feature counts track named Mars feature types such as craters, chasmata, valles, fossae, and montes. Theme counts group rationales into science categories including impact craters, gullies/slopes, aeolian processes, polar/ice/frost, layers/stratigraphy, mineralogy, volcanic processes, tectonics, and landing-site/engineering observations.}
    \label{fig:rationale_theme_counts_appendix}
\end{figure}

\begin{figure}[htbp]
    \centering
    \includegraphics[width=\linewidth,height=0.80\textheight,keepaspectratio]{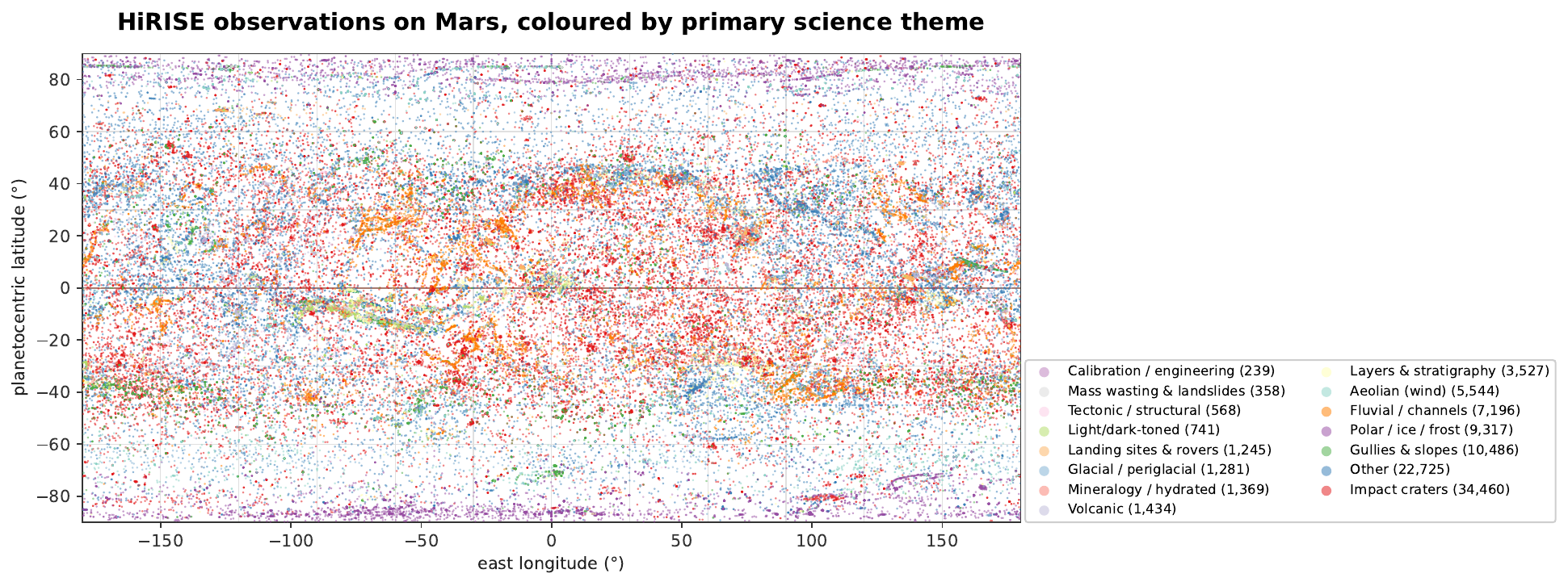}
    \caption{\textbf{Observation map colored by primary science theme.} Each observation is assigned a primary theme from its matched keyword list and plotted by longitude/latitude. This links the keyword-derived catalog characterization to geography; text is not a conditioning input in the reported model.}
    \label{fig:rationale_themed_map_appendix}
\end{figure}

\begin{figure}[htbp]
    \centering
    \includegraphics[width=\linewidth,height=0.80\textheight,keepaspectratio]{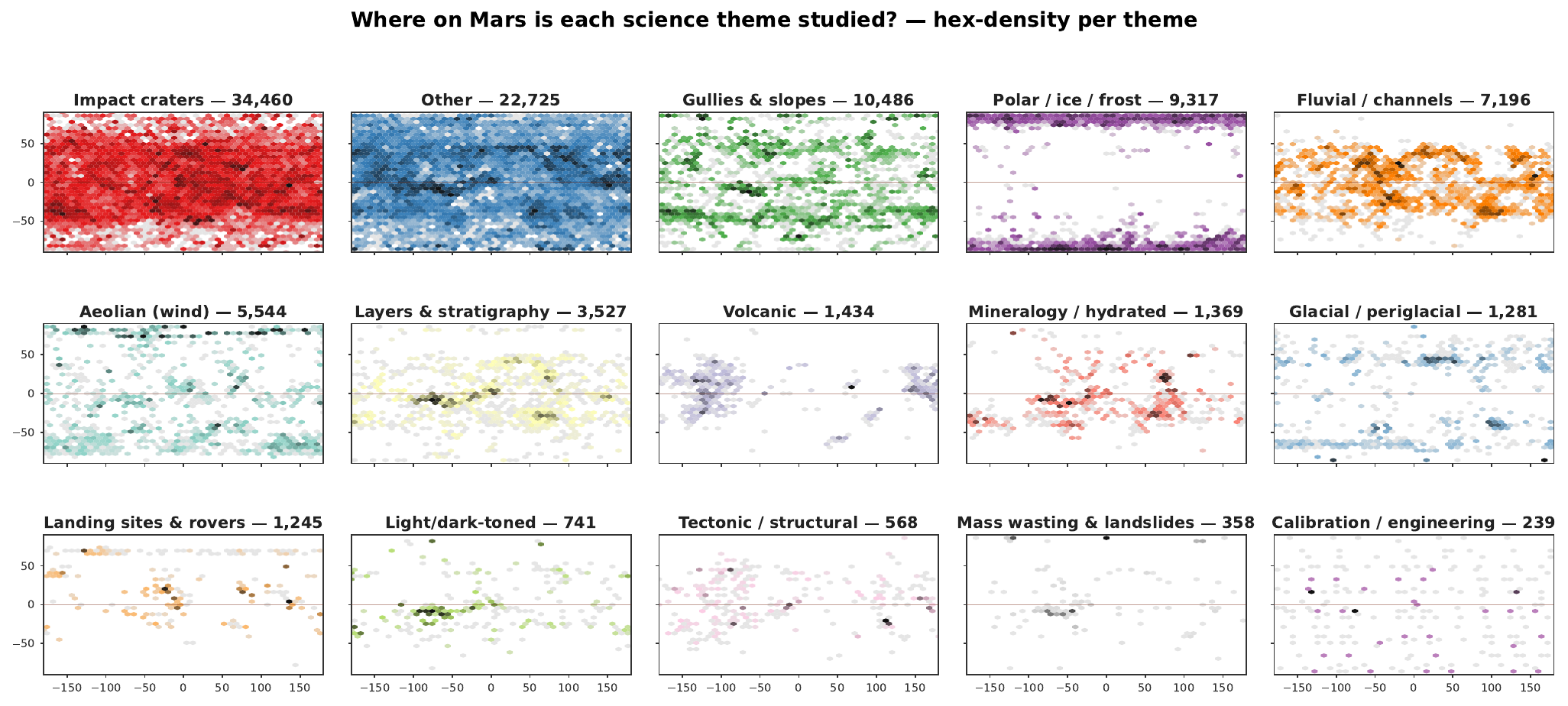}
    \caption{\textbf{Theme-specific Mars maps.} Small multiples reduce overplotting by drawing one hex-density map per science theme, revealing where different rationale classes cluster spatially.}
    \label{fig:rationale_smallmultiples_appendix}
\end{figure}

\begin{figure}[htbp]
    \centering
    \includegraphics[width=\linewidth,height=0.80\textheight,keepaspectratio]{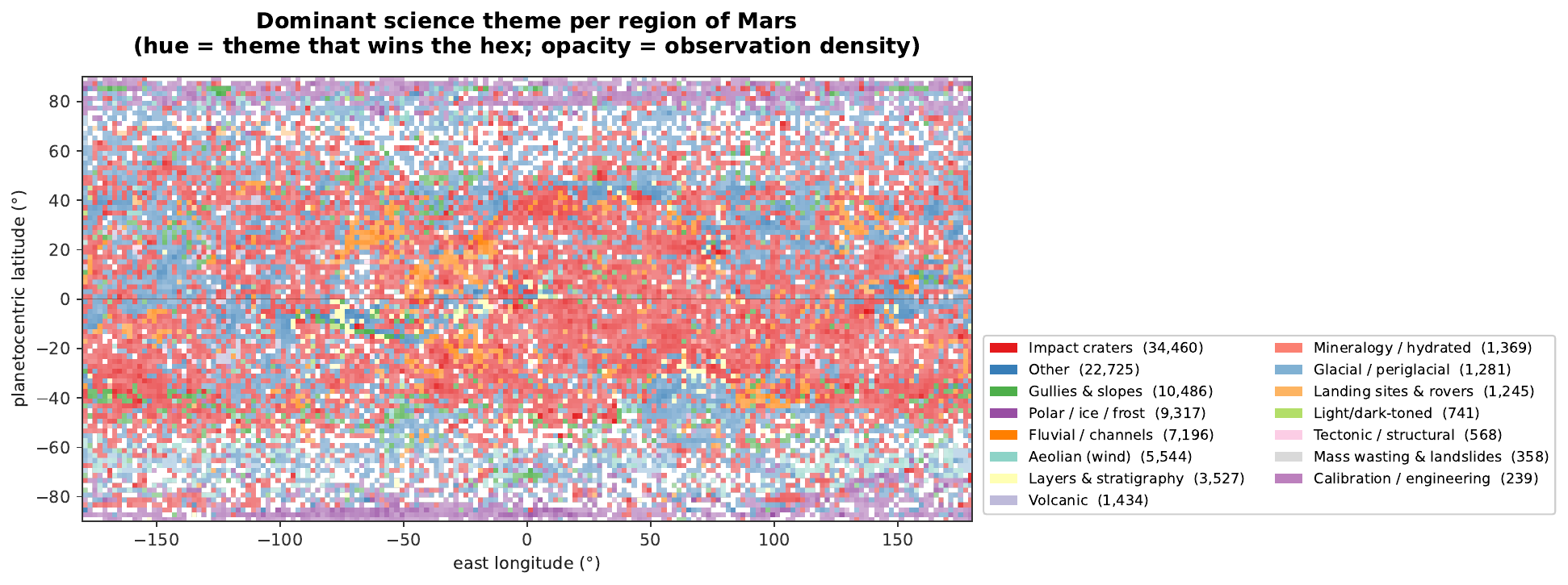}
    \caption{\textbf{Dominant theme per Mars region.} Each spatial bin is colored by the locally dominant primary theme, with opacity scaled by log observation count. This preserves a single-map view while reducing the overplotting problem of point-wise thematic maps.}
    \label{fig:rationale_dominant_map_appendix}
\end{figure}

\begin{figure}[htbp]
    \centering
    \begin{subfigure}[b]{0.46\linewidth}
        \centering
        \includegraphics[width=\linewidth,height=0.80\textheight,keepaspectratio]{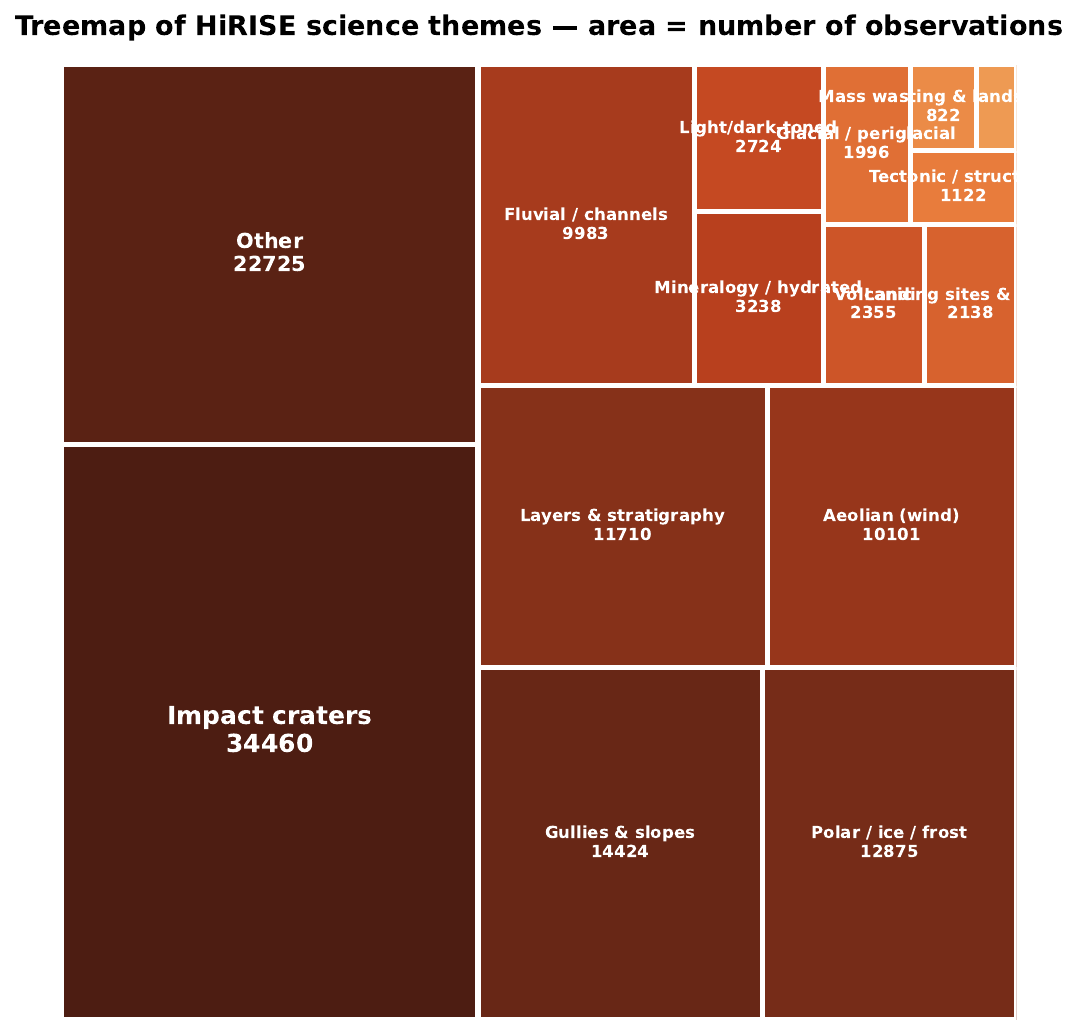}
        \caption{Theme proportions}
    \end{subfigure}
    \hfill
    \begin{subfigure}[b]{0.50\linewidth}
        \centering
        \includegraphics[width=\linewidth,height=0.80\textheight,keepaspectratio]{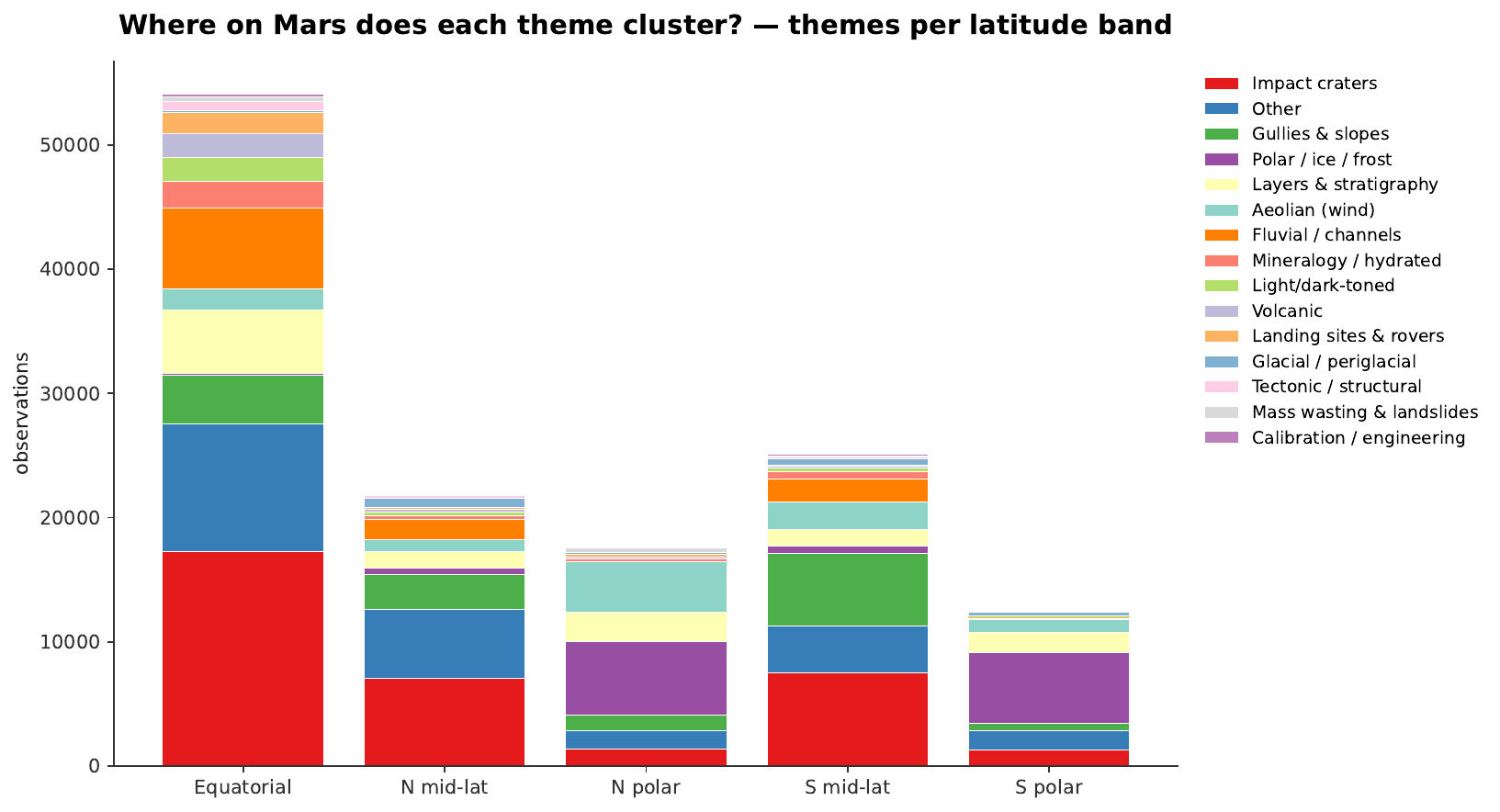}
        \caption{Themes by latitude band}
    \end{subfigure}
    \caption{\textbf{Theme imbalance and latitude dependence.} The treemap shows the global class imbalance in rationale-derived science themes, while the stacked latitude-band plot shows that theme frequencies vary by equatorial, mid-latitude, and polar regions. These distributions motivate geographically grouped evaluation and careful sampling for any future text-conditioned extension. Keyword themes describe targeting rationales, not verified geological labels for every pixel.}
    \label{fig:rationale_theme_distribution_appendix}
\end{figure}

\newpage
\FloatBarrier
\section{Dataset patch validation and sampling} \label{sec:sampling}

The catalog bounding box encloses a product's footprint and can contain substantial nodata around an angled strip. We approximate the outer valid footprint by its convex hull and accept candidate patches according to their intersection with that footprint. A separate pixel mask is needed because the hull can bridge internal holes and concavities. The default 50\% threshold refers to patch--footprint overlap; it is distinct from overlap between neighboring patches. HiRISE observations may overlap spatially, so footprint-aware sampling also needs observation grouping when constructing train/test splits. Figure~\ref{fig:strip_validation} shows accepted and rejected candidates.
\begin{figure}[htbp]
    \centering
    \includegraphics[width=1\linewidth,height=0.80\textheight,keepaspectratio]{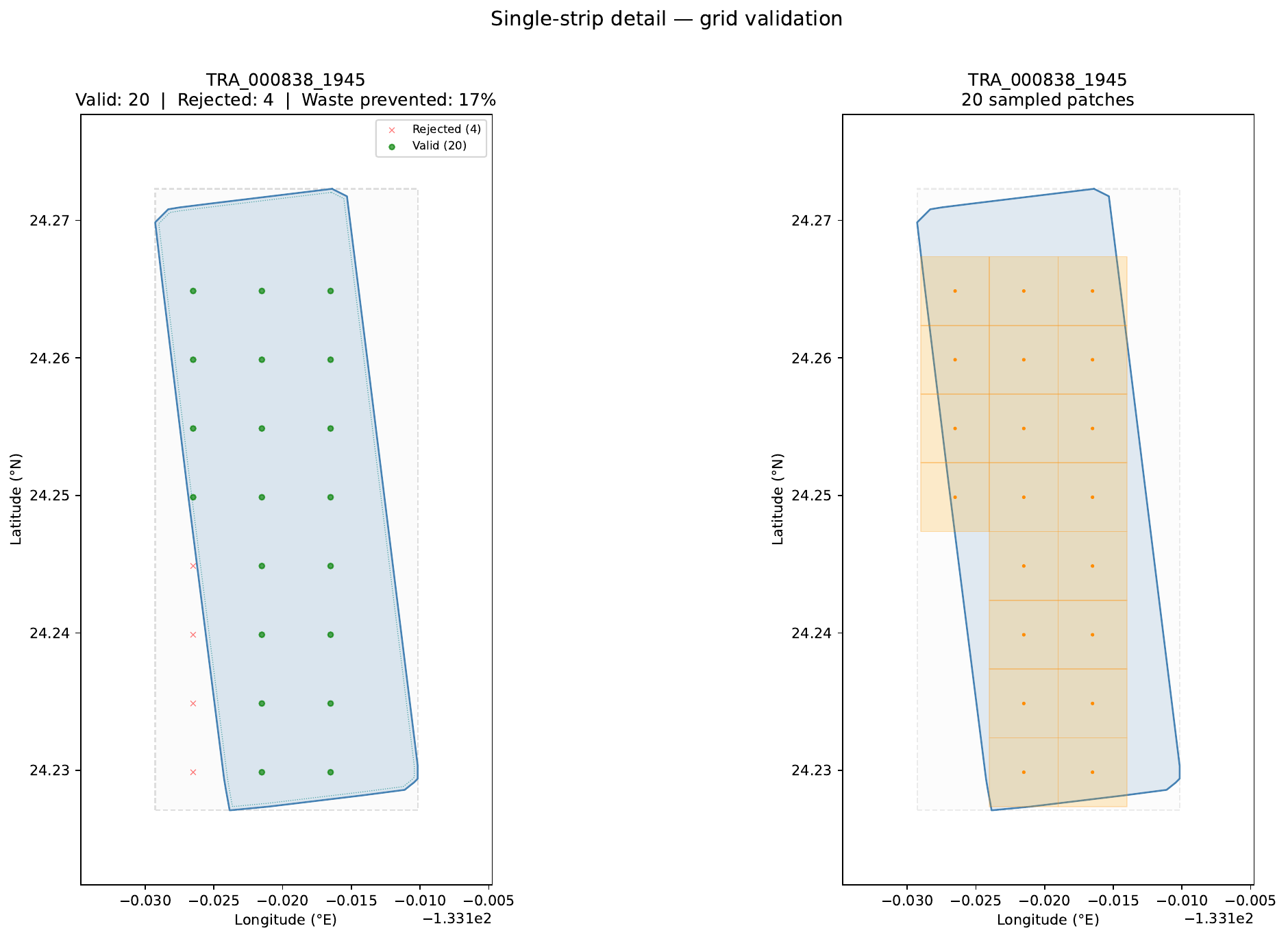}
    \caption{Demonstration of the patch sampling for a single strip inside the HiRISE dataset, the gray box represents the max bounding box that is provided from the metadata, the \textcolor{blue}{contour plot} represents the region where data is available.}
    \label{fig:strip_validation}
\end{figure}
This sampling could be further optimized by using sampling centers based on explicitly estimating the oriented envelope (minimum rotated rectangle \cite{Shapely.minimum_rotated_rectangleDocumentation}) that encloses the input orbital trip, ensuring the resulting rectangle has the minimum area possible to provide better coverage. From this custom sampler, we provide example region outputs in Figure \ref{fig:thumnail_examples}.
\begin{figure}[htbp]
    \centering
    \includegraphics[width=1\linewidth,height=0.80\textheight,keepaspectratio]{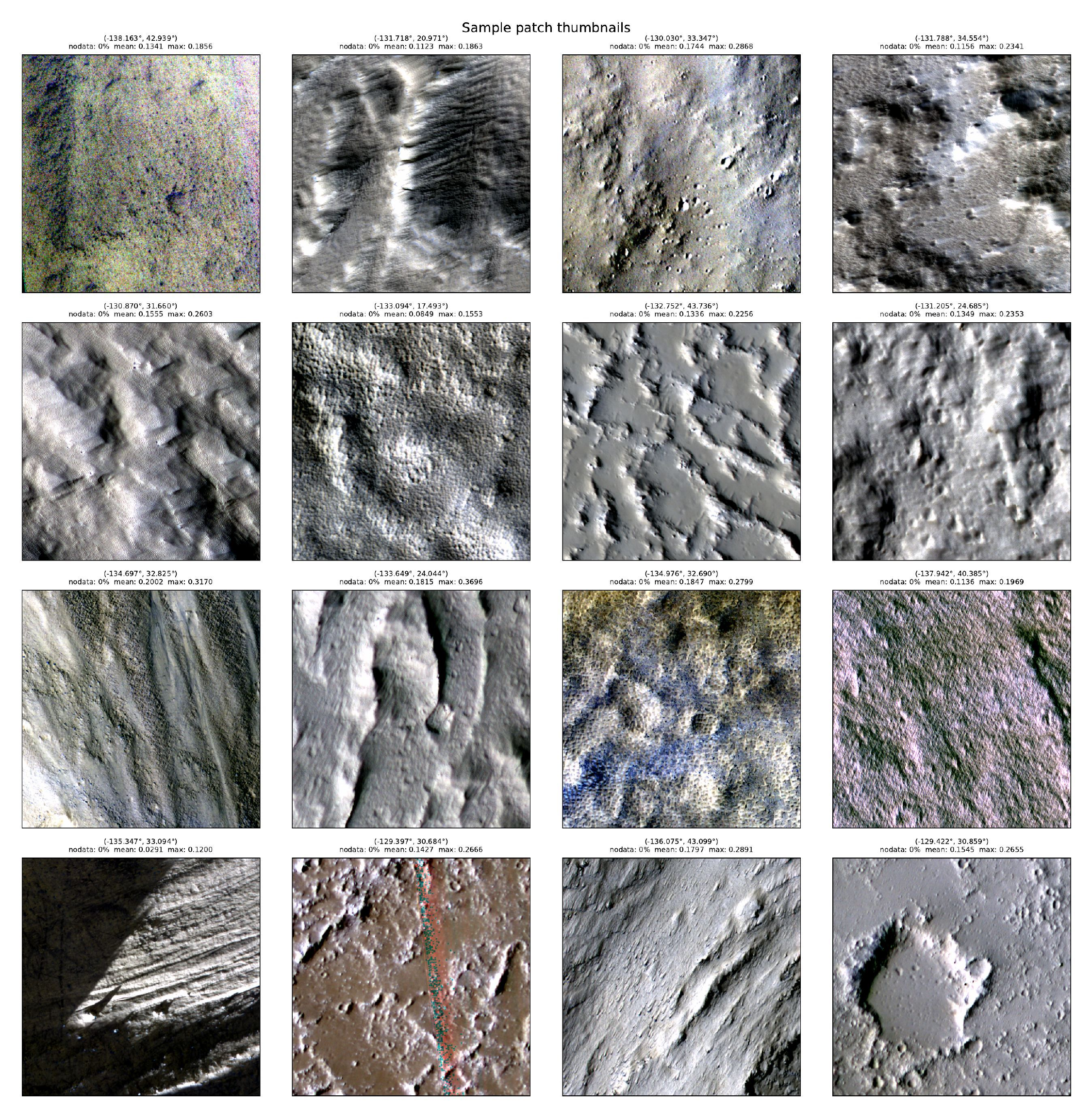}
    \caption{Thumbnail plot grid of 16 HiRISE measurements from the "Olympus" dataset target with angular resolution of 0.005 degrees. Due to the radiometry calibration, the image scale is 0-0.2; to improve visualization, we normalize per-patch pixel intensities to the 2nd and 98th percentiles. From these randomly selected plots, we can clearly see the range in random features available in the currently selected subset of the HiRISE dataset.}
    \label{fig:thumnail_examples}
\end{figure}
\FloatBarrier
\subsection{Bin-packing sampling}

Consider the problem of tiling a convex polygon $P \subset \mathbb{R}^2$ with a set of axis-aligned squares of uniform side length $L$. The packing must satisfy two constraints:
\begin{enumerate}
    \item \textbf{Overlap Guarantee:} Each square must overlap with $P$ by at least a fraction $p \in (0, 1]$ of its total area.
    \item \textbf{Stride:} A parameter $\delta\in[0,1)$ defines $\sigma=L(1-\delta)$ between adjacent centers within a row and between rows. For a purely horizontal or vertical shift of $\sigma$, the area-overlap fraction is $\delta$; for arbitrary offsets $(\Delta x,\Delta y)$ it is $(1-|\Delta x|/L)_+(1-|\Delta y|/L)_+$.
\end{enumerate}
General two-dimensional packing problems can be computationally difficult~\cite{Garey1979ComputersNP-completeness,Lodi2002Two-dimensionalSurvey}. Our restricted problem first constructs a feasible center region~\cite{Lozano-Perez1983SpatialApproach} and then maximizes the number of centers independently on fixed rows. The result below concerns convex footprints, exact overlap tests, and a fixed orientation/row system. It does not inherit a global optimality guarantee for arbitrary two-dimensional packing.

\FloatBarrier
\subsubsection{Validity of the Configuration Space}
\begin{definition}[Valid Center Region]
Let $S(c)$ be an axis-aligned square of side length $L$ centered at $c = (x,y)$. The valid center region $R$ is defined as:
\[
R = \{ c \in \mathbb{R}^2 \mid \text{Area}(P \cap S(c)) \ge p L^2 \}
\]
\end{definition}
\begin{theorem}[Convexity of the Valid Region]
If $P$ is a convex polygon, the valid center region $R$ is convex.
\end{theorem}
\begin{proof}
Let $K_i=P\cap S(c_i)$ for two feasible centers $c_0,c_1\in R$ and let $c_t=(1-t)c_0+tc_1$, $t\in[0,1]$. Convexity of $P$ and the translated square gives
\[
(1-t)K_0+tK_1\subseteq P\cap S(c_t).
\]
The Brunn--Minkowski inequality~\cite{Schneider2013ConvexTheory} therefore yields
\[
\sqrt{\operatorname{Area}(P\cap S(c_t))}
\ge (1-t)\sqrt{\operatorname{Area}(K_0)}+t\sqrt{\operatorname{Area}(K_1)}
\ge\sqrt{pL^2}.
\]
Thus $c_t\in R$. The result also allows $R$ to be empty; strict log-concavity is not required.

\end{proof}
\begin{lemma}[Feasible interval along a ray]
Let $c_0\in R$ be a verified feasible center. For a unit vector $u$, the set $\{r\ge0:c_0+ru\in R\}$ is an interval containing zero.
\end{lemma}
\begin{proof}
Intersecting a convex set with a ray gives a convex subset of that ray. Since $P$ is bounded, $L>0$, and $p>0$, the feasible set is bounded; continuity of the intersection area makes it closed. Consequently the interval has the form $[0,r_{\max}]$, possibly with $r_{\max}=0$.
\end{proof}
This result justifies searching for the exit from a feasible interval, but does not imply that intersection area decreases immediately from the polygon centroid. The centroid must first be checked for feasibility. A bracketed root finder such as Brent's method~\cite{Brent1972AlgorithmsDerivatives} additionally needs an inside point with positive residual and an outside point with negative residual. A threshold plateau may contain multiple roots, so an implementation must define the outer feasible endpoint and its numerical tolerance. Sampling finitely many directions produces an approximation to $R$; candidate centers should be checked against the original overlap constraint.


As demonstrated in Figure \ref{fig:min_overlap_polygon}, the minimum overlap polygon restricts the valid-center region.
\begin{figure}[htbp]
    \centering
    \includegraphics[width=1\linewidth,height=0.80\textheight,keepaspectratio]{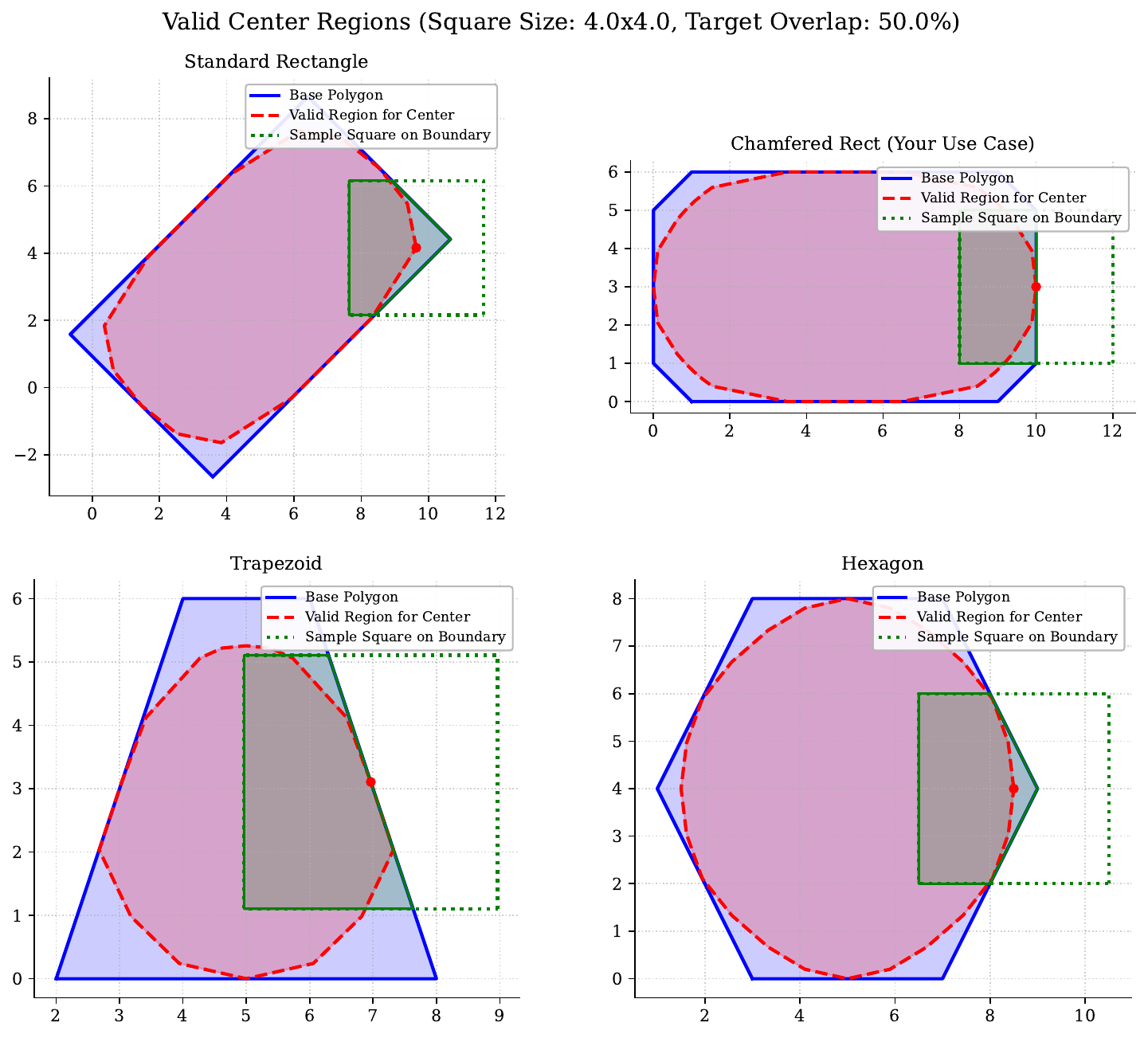}
    \caption{Feasible center regions for four convex footprints at $p=0.5$. Blue shows the footprint, red the sampled feasible boundary, and green a test square. The exact region is convex but need not be polygonal; a sampled polygon approximates it.}
    \label{fig:min_overlap_polygon}
\end{figure}
\FloatBarrier
\subsubsection{Optimality of Independent Strip Packing}

Having established the convex bounding geometry $R$, the algorithm restricts all square centers to $c \in R$. The packing algorithm discretizes the space into independent horizontal (or vertical) strips.
\begin{definition}[Strip Domain]
For a given vertical phase shift $y_0$ and stride $\sigma = L(1-\delta)$, the $k$-th horizontal strip is evaluated at $y_k = y_0 + k\sigma$. The available domain for placing centers is the 1D intersection segment $I_k = R \cap \{ (x, y_k) \mid x \in \mathbb{R} \}$, illustrated in Figures \ref{fig:hero_single_strip_po00} and \ref{fig:hero_single_strip_po50}.
\end{definition}
\begin{figure}[hbtp]
    \centering
    \includegraphics[width=1\linewidth,height=0.80\textheight,keepaspectratio]{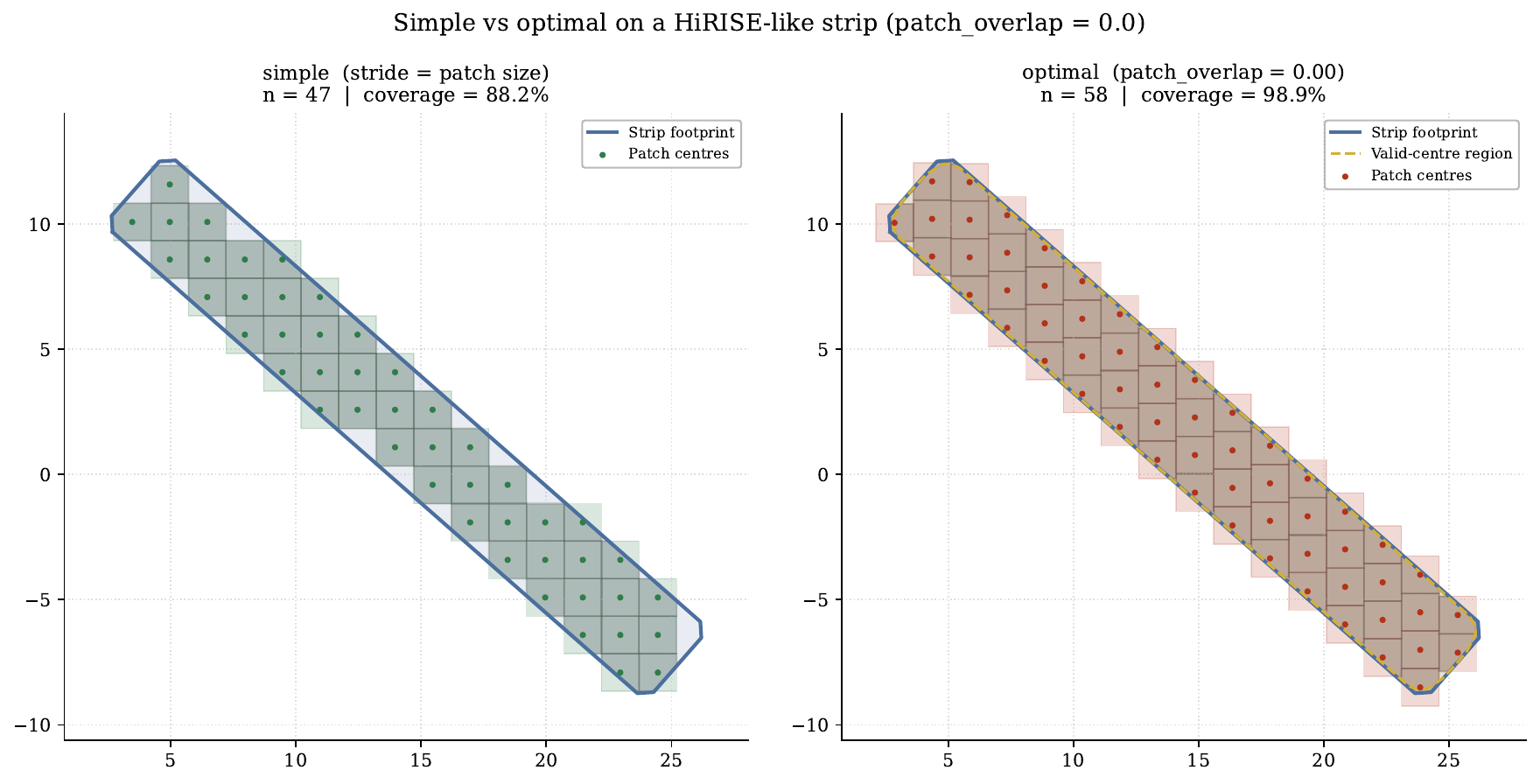}
    \caption{Column-locked grid and independent-row placement on one synthetic strip at $\delta=0$. The exported counts are 47 and 58 valid patches, with reported coverage 88.2\% and 98.9\%. These illustrate the benefit on this geometry; the theorem guarantees maximal count within each fixed row, not maximal covered area.}
    \label{fig:hero_single_strip_po00}
\end{figure}
\begin{figure}[hbtp]
    \centering
    \includegraphics[width=1\linewidth,height=0.80\textheight,keepaspectratio]{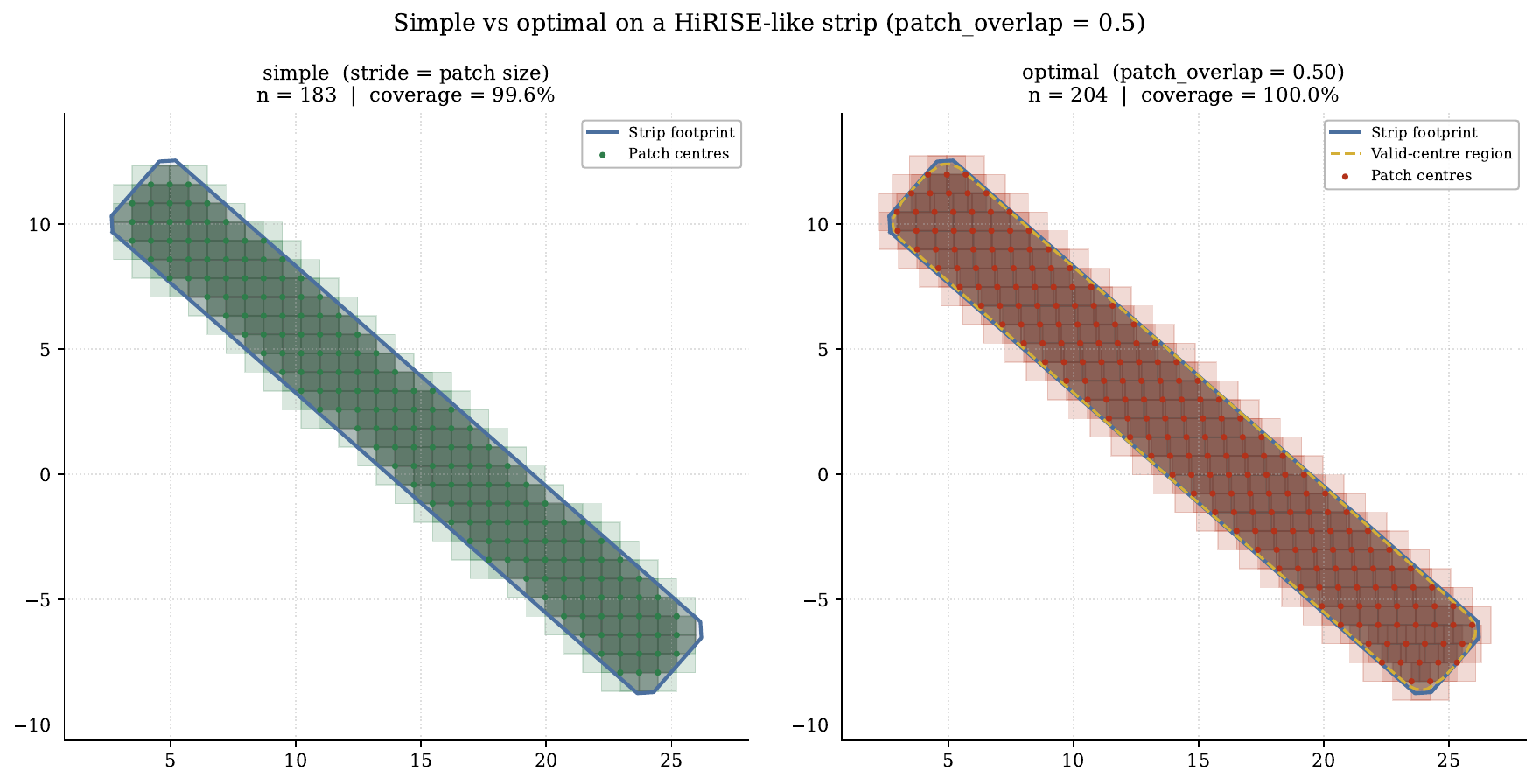}
    \caption{The same comparison at $\delta=0.5$. The export reports 183 versus 204 valid patches and 99.6\% versus 100\% footprint coverage. ``Optimal'' in the retained plot refers to independent-row placement under the stated restrictions.}
    \label{fig:hero_single_strip_po50}
\end{figure}
Since $R$ is convex, each nonempty $I_k$ is a line segment (possibly a single point) $[x_{min}, x_{max}]$ with available width $W = x_{max} - x_{min}$.
\begin{theorem}[1D Strip Optimality]
Given a continuous segment $I_k$ of width $W$, and a minimum required stride $\sigma > 0$ between points, the maximum number of square centers $N$ that can be packed is exactly $N_{max} = \lfloor W / \sigma \rfloor + 1$. Independent placement within that row attains this bound. Empty rows contribute zero centers.
\end{theorem}
\begin{proof}
Let a set of points $\{x_1, x_2, \dots, x_N\}$ be placed on the interval $[x_{min}, x_{max}]$ such that $x_{i+1} - x_i \ge \sigma$.
The interval width bounds the total distance spanned by $N$ points:
\[
\sum_{i=1}^{N-1} (x_{i+1} - x_i) \le x_{max} - x_{min} = W
\]
Since $x_{i+1} - x_i \ge \sigma$, we have $(N-1)\sigma \le W$.
Solving for $N$ yields $N \le \frac{W}{\sigma} + 1$.
Since $N$ must be an integer, $N_{max} = \lfloor W/\sigma \rfloor + 1$.

The algorithm calculates the required count directly via this formula. It then centers the points on $I_k$ by placing the first point at $x_1 = x_{min} + \frac{1}{2}(W - (N_{max}-1)\sigma)$. The centering offset $\tfrac12[W-(N_{max}-1)\sigma]$ is nonnegative, hence $x_1\ge x_{min}$; the coordinate $x_1$ itself need not be nonnegative. The last point is $x_{N_{max}} = x_1 + (N_{max}-1)\sigma = x_{max} - \frac{1}{2}(W - (N_{max}-1)\sigma) \le x_{max}$. Thus, all points are perfectly contained in $I_k$, proving that the algorithm achieves the row-wise upper bound on the number of centers. Demonstrating this optimal placement using a variety of polygons (Figure \ref{fig:gallery_strips}), this figure illustrates that the optimal square-center placement algorithm provides a consistent improvement over the standard axis-aligned box placement.
\end{proof}
\begin{figure}[htbp]
    \centering
    \includegraphics[width=0.8\linewidth,height=0.80\textheight,keepaspectratio]{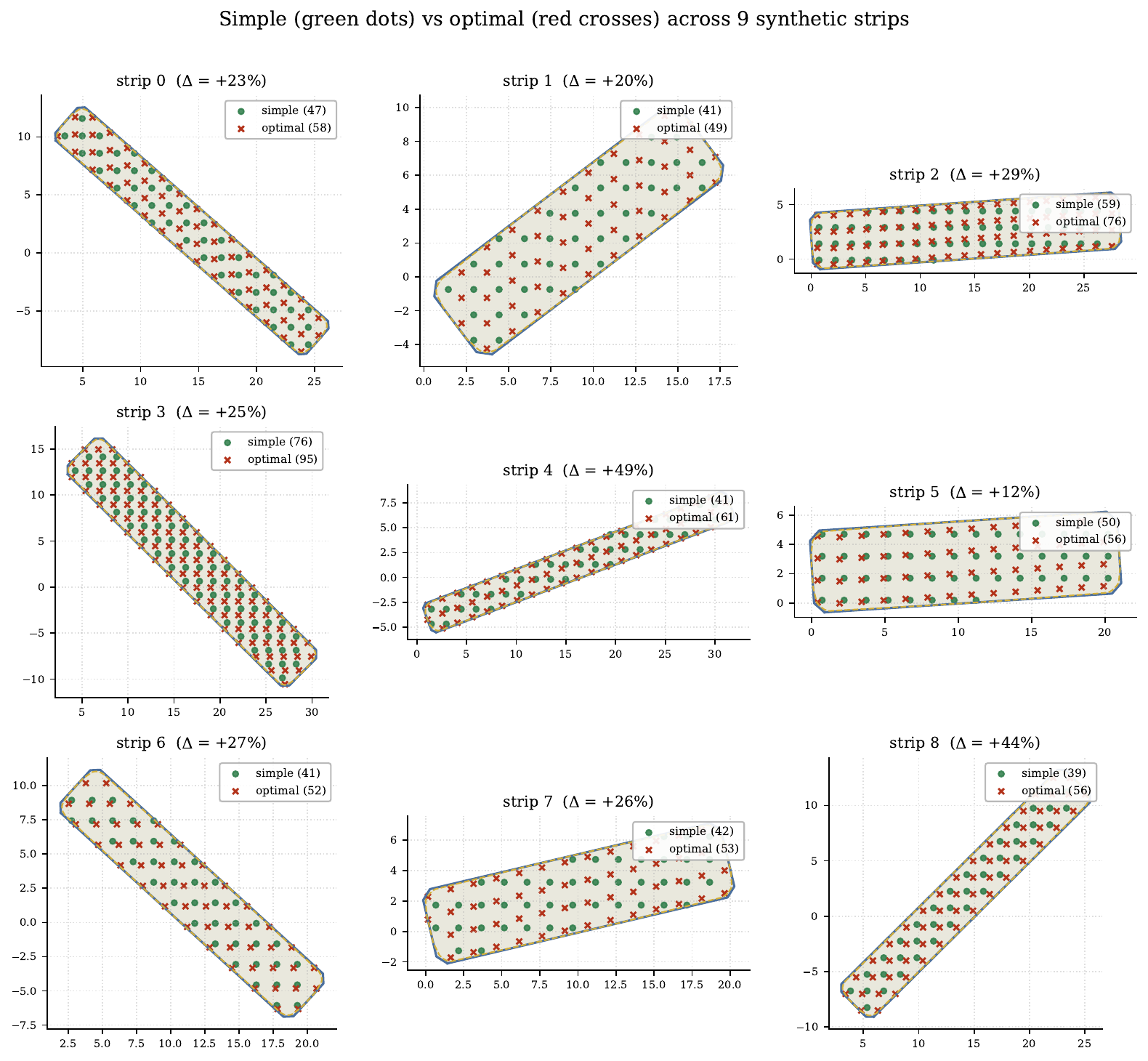}
    \caption{Nine synthetic footprints comparing column-locked centers (green) and independent-row centers (red). Each panel retains the original counts and relative gain. The examples support a descriptive yield comparison under the sampled geometries; they do not establish a universal coverage improvement.}
    \label{fig:gallery_strips}
\end{figure}
\begin{lemma}[Superiority over Orthogonal Grids]
Fix the patch orientation, stride, and row coordinates $y_k$. Let $N_{grid}$ count valid centers from any column-locked orthogonal grid on these rows, and let $N_{strip}$ be the sum of the row-wise maxima. Then $N_{strip}\ge N_{grid}$. This is a cardinality result for the same row system, not a proof of globally optimal two-dimensional coverage.
\end{lemma}
\begin{proof}
A standard 2D grid enforces strict column locking; if row $k$ contains a point at $x_i$, row $k+1$ must place a point precisely at $x_i$ to maintain the column. The independent strip algorithm relaxes this constraint, evaluating each row $I_k$ exclusively based on its own width $W_k$. Since a global grid is merely a special case of strip packing where $x_{start}$ is identically constrained across all rows, the unconstrained optimization of each row independently cannot perform worse than the constrained case. Therefore, $N_{strip} \ge N_{grid}$. The synthetic comparisons illustrating this restricted guarantee are retained in Figures \ref{fig:publication_figure}, \ref{fig:sweep_min_overlap}, and \ref{fig:sweep_patch_overlap}.
\end{proof}
\begin{figure}[htbp]
    \centering
    \includegraphics[width=1\linewidth,height=0.80\textheight,keepaspectratio]{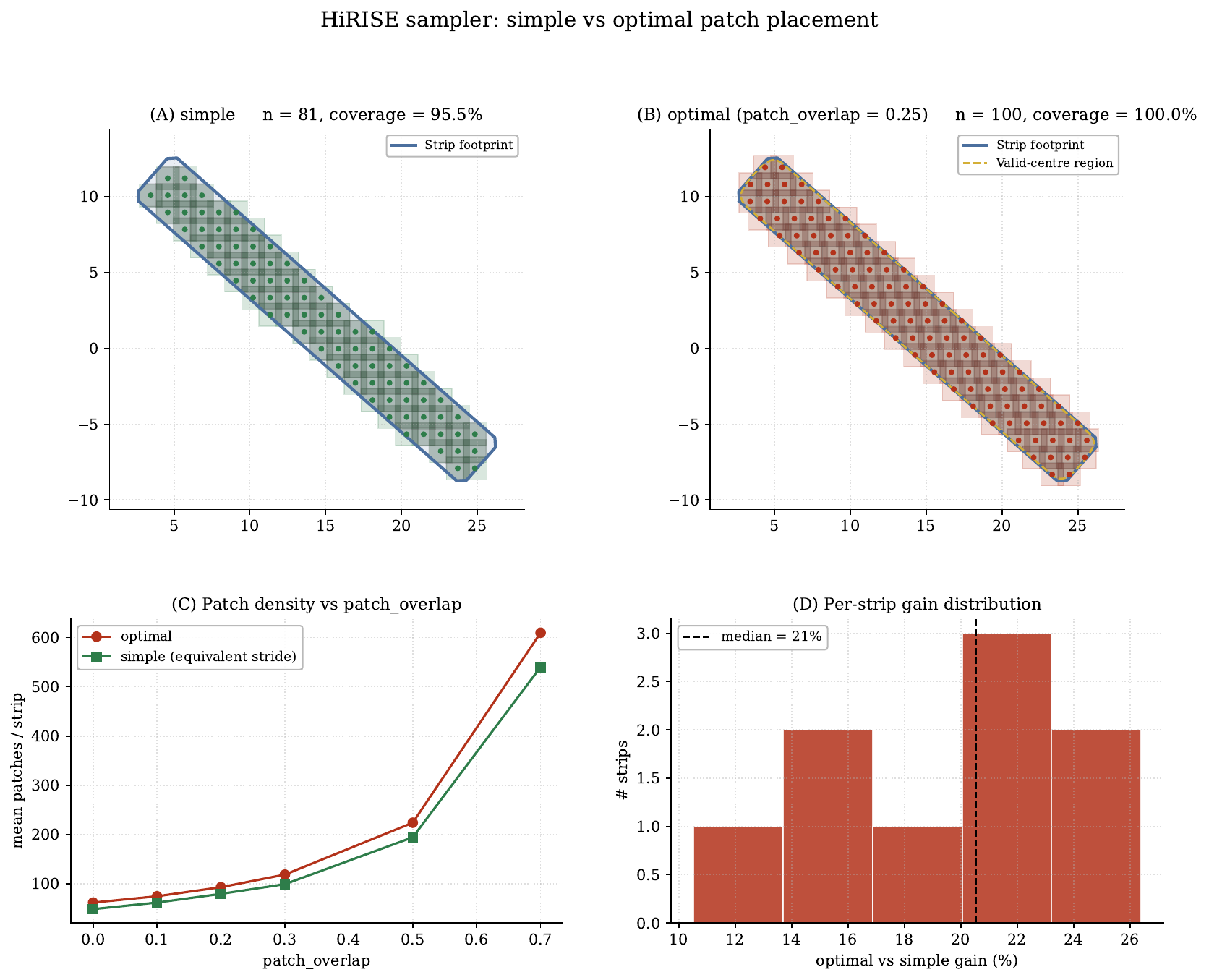}
    \caption{Comprehensive evaluation of the simple versus optimal patch placement algorithms. Panels (A) and (B) illustrate the spatial distribution and coverage gains on a representative synthetic strip. Panel (C) tracks the mean patch density as a function of the internal patch overlap parameter. Panel (D) presents the distribution of per-strip percentage gains across the synthetic dataset, indicating higher patch counts in the illustrated examples; these are descriptive comparisons, not a statistical significance test.}
    \label{fig:publication_figure}
\end{figure}
\begin{figure}[htbp]
    \centering
    \includegraphics[width=1\linewidth,height=0.80\textheight,keepaspectratio]{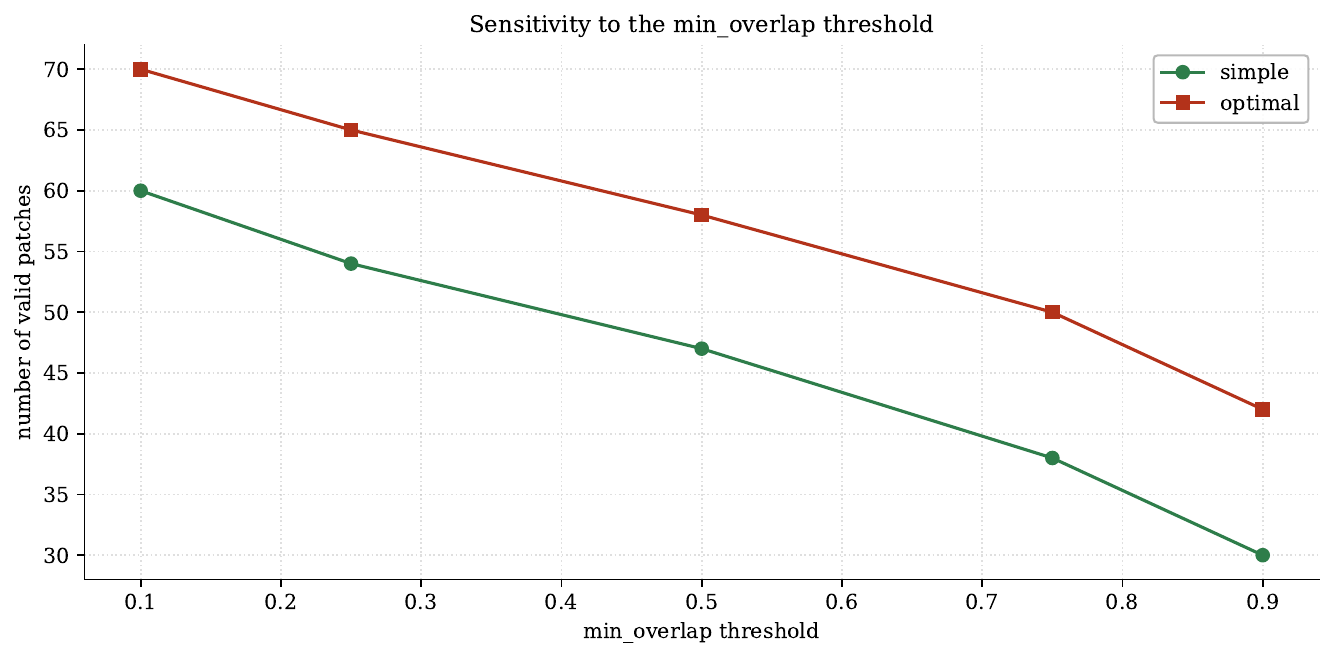}
    \caption{Sensitivity analysis of the valid patch yield with respect to the minimum required overlap threshold ($p$). The optimal independent strip packing algorithm consistently identifies more valid patch centers across all threshold constraints than the simple grid baseline.}
    \label{fig:sweep_min_overlap}
\end{figure}
\begin{figure}[htbp]
    \centering
    \includegraphics[width=1\linewidth,height=0.80\textheight,keepaspectratio]{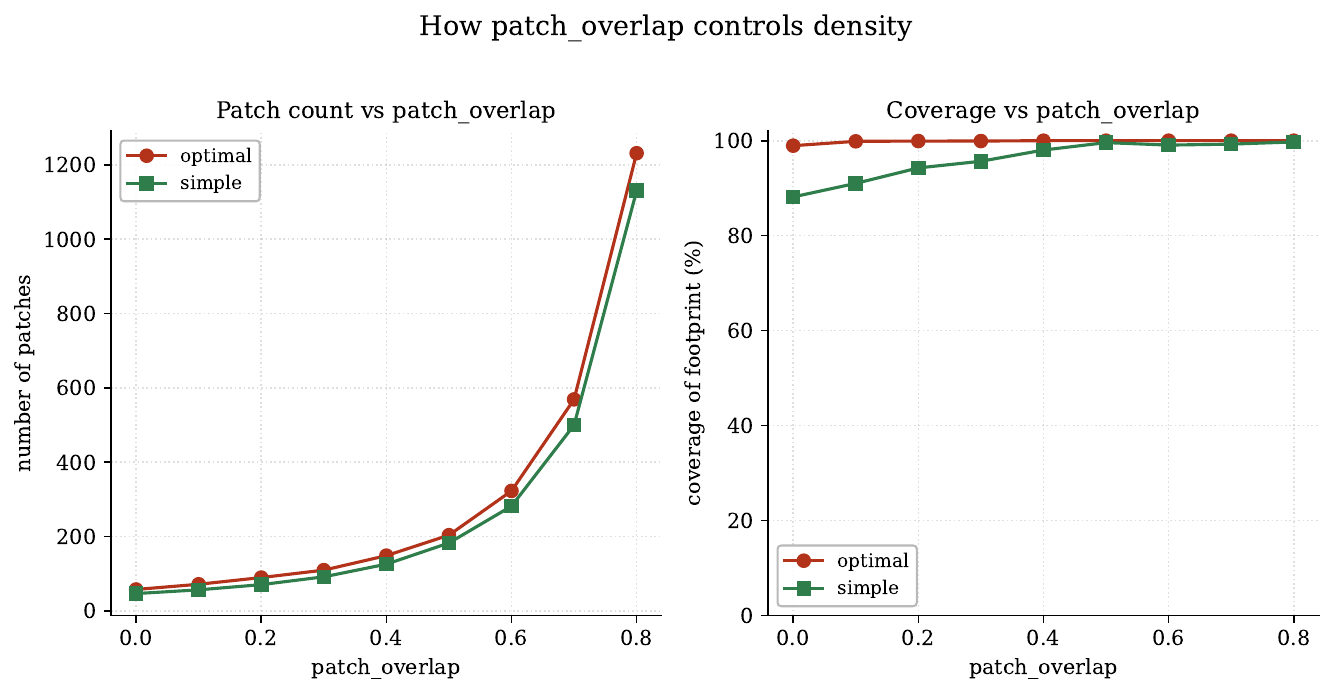}
    \caption{Impact of the patch overlap parameter ($\delta$) on packing density and footprint coverage. The left panel shows how the total patch count scales with increasing overlap, while the right panel illustrates the corresponding increase in total area coverage, highlighting the optimal algorithm's superior space utilization.}
    \label{fig:sweep_patch_overlap}
\end{figure}
%
%

\newpage
\FloatBarrier
\section{DTM Dataset} \label{sec:dtm_dataset}

\FloatBarrier
\subsection{DTM characteristics and processing pipeline} \label{sec:dtm_characteristics}

HiRISE DTM posting is commonly about four times the source-image pixel scale: 0.25--0.5\,m imagery can accompany terrain grids at 1--2\,m posting. Vertical precision depends on stereo geometry, matching quality, and terrain; it is not a universal sub-meter error bound. For the pipeline's nominal 1\,m DTM grid, the equatorial angular spacing based on the reference equatorial radius is:
\begin{align}
    \text{res}_{\text{DTM}} = \frac{1}{59,274.70} \approx 1.687 \times 10^{-5} \text{ \degree/pixel}
\end{align}
This posting is four times coarser than a 0.25\,m image grid and twice as coarse as a 0.5\,m image grid. The comparison uses posting only; interpolation to the image grid does not increase independently measured terrain detail.

\FloatBarrier
\subsubsection{Data formats and metadata}

Unlike standard RDR products, DTM elevations are stored as single-file PDS3 products (`.IMG`) with the metadata label embedded directly at the start of the binary file. The raster data is stored as 32-bit IEEE 754 single-precision floats representing physical elevation in meters. Because the data is already in physical units, no secondary radiometric calibration is applied to the elevation values.

Conversely, the orthorectified images (orthos) attached to each stereo pair are JPEG2000 (`.JP2`) files with detached `.LBL` sidecars. These orthos are produced via the SOCET Set pipeline, which outputs 8-bit brightness values rather than the 10-bit instrument digital number (DN) range preserved in the primary RDR path. Consequently, the standard calibration constants for converting to radiance factor (I/F) cannot be naively reused; each ortho image requires specific scaling and offset factors parsed directly from its respective label, as compared in Table \ref{tab:dtm_vs_rdr_ortho}.

\begin{table}[H]
    \centering
    \caption{Example encoded-image attributes in the inspected RDR and DTM-ortho products. Read bit depth, scaling, and offset from each label; these examples are not universal product constants.}
    \label{tab:dtm_vs_rdr_ortho}
    \begin{tabular}{@{} l c c @{}}
        \toprule
        \textbf{Attribute} & \textbf{Standard RDR} & \textbf{DTM Ortho} \\
        \midrule
        Sample Bits & 16 & 8 \\
        Effective Max DN & 1023 (10-bit) & 255 (8-bit) \\
        Typical Scaling Factor & $\sim 1.24 \times 10^{-4}$ & $\sim 8.49 \times 10^{-5}$ \\
        Typical Offset & $\sim 0.033$ & $\sim 0.069$ \\
        \bottomrule
    \end{tabular}
\end{table}
The orthorectified images used for DTM generation use mosaics of either the RED CCDs or three-band images comprising near-infrared, red visible, and blue-green visible (i.e., IRB). Each of the three color bands can be converted from DN to I/F using the same formula as for the single band image.

Figure~\ref{fig:pixel_histograms_dtm} compares the left and right RED distributions for the inspected 200-product subset. Similar aggregate histograms are useful for checking encoding and exposure ranges, although they do not demonstrate pixelwise radiometric correspondence. Pixel totals count sampled values and can include repeated coverage from overlapping patches.
\begin{figure}[htbp]
     \centering
     \begin{subfigure}[b]{0.49\textwidth}
         \centering
         \includegraphics[width=\textwidth,height=0.80\textheight,keepaspectratio]{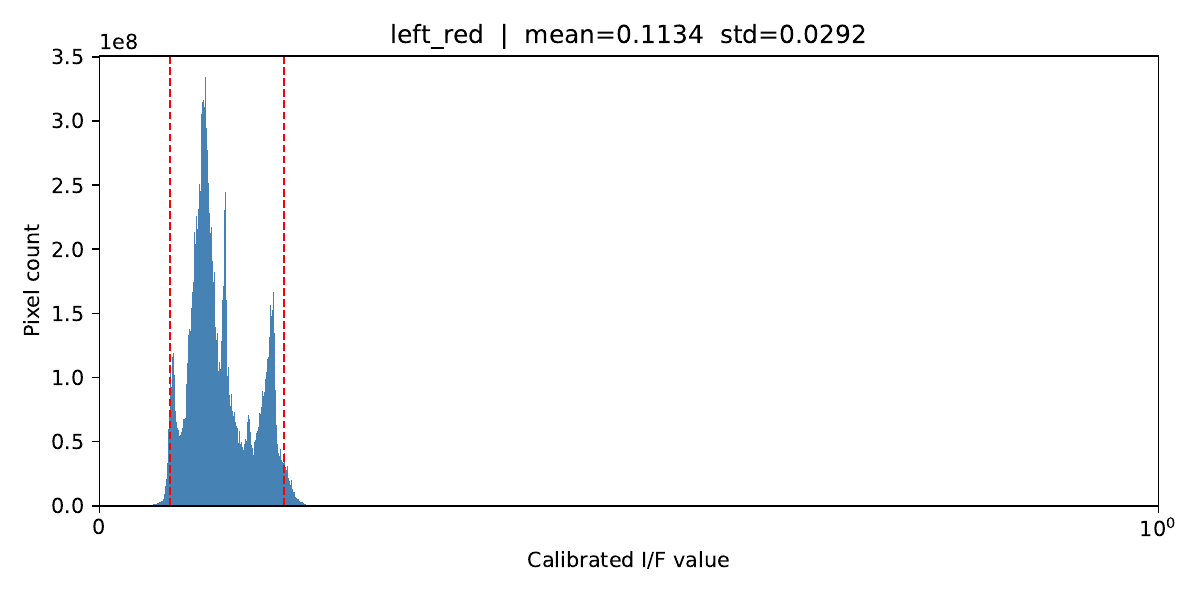}
         \caption{Left red spectral band}
         \label{fig:histogram_dtm_left_red}
     \end{subfigure}
     \hfill
     \begin{subfigure}[b]{0.49\textwidth}
         \centering
         \includegraphics[width=\textwidth,height=0.80\textheight,keepaspectratio]{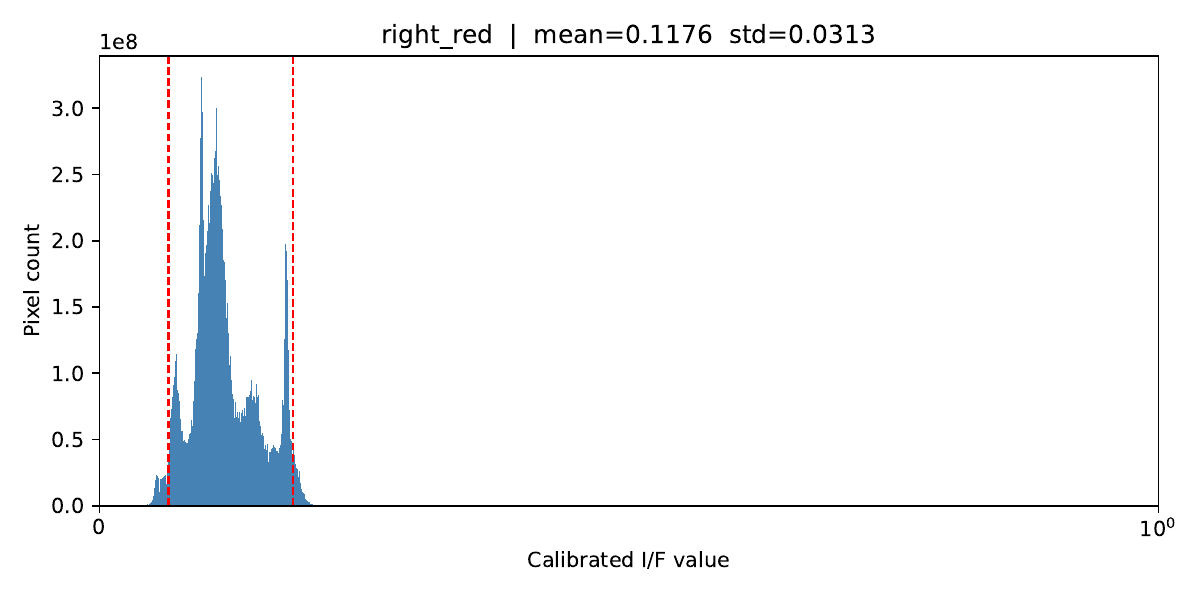}
         \caption{Right red spectral band}
         \label{fig:histogram_dtm_right_red}
     \end{subfigure}
        \caption{Per-pixel radiometrically calibrated distribution of valid patches with a patch size of 0.018 degrees over 200 HiRISE DTM measurement readings. These patches are derived from the HiRISE DTM dataset and subsampled by filtering samples with a bounding box of $[-120\degree, -30\degree, 150\degree, 30\degree]$. The histograms are unnormalized frequency histograms that provide aggregated information on approximately 25 billion pixels per channel. The red lines in the histograms represent the 2nd and 98th percentiles of the data.}
        \label{fig:pixel_histograms_dtm}
\end{figure}
The problem is the elevation data: as noted in Figure \ref{fig:pixel_histograms_dtm_elevation} and Table \ref{tab:dataset_stats_extended_dtm}, the elevation range is in absolute measurement values, with minimums ranging from $-4460 m$ to $ 131 m$. From this data, we cannot use the absolute elevation data for the neural network application; we thus preprocessed the elevation patches.
\begin{figure}[htbp]
     \centering
         \includegraphics[width=\textwidth,height=0.80\textheight,keepaspectratio]{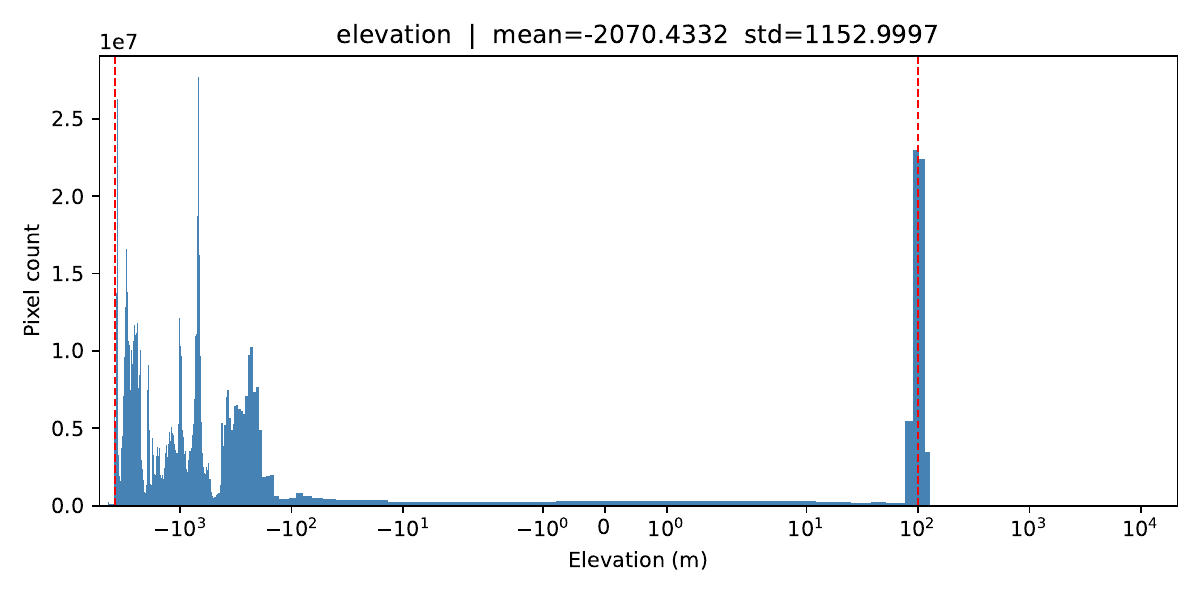}
        \caption{Per-pixel elevation data of valid patches with a patch size of 0.018 degrees over 200 HiRISE DTM measurement readings. These patches are derived from the HiRISE DTM dataset and subsampled by filtering samples within a bounding box of $[-120\degree, -30\degree, 150\degree, 30\degree]$. The histograms are semi-log-scaled, unnormalized frequency histograms that provide aggregated information on approximately 1 billion pixels per channel (there are fewer pixels due to resolution subsampling). The red lines in the histograms represent the 2nd and 98th percentiles of the data.}
        \label{fig:pixel_histograms_dtm_elevation}
\end{figure}
\begin{table}[htbp]
\centering
\caption{Summary statistics of the dataset across spectral channels. P2 and P98 represent the 2\% and 98\% percentiles. CV denotes the coefficient of variation ($\sigma/\mu$). DRCR denotes dynamic range compression ratio ($(P_{98}-P_{2})/\mu$).}
\label{tab:dataset_stats_extended_dtm}

\begin{tabular}{
l
S[table-format=1.4e2]
S[table-format=-4.4]
S[table-format=-4.4]
}
\toprule
\textbf{Channel} & {\textbf{$N_{\text{pixels}}$}} & {\textbf{Mean}} & {\textbf{Std}} \\
\midrule
Elevation & 1.6092e+09 & -2070.4332 & 1152.9997 \\
Left red & 2.5751e+10 & 0.1134 & 0.0292 \\
Right red & 2.6228e+10 & 0.1176 & 0.0313 \\
\bottomrule
\end{tabular}
\vspace{1.5em}

\begin{tabular}{
l
S[table-format=-4.4]
S[table-format=-4.4]
S[table-format=-4.4]
S[table-format=2.4]
S[table-format=-1.3]
S[table-format=-1.3]
}
\toprule
\textbf{Channel} & {\textbf{Min}} & {\textbf{Max}} & {\textbf{P2}} & {\textbf{P98}} & {\textbf{CV}} & {\textbf{DRCR}} \\
\midrule
Elevation & -4460.0215 & 131.0734 & -3818.7003 & 99.5918 & -0.557 & -1.892 \\
Left red & 0.0425 & 0.2005 & 0.0670 & 0.1739 & 0.258 & 0.942 \\
Right red & 0.0380 & 0.2085 & 0.0653 & 0.1829 & 0.266 & 1.000 \\
\bottomrule
\end{tabular}
\end{table}
Each terrain patch is detrended by fitting a plane to its valid pixels and subtracting that plane from the original elevations. The target is therefore local residual relief rather than absolute elevation or regional slope. The reference plane is metadata for reconstruction and must not be silently supplied from test ground truth in a deployable single-image evaluation.

\textbf{Detrending}:
Let $\mathcal{S}$ be the set of coordinates for all valid pixels as defined by the mask $M$:
\begin{align}
\mathcal{S} = \{(x_i, y_i) \in [0, W-1] \times [0, H-1] \mid M(y_i, x_i) = 1\}
\end{align}
Let $N = |\mathcal{S}|$ be the total number of valid pixels. We model the global linear trend as a plane defined by the function $z(x, y) = c_0 x + c_1 y + c_2$. To find the optimal parameters $\mathbf{c} = [c_0, c_1, c_2]^T$, we minimize the sum of squared errors evaluated strictly over the valid subset $\mathcal{S}$:
\begin{align}
\mathbf{c}^* = \arg\min_{\mathbf{c}} \sum_{(x_i, y_i) \in \mathcal{S}} (Z(y_i, x_i) - (c_0 x_i + c_1 y_i + c_2))^2
\end{align}
The least-square formulation is characterized through the design matrix $A \in \mathbb{R}^{N \times 3}$ and the observation vector $\mathbf{z} \in \mathbb{R}^N$:
\begin{align}
A = \begin{bmatrix} x_1 & y_1 & 1 \\ x_2 & y_2 & 1 \\ \vdots & \vdots & \vdots \\ x_N & y_N & 1 \end{bmatrix}, \quad \mathbf{z} = \begin{bmatrix} Z(y_1, x_1) \\ Z(y_2, x_2) \\ \vdots \\ Z(y_N, x_N) \end{bmatrix}
\end{align}
When $A$ has full column rank, the ordinary least-squares solution can be written using the normal equations. A QR or SVD solve is preferable to explicitly forming the inverse in numerical code:
\begin{align}
\mathbf{c}^* = (A^T A)^{-1} A^T \mathbf{z}
\end{align}
Finally, the final detrended matrix $\hat{Z} \in \mathbb{R}^{H \times W}$ is computed by evaluating the estimated trend surface over the entire grid and subtracting it from the original data:
\begin{align}
\hat{Z}(y, x) = Z(y, x) - (c_0^* x + c_1^* y + c_2^*)
\end{align}
The final statistics of the detrended dataset are represented both in Figure \ref{fig:pixel_histograms_dtm_elevation_detrended} and Table \ref{tab:dataset_stats_extended_elevation_centered}.
\begin{figure}[htbp]
     \centering
         \includegraphics[width=\textwidth,height=0.80\textheight,keepaspectratio]{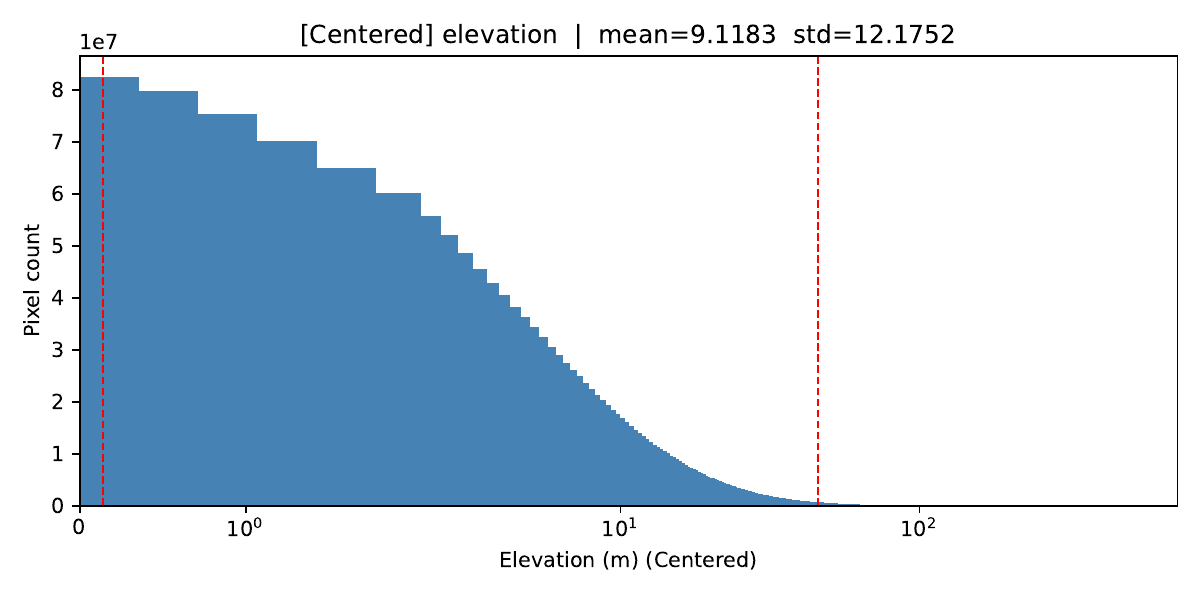}
        \caption{Per-pixel detrended absolute value elevation data of valid patches with a patch size of 0.018 degrees over 200 HiRISE DTM measurement readings. These patches are derived from the HiRISE DTM dataset and subsampled by filtering samples with a bounding box of $[-120\degree, -30\degree, 150\degree, 30\degree]$. The histograms are semi-log-scaled, unnormalized frequency histograms that provide aggregated information on approximately 1 billion pixels per channel (there are fewer pixels due to resolution subsampling). The red lines in the histograms represent the 2nd and 98th percentiles of the data.}
        \label{fig:pixel_histograms_dtm_elevation_detrended}
\end{figure}
\begin{table}[tbh]
\centering
\caption{Summary statistics of the absolute detrended residual $|Z-\mathrm{plane}|$ in meters. P2 and P98 represent the 2\% and 98\% percentiles. CV denotes the coefficient of variation ($\sigma/\mu$). DRCR denotes dynamic range compression ratio ($(P_{98}-P_{2})/\mu$).}
\label{tab:dataset_stats_extended_elevation_centered}

\begin{tabular}{
l
S[table-format=1.4e2]
S[table-format=3.4]
S[table-format=3.4]
}
\toprule
\textbf{Channel} & {\textbf{$N_{\text{pixels}}$}} & {\textbf{Mean}} & {\textbf{Std}} \\
\midrule
Elevation & 1.6092e+09 & 9.1183 & 12.1752 \\
\bottomrule
\end{tabular}
\vspace{1.5em}

\begin{tabular}{
l
S[table-format=3.4]
S[table-format=3.4]
S[table-format=3.4]
S[table-format=3.4]
S[table-format=1.3]
S[table-format=1.3]
}
\toprule
\textbf{Channel} & {\textbf{Min}} & {\textbf{Max}} & {\textbf{P2}} & {\textbf{P98}} & {\textbf{CV}} & {\textbf{DRCR}} \\
\midrule
Elevation & 0.0000 & 730.2763 & 0.1393 & 45.9075 & 1.335 & 5.019 \\
\bottomrule
\end{tabular}
\end{table}
The residual-magnitude distribution is strongly skewed: its 98th percentile is approximately 45.9\,m, while the maximum in the exported statistics is approximately 730.3\,m. Scaling solely by that maximum allocates a small fraction of the available numerical range to low-amplitude relief. This is a dynamic-range issue; a pointwise normalization does not itself create higher spatial resolution.

The residual statistics below refer to $|Z-\text{plane}|$ when shown on a positive magnitude axis. Negative coefficients of variation in the absolute-elevation table arise from division by a negative mean and have no datum-independent interpretation; all exported values are retained for audit.

\FloatBarrier
\subsubsection{Dataset normalization}

Orthoimage intensities are normalized independently for each patch using the 2nd and 98th percentiles of its positive pixels. Clipping limits the encoder input to $[-1,1]$:
\begin{align}
    o_{normalized} = \operatorname{clip}\left(2\frac{o - o_{2\%}}{o_{98\%} - o_{2\%}}-1, -1, 1 \right)
\end{align}
The anchors $o_{2\%}<o_{98\%}$ are computed from the current patch. A constant positive patch maps to zero. This normalization reduces sensitivity to patch-wide intensity range, but does not preserve radiometric comparability across overlapping patches.

Elevation residuals use a signed logarithm to increase numerical sensitivity near zero and compress the tails. A shared reference scale gives a consistent mapping across patches; cached examples must be interpreted using their recorded scale. The unclipped and clipped variants below separate the transformation itself from the information loss caused by saturation. The current configuration uses $S_{\mathrm{ref}}=45.9075$\,m without clipping; the reported experiment and the separate clipping-budget analysis use the conventions identified in the reproducibility discussion.

Unlike positive camera depth, detrended elevation is signed and centered near zero. We therefore use a symmetric logarithmic mapping, related to companding in waveform synthesis~\cite{Oord2016WaveNet:Audio} and symlog value transformations~\cite{HafnerMasteringModels}. For residual elevation $\hat Z(y,x)$ and positive scale $S_{\mathrm{ref}}$, the unclipped encoding is:
\begin{align}
    Z_{\text{norm}}(y, x) = \operatorname{sign}(\hat{Z}(y, x)) \log_2\left(1 + \frac{|\hat{Z}(y, x)|}{S_{\text{ref}}}\right)
\end{align}

At $|\hat Z|=S_{\mathrm{ref}}$, the transform reaches $\pm1$. Its derivative is $1/[\ln(2)(S_{\mathrm{ref}}+|\hat Z|)]$, so near-zero gain relative to linear scaling by the \emph{same} reference scale is $1/\ln2\approx1.443$. Larger reported gains use a different, maximum-based linear baseline. For the bounded representation, define $q=\operatorname{clip}(Z_{\mathrm{norm}},-1,1)$ explicitly. Values with $|\hat Z|>S_{\mathrm{ref}}$ then saturate and cannot be recovered uniquely.

The unclipped transform has a deterministic inverse using only the global reference scale. Applied to a clipped prediction, the same expression recovers residual relief within the representable interval and maps saturated endpoints to $\pm S_{\mathrm{ref}}$; it does not recover the original tail magnitudes:
\begin{align}
    \hat{Z}_{\text{pred}}(y, x) = \operatorname{sign}(\tilde{Z}_{\text{norm}}(y, x)) \cdot S_{\text{ref}} \cdot \left( 2^{|\tilde{Z}_{\text{norm}}(y, x)|} - 1 \right)
\end{align}

Absolute elevation can be reconstructed as $\hat Z_{\mathrm{pred}}+c_0^*x+c_1^*y+c_2^*$ only when the trend plane is available independently at deployment. A plane fitted to the reference DTM is available during training and reference-based diagnostics, but not from an otherwise unlabeled single image. Recovering or anchoring this missing trend, for example with an external coarse elevation product, is a separate reconstruction problem.

Clipping can flatten extreme scarps even when its dataset-average penalty is small. An image VAE is trained on a particular input distribution; this motivates a bounded encoding but does not establish a mathematical restriction that its output must lie in $[-1,1]$. Future work should compare clipping, alternative companding scales, and an adapted decoder using identical held-out scenes.

The selection of the reference scale $S_{\text{ref}}$ dictates the behavior of the Global Log Normalization; some statistics about choosing the normalization clipping information are provided in Figures \ref{fig:global_log_normalization_scaling} and \ref{fig:global_log_normalization_scaling_optimization}. For a bounded encoding, the design trades tail saturation against numerical sensitivity to small residuals; disabling clipping removes saturation but extends the input range. Rather than relying on an arbitrary selection, we formulate the choice of $S_{\text{ref}}$ as a formal constrained optimization problem.
\begin{figure}[htbp]
    \centering
    \includegraphics[width=1\linewidth,height=0.80\textheight,keepaspectratio]{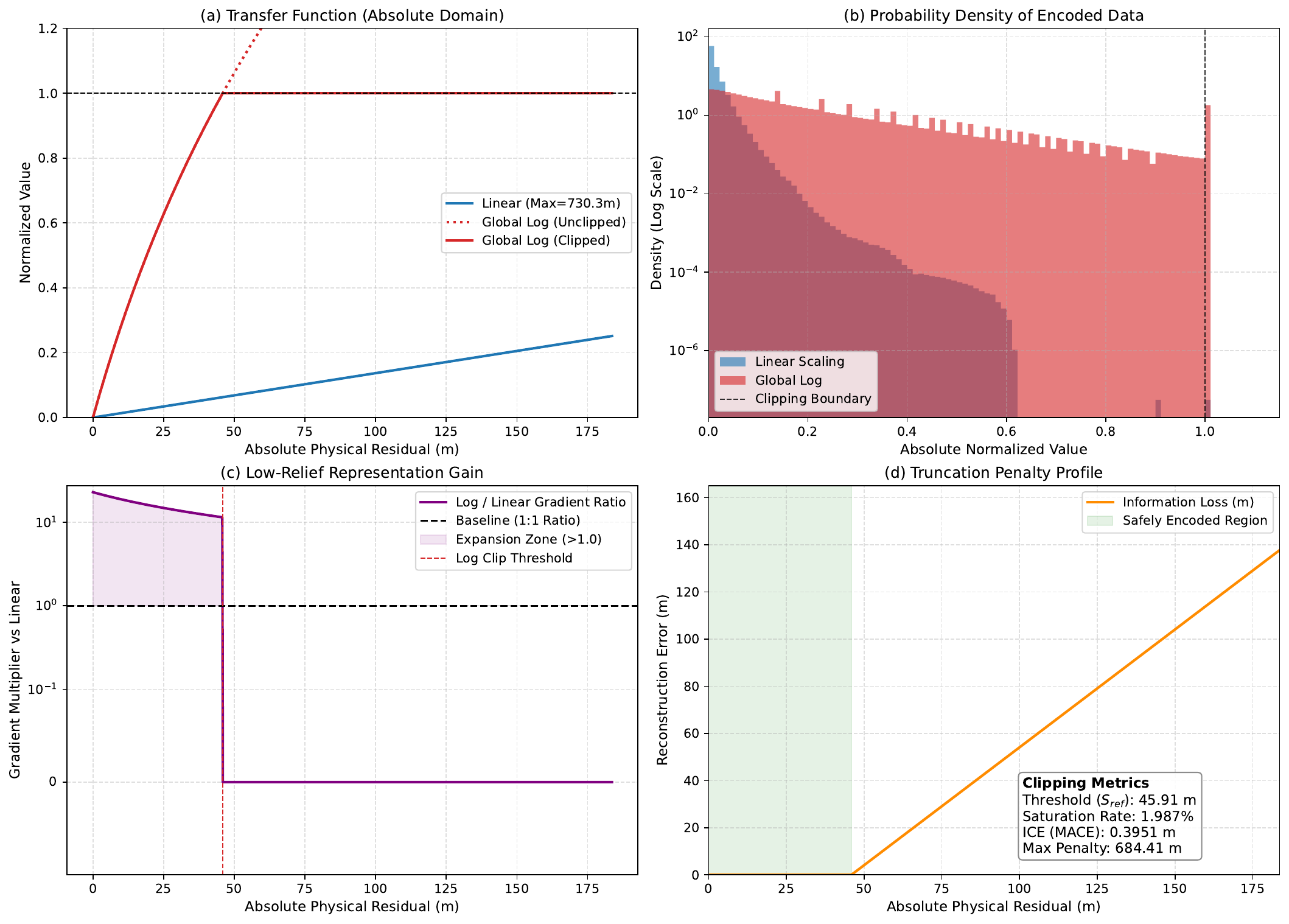}
    \caption{\textbf{Signed-log encoding and its clipping cost.} (a) Transfer functions using the illustrative 98th-percentile scale $S_{\mathrm{ref}}\approx45.9$\,m. (b) Encoded residual-magnitude densities. (c) Derivative gain relative to linear scaling by the observed maximum ($\sim730.3$\,m), not a spatial-resolution gain. (d) Meter-valued error from saturating extreme residuals. The complete study is retained to expose the tradeoff between central dynamic range and tail information loss.}
    \label{fig:global_log_normalization_scaling}
\end{figure}
\textbf{Mean Absolute Clipping Error (MACE)}: For bounded encoding, each residual whose magnitude exceeds $S_{\mathrm{ref}}$ loses $|\hat Z_i|-S_{\mathrm{ref}}$ meters on inversion. Over $N$ valid residual samples, the mean loss is:
\begin{align}
    C(S_{\text{ref}}) = \frac{1}{N} \sum_{i=1}^{N} \max(0, |\hat Z_i| - S_{\text{ref}})
\end{align}
\textbf{Base Expansion Factor}: Smaller reference scales allocate more numerical range to low-amplitude relief, with near-zero gain proportional to $1/S_{\mathrm{ref}}$, while increasing clipping. This is analogous to companding~\cite{Oord2016WaveNet:Audio}; the gain is not exponential in $1/S_{\mathrm{ref}}$ and does not certify recoverable geomorphic detail.

We quantify the numerical gain relative to $\mathcal F_{\mathrm{linear}}(x)=x/Z_{\max}$, where $Z_{\max}=\max_i|\hat Z_i|$, by the derivative ratio at zero:
\begin{align}
    B(S_{\text{ref}}) = \frac{\left. \frac{d}{dx} \mathcal{F}_{\text{log}}(x) \right|_{x \to 0}}{\left. \frac{d}{dx} \mathcal{F}_{\text{linear}}(x) \right|_{x \to 0}} = \frac{Z_{\max}}{\ln(2) \cdot S_{\text{ref}}}
\end{align}
\textbf{Constrained Optimization Objective}: Select the scale that maximizes near-zero numerical gain subject to a chosen average clipping-error budget $\epsilon=0.10$\,m. This tolerance is not a measured HiRISE noise floor and does not bound individual-pixel errors. Since both $B(S)$ and $C(S)$ are nonincreasing for $S>0$, the solution is the smallest feasible scale (or the smallest feasible candidate in a discrete sweep):
\begin{align}
  S_{\text{ref}}^* = \arg\max_{S_{\text{ref}}} B(S_{\text{ref}}) \quad \text{subject to} \quad C(S_{\text{ref}}) \le \epsilon
\end{align}
\begin{figure}[htbp]
    \centering
    \includegraphics[width=1\linewidth,height=0.80\textheight,keepaspectratio]{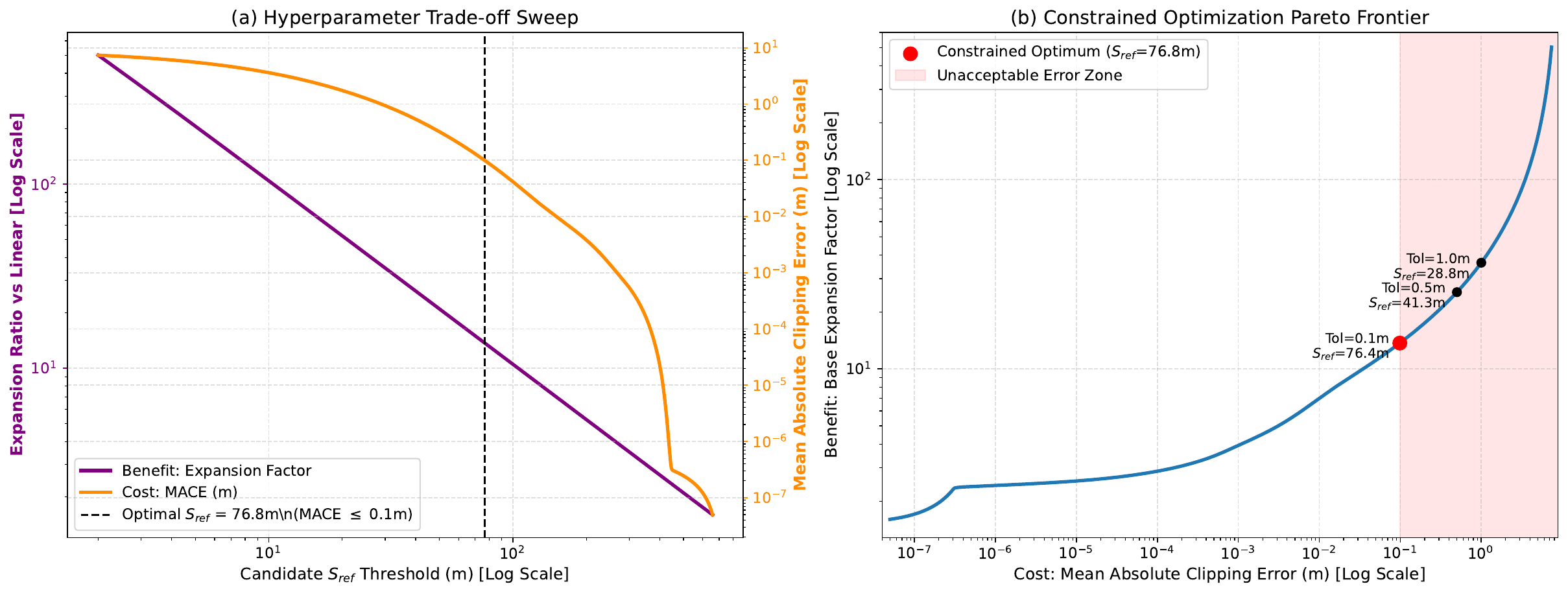}
    \caption{\textbf{Reference-scale selection under an average clipping budget.} The sweep relates numerical gain $B(S)$ to meter-valued mean absolute clipping error $C(S)$. The marked scale is approximately 76.8\,m for a 0.1\,m budget. The constraint limits an average over the inspected residual sample; it does not ensure small errors on individual cliffs or establish an instrument-precision threshold.}
    \label{fig:global_log_normalization_scaling_optimization}
\end{figure}
Figure~\ref{fig:global_log_normalization_scaling_optimization} marks an optimum near 76.8\,m. The archived table below instead reports 76.4\,m, with MACE 0.0987\,m and gain 13.69. Both exports are retained because their sample definitions and rounding must be reconciled from the original statistics before the training value can be certified. The qualitative tradeoff is unchanged, but these values should not be silently treated as identical.
\begin{table}
    \centering
        \caption{Study of the $S_{ref}$ parameters based on the different scale using the constrained Pareto optimization routine or using the data's 98\% range.}
    \begin{tabular}{cccc}
    \toprule
       \textbf{ $S_{ref}$} &  \textbf{\makecell{Saturation Rate \\ (Data lost to tails)}} (\%)& \textbf{\makecell{Irreducible Clipping \\ Error (MACE)} (m)}& \textbf{Base Expansion Factor}\\
        \midrule
        98\% (45.9)&  1.9871\% &  0.39506& 22.92\\
       Optimal (76.4m)&  0.3986\%&  0.0987& 13.69\\
       \bottomrule
    \end{tabular}
    \label{tab:placeholder}
\end{table}
\FloatBarrier
\subsubsection{Stereo-pair alignment and resolution scaling}

DTM products and their corresponding ortho images are explicitly linked in the pipeline using a canonical stereo-pair key. This key is derived by concatenating the left and right observation IDs (e.g., \texttt{ESP\_011265\_1560} and \texttt{ESP\_011331\_1560}). Because a single physical observation can participate in multiple stereo pairs, this many-to-many mapping ensures the consistent alignment of the topographic data with its corresponding visual counterparts.

Furthermore, both DTMs and orthos embed a grid-spacing identifier in their product strings to denote pixel scale (e.g., `A` for 0.25 m/pixel, `C` for 1.00 m/pixel). This allows the dataset pipeline to dynamically filter and select the finest available scale for localized patch extraction \cite{HiRISEModels}.

\FloatBarrier
\subsubsection{Physical units and nodata handling}

Each DTM pixel is a scalar elevation in meters, representing either the areoid elevation (relative to the Mars MOLA-derived geoid) or the planetary radius (distance from the center of Mars).

Instead of using a zero-value sentinel for nodata, DTM rasters encode missing data using the IEEE 754 single-precision minimum ($\approx -3.40 \times 10^{38}$). During the loading pipeline, these sentinels are explicitly converted to Not-a-Number (NaN) masks. This conversion is strictly enforced to ensure that shadowed regions, data gaps, or unmapped boundaries do not artificially skew normalization statistics or inadvertently bias downstream loss calculations during model training.

\FloatBarrier
\subsubsection{Known topographic artefacts}

HiRISE DTMs contain several recognized structural artifacts stemming from the physical instrument or the automated stereo correlation process \cite{Kirk2009UltrahighSites}. When training deep learning models on this dataset, we must account for these phenomena \cite{HiRISEModels}:
\begin{itemize}
    \item \textbf{SOCET Set ``Boxes'':} Square regions offset vertically by 0.5--1 m, resulting from automated correlation boundaries in the processing software. These introduce step-function label noise in otherwise smooth terrains.
    \item \textbf{CCD Seams:} Sub-meter vertical jumps running along lines at regular cross-track intervals due to residual mis-registration between the 10 constituent CCDs.
    \item \textbf{High-frequency Jitter:} Spacecraft jitter that introduces cross-track elevation ripples with amplitudes of 10--50 cm, often appearing as high-pass noise superimposed on the true topography.
    \item \textbf{Long-baseline Tilts:} Gradual drift in the stereo model producing kilometer-scale undulations. This drift can affect absolute elevation predictions at large scales but generally preserves local relative topography.
\end{itemize}
Because these artifacts cause significant problems for the neural network's training process, a pipeline was developed to filter out structural artifacts, as shown in Figure \ref{fig:manifest_filter}, especially the CCD seams, which result from the merging of multiple orthorectified and DTM geospatially adjacent images.
\begin{figure}[htbp]
    \centering
    \includegraphics[width=1\textwidth,height=0.80\textheight,keepaspectratio]{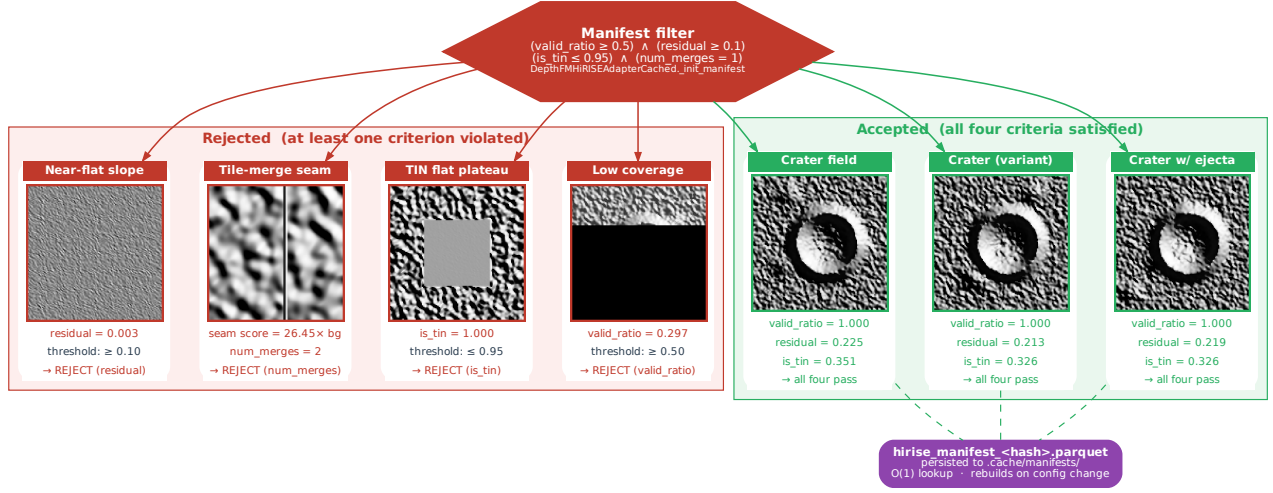}
    \caption{Pipeline and criteria for filtering the dataset elements to ensure the dataset can serve as a training dataset for a neural network; the valid components are cached in a manifest file for faster access down the line. The first filter, the \textbf{near-flat slope}, removed patches where the average residual error from the detrended patch was close to zero, representing featureless DTM. The \textbf{tile-merge seam} aims to remove CCD seam artifacts and merged DTM, but it produces non-continuous transition edges along the seams, which throw neural networks off. \textbf{TIN (Triangulated irregular network) flat plateau} represents the potential faceted and manually interpolated areas, which are results of terrain that was very bland (low contrast) or deeply shadowed with low contrast and low signal, which causes the areas to only generally approximate the shape of the terrain \cite{HiRISEModels}. Finally, low-coverage is an optional non-zero data percentage filter, if not handled as a parameter of the HiRISE patch sampler; see Appendix \ref{sec:sampling}. }
\label{fig:manifest_filter}
\end{figure}
\FloatBarrier
\subsubsection{Merging Seam detection}

Mosaic seams can introduce abrupt radiometric or elevation changes that do not correspond to terrain. The detector below combines sharp gradients, local distribution shifts, and line geometry to rank candidate seams. Natural ridges, shadows, and crater rims can produce similar responses, so the score is a screening heuristic. Its thresholds require labeled acceptance/rejection checks, including retained natural features, before it can be interpreted as a reliable quality filter.

The pipeline for the seam detection is illustrated in Figure \ref{fig:seam_filter},
\begin{figure}[htbp]
    \centering
    \includegraphics[height=0.80\textheight,width=\linewidth,keepaspectratio]{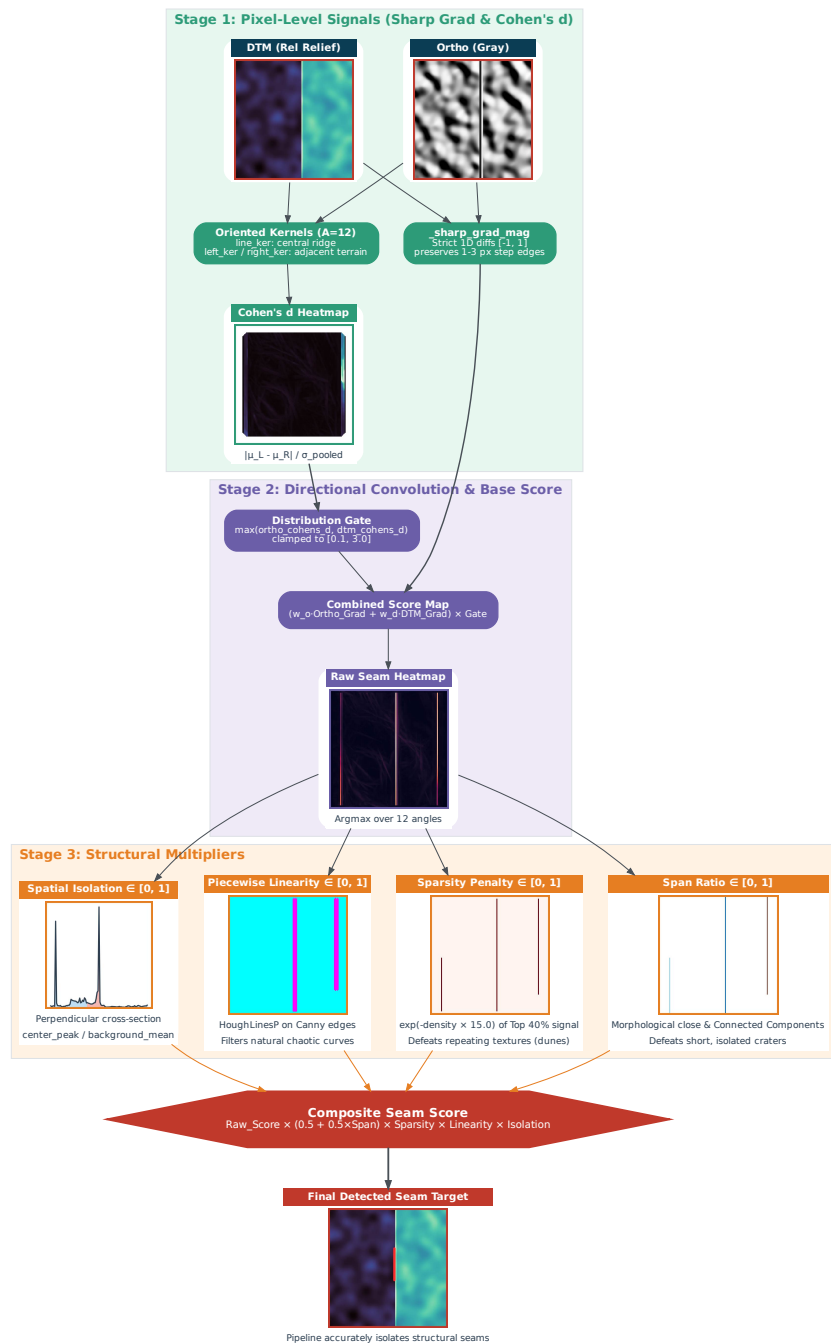}
    \caption{\textbf{Automated Seam Artifact Detection Pipeline.} The proposed architecture identifies unnatural mosaicking boundaries through a three-stage process. \textbf{Stage 1} computes unblurred step-gradients and calculates a statistical distribution shift (Cohen's $d$) between adjacent terrain blocks for both the orthoimage and DTM. \textbf{Stage 2} aggregates these pixel-level signals using multi-angle directional convolutions gated by the statistical shift, producing a base anomaly heatmap. \textbf{Stage 3} suppresses false positives from natural topography (e.g., craters and dunes) by applying four geometric penalty multipliers: Span Ratio, Sparsity, Piecewise Linearity (via Hough Transform), and Spatial Isolation. The composite score ranks seam candidates; this diagram does not establish detection accuracy.}
\label{fig:seam_filter}
\end{figure}
Mosaicking seams are artificial linear discontinuities introduced during the generation of planetary orthoimages and Digital Terrain Models (DTMs). Because these artifacts can severely corrupt the latent space of generative models by introducing non-physical edges, they must be robustly identified and filtered. We propose a multi-stage, directionally-aware detection pipeline intended to distinguish artificial seams from natural geological features (such as craters, ridges, and dune fields) by evaluating both local statistical shifts and global geometric structure.

\textbf{Stage 1: Pixel-Level Signals and Statistical Shifts}

The pipeline begins by analyzing the high-frequency spatial gradients of both the orthoimage and the DTM. To prevent the spatial blurring inherent to standard $3 \times 3$ operators (e.g., Sobel), we employ strict 1D forward-difference kernels ($[-1, 1]$), which retain narrow step responses without the smoothing of a Sobel operator; natural edges can produce the same response.

\textit{Sharp Gradient Magnitude}:
The 1D forward difference gradient magnitude, designed to preserve exact 1-to-3 pixel step edges without spatial blurring, is computed as:
\begin{equation}
G = \sqrt{(\nabla_x I)^2 + (\nabla_y I)^2 + \epsilon}
\end{equation}
Where $\nabla_x I$ and $\nabla_y I$ are the 1D discrete forward differences (convolutions with $[-1, 1]$ and $[-1, 1]^T$) and $\epsilon = 10^{-8}$.

\textit{Pooled Standard Deviation:}
Simultaneously, the algorithm evaluates the statistical divergence across potential seam boundaries using oriented kernels at multiple angles. To calculate the statistical shift \cite{Cohen2013StatisticalSciences}, the pooled standard deviation incorporates a global variance floor to prevent division by zero in homogeneous regions:
\begin{equation}
\sigma_{pooled} = \sqrt{\frac{\sigma_L^2 + \sigma_R^2}{2}} + \sigma_{floor}
\end{equation}
Where $\sigma_L^2$ and $\sigma_R^2$ are the local variances of the adjacent terrain blocks, and $\sigma_{floor}$ is scaled to 10\% of the global standard deviation.

\textbf{Stage 2: Directional Convolution and Base Score}

The base-seam heatmap is constructed by projecting the combined gradient magnitudes onto a set of discrete orientations (e.g., $12$ angles) using directional convolutions. This aggregated signal is then modulated by a distribution gate derived from the maximum Cohen's $d$ values for both the orthoimage and the DTM. The gate weights the directional response by a local distribution shift; its lower clip of 0.1 leaves a nonzero contribution even when that shift is small.

\textit{Base Directional Score ($S_{base}$):}
A distribution gate modulates the combined directional gradient magnitude:
\begin{equation}
S_{base} = \left( w_{ortho} \frac{\bar{G}_{ortho,\theta}}{B_{ortho}} + w_{dtm} \frac{\bar{G}_{dtm,\theta}}{B_{dtm}} \right) \times \text{clip}(\max(d_{ortho}, d_{dtm}), 0.1, 3.0)
\end{equation}
Where $\bar{G}_{...,\theta}$ is the average gradient magnitude along the oriented line kernel at angle $\theta$, $B_{...}$ is the background average gradient, $w$ represents the modality weights, and $d$ represents the computed Cohen's $d$ shifts.

\textbf{Stage 3: Structural and Geometric Multipliers}

Natural topography frequently produces sharp, shifted gradients (e.g., crater rims or dune slip faces). To filter out these false positives, the base heatmap is scaled by four structural multipliers bounded in $[0,1]$; statistical independence is not assumed:

\textit{Sparsity Penalty ($M_{sparsity}$):}
Natural repeating textures, such as dune fields, generate dense clusters of high gradient signals. A true seam, however, represents a singular, sparse line comprising a negligible fraction of the total area. The sparsity multiplier heavily penalizes dense signal regions using an exponential decay function. Modeled as an exponential decay function based on the density of the high-scoring signal:
\begin{equation}
M_{sparsity} = \exp(-15.0 \cdot \rho)
\end{equation}
Where $\rho = \frac{\text{Area}_{hot}}{\text{Area}_{valid}}$ represents the ratio of strongly activated pixels to the total valid area.

\textit{Span Ratio ($M_{span}$):}
True seams typically span the entire length of a data tile, whereas craters are localized. We isolate the core signal and apply morphological closing followed by connected components analysis. The span ratio represents the physical length of the longest continuous feature relative to the tile dimensions. The implemented proxy is the diagonal of a connected-component bounding box divided by the maximum tile dimension, measured in pixels:
\begin{equation}
M_{span} = \min\left( \frac{\max_i \sqrt{w_i^2 + h_i^2}}{\max(H, W)}, 1.0 \right)
\end{equation}
Where $w_i$ and $h_i$ are the dimensions of the bounding box for the $i$-th connected component.

\textit{Piecewise Linearity ($M_{linear}$):}
Seams are highly structured and piecewise linear, even when sharp corners are accounted for. The ratio of high-scoring pixels that successfully map to structured lines via the Probabilistic Hough Transform \cite{Matas2000RobustTransform} applied to Canny edges \cite{Canny1986ADetection}:
\begin{equation}
M_{linear} = \frac{| P_{signal} \cap P_{Hough} |}{| P_{signal} |}
\end{equation}
Where $P_{signal}$ is the set of binary thresholded high-gradient pixels, and $P_{Hough}$ is the mask of pixels reconstructed from the detected Hough lines.

\textit{Spatial Isolation ($M_{isolation}$):}
Seams are geometrically isolated anomalies. Comparing the perpendicular cross-section seam peak to the parallel background terrain:
\begin{align}
R_{iso} &= \frac{\max(P_{center})}{\mu_{bg}} \\
M_{isolation} &= \text{clip}\left(\frac{R_{iso} - 1.5}{1.5}, 0, 1\right)
\end{align}
Where $\max(P_{center})$ is the peak signal value within the central exclusion zone, and $\mu_{bg}$ is the mean signal of the parallel background regions.

\textbf{Final Composite Score}
The final gating metric used to determine if a tile contains an artificial seam:
\begin{equation}
S_{final} = S_{base} \cdot \left(0.5 + 0.5 \cdot M_{span}\right) \cdot M_{sparsity} \cdot M_{linear} \cdot M_{isolation}
\end{equation}
To make these expressions well-defined, background-gradient and background-intensity denominators require a positive numerical floor. Empty signal or component sets require an explicit zero-score convention; a tile with no valid pixels is excluded. These implementation guards must be recorded with the thresholds. Each component is visualized in Figure \ref{fig:dataset_architecture_dtm_seam}.
\begin{figure}[htbp]
    \centering
    \includegraphics[width=1\textwidth,height=0.80\textheight,keepaspectratio]{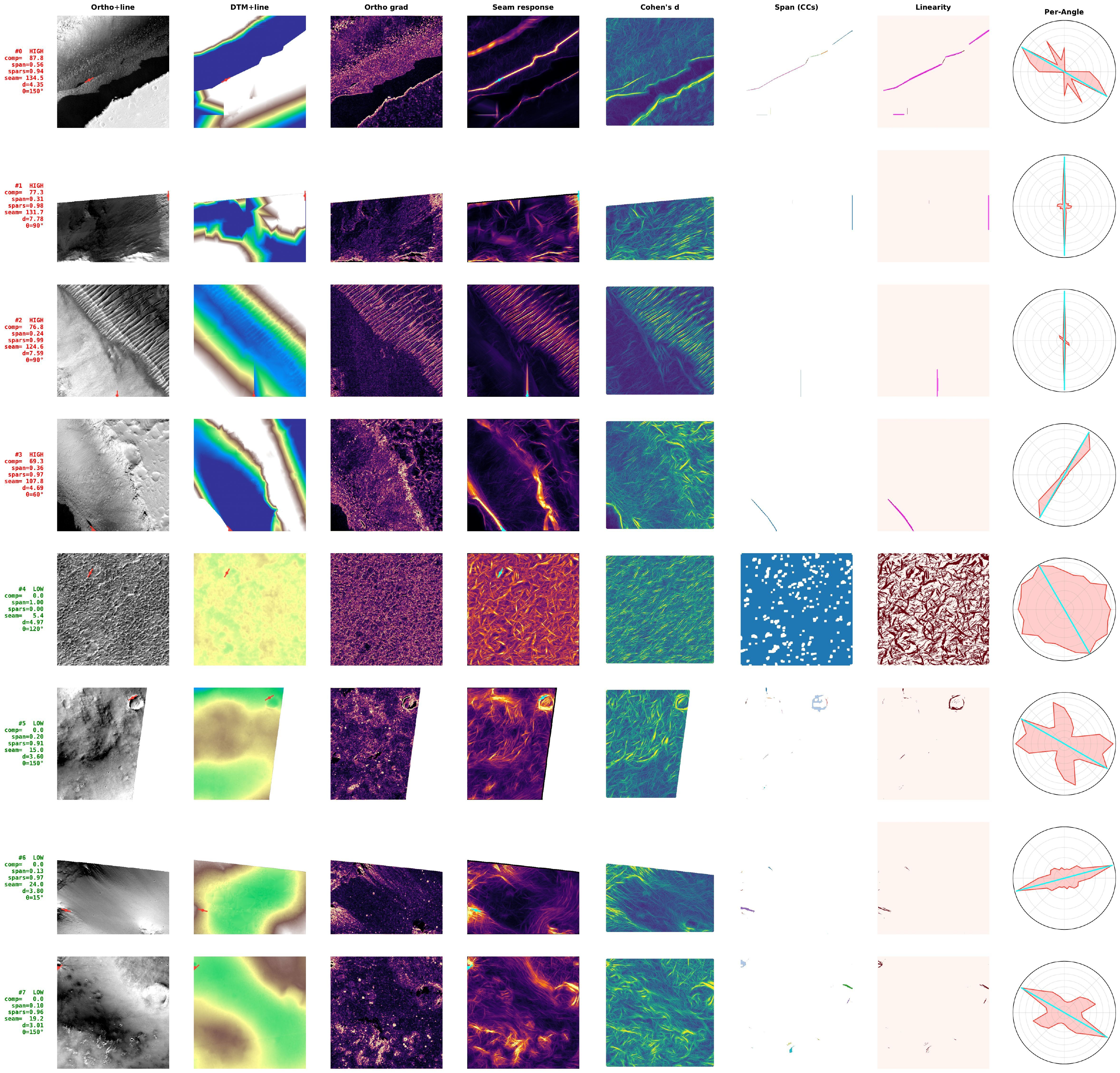}
    \caption{Diagnostic visualization grid for automated seam artifact detection. Each row represents an evaluated scan region comprising eight analytical panels: (1) \textbf{Ortho+line}: RGB orthomosaic with the best-scoring candidate line superimposed; (2) \textbf{DTM+line}: Digital Terrain Model (DTM) with candidate line; (3) \textbf{Ortho grad}: Sharp gradient magnitude utilizing 1D difference kernels to emphasize exact pixel step edges; (4) \textbf{Seam response}: Heatmap of the directional seam scores; (5) \textbf{Cohen's $d$}: Heatmap representing localized distribution shifts between terrain separated by the candidate line; (6) \textbf{Span (CCs)}: Connected components map of the core signal, utilized to evaluate structural span and filter isolated anomalies; (7) \textbf{Linearity}: Probabilistic Hough transform lines overlaid on a sparsity hot-mask to quantify piecewise linearity; and (8) \textbf{Per-Angle}: A polar plot illustrating the maximum response as a function of orientation angle $\theta$. Textual annotations (left) detail the calculated structural multipliers, composite score, span ratio, sparsity, raw seam score, Cohen's $d$, and optimal angle $\theta$, classifying the sample as \textit{HIGH} or \textit{LOW} by the heuristic threshold. These labels are not calibrated probabilities.}
\label{fig:dataset_architecture_dtm_seam}
\end{figure}
%
%

\FloatBarrier
\subsection{Dataset subset}

The inspected DTM catalog contains 11,692 terrain-related records supplied through the PDS by multiple production groups~\cite{HiRISEModels}. Associated orthoimages, resolution variants, and terrain rasters must be grouped by terrain product and stereo observation before counting independent DTMs. The 652.495\,MB JPEG2000 mean reported for the RDR catalog is not a separately measured DTM-corpus file size. Figure~\ref{fig:toptographic_map_dtm} retains the full catalog and selected-region coverage for this snapshot.
\begin{figure}[htbp]
    \centering
    \includegraphics[width=1\linewidth,height=0.80\textheight,keepaspectratio]{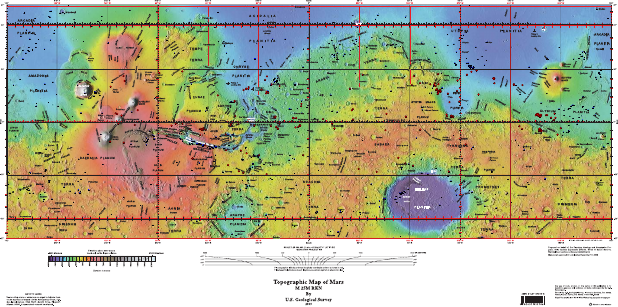}
    \caption{Topographic Map of Mars \cite{Topographic2782} through the view of Mercator Projection with longitude range of $-180^{\circ}$ to $180^{\circ}$ and latitude range of approximately $-57^{\circ}$ to $57^{\circ}$. The map was generated by the U.S. Geological Survey, and on the map were overlaid all the HiRISE sample locations in \textcolor{blue}{blue} within that topographical range, and in \textcolor{red}{red} is the subsampled dataset used as the scope of the project. The \textcolor{red}{sub-sampled} dataset was selected by filtering samples with a bounding-box of $[-120^\circ, -30^\circ, 150^\circ, 30^\circ]$.}
    \label{fig:toptographic_map_dtm}
\end{figure}
The selected bounding box $[-120^{\circ},-30^{\circ},150^{\circ},30^{\circ}]$ reduces the catalog from 11,692 to 5,117 records and limits distortion in the pipeline's geographic-grid representation. This is a processing choice, not a claim that high-latitude HiRISE products are unusable. Polar or locally metric projections can support those regions. The filtered record count is distinct from the number of unique DTMs, inspected products, or training patches.

\FloatBarrier
\subsection{Pipeline workflow}

The full pipeline workflow showing how the dataset reader and sampler work together is visualized in Figure \ref{fig:dataset_architecture_dtm}.
\begin{figure}[htbp]
    \centering
    \includegraphics[width=1\textwidth,height=0.80\textheight,keepaspectratio]{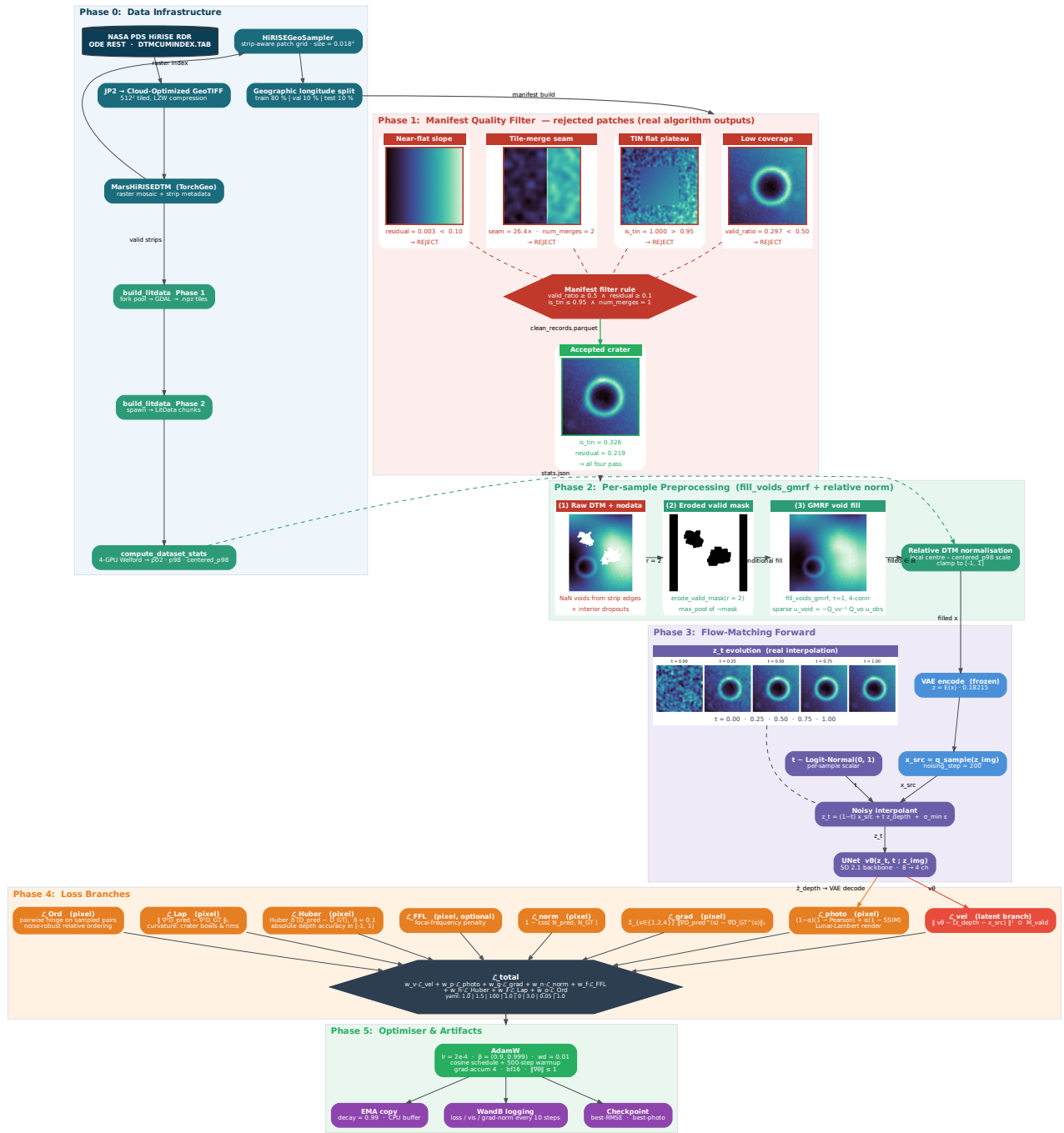}
    \caption{Full DTM pipeline architecture overview of the Mars DTM HiRISE dataset loader and sampler}
\label{fig:dataset_architecture_dtm}
\end{figure}
Missing pixels require finite values before VAE encoding. Replacing them with zero can create an artificial edge against neighboring valid relief or image intensity. We therefore use smooth infilling while retaining the observed-data mask for supervision and evaluation. Filling makes the input numerically complete; it does not add measured terrain or guarantee freedom from boundary artifacts.

\FloatBarrier
\subsection{GMRF Hole Infilling} \label{sec:gmrf_hole}

A Gaussian Markov random field (GMRF) provides a local quadratic smoothness model for missing pixels~\cite{Markov2005GaussianApplications}. Conditioning this model on the valid boundary yields a harmonic-style fill that reduces abrupt zero-padding transitions. The model is intentionally simple and does not reconstruct unobserved craters, dunes, or texture from physical evidence.

Before infilling, the validity mask is eroded by a configured radius to exclude potentially unreliable boundary pixels. This trades observed support for smoother boundary conditioning; erosion alone cannot establish that the remaining data are accurate. The radius and the evaluation mask must therefore be recorded separately.

Following standard GMRF formulations for spatial data \cite{Markov2005GaussianApplications}, we define the image domain as a grid graph parameterized by a precision matrix $Q = \tau L$, where $L$ is the sparse graph Laplacian and $\tau > 0$ dictates the strength of spatial smoothing. The optimal values for the missing regions ($u_v$) are obtained by computing the mean of the conditional distribution given the observed valid pixels ($u_o$).To guarantee numerical stability during the sparse linear solve, particularly when void regions are vast, completely isolated, or subject to floating-point imprecision, we modify the standard conditional solve by introducing a small positive diagonal regularization term ($\epsilon \approx 10^{-6}$). This is conceptually similar to a "nugget" effect in kriging and geostatistics \cite{Cressie2015StatisticsEdition}. The regularized conditional solve is therefore formulated as:
\begin{align}
    u_v = -(Q_{vv} + \epsilon I)^{-1} Q_{vo} u_o
\end{align}
where the subscripts $v$ and $o$ denote the block-partitioned indices for the void and observed pixels, respectively, and $I$ is the identity matrix.

In computation, solve $(Q_{vv}+\epsilon I)u_v=-Q_{vo}u_o$ with a sparse solver rather than forming an inverse. Without regularization, the conditional mean is unique when every void component is connected to observed support; otherwise the intrinsic Laplacian can be singular. Positive $\epsilon$ makes the void block invertible but biases the solution toward zero. The often-cited $\mathcal O(N^{3/2})$ direct-solve cost applies to suitable two-dimensional grids and orderings, not arbitrary masks or solvers. Figure~\ref{fig:smooth_fill_inspection} retains the complete set of fill examples. Boundary artifacts persisted in training, so the method should be assessed as mitigation rather than an artifact-free solution. The Gaussian assumption describes the spatial prior, not a requirement that the empirical elevation histogram itself be Gaussian.
\begin{figure}[htbp]
\centering
\includegraphics[height=0.80\textheight,width=\linewidth,keepaspectratio]{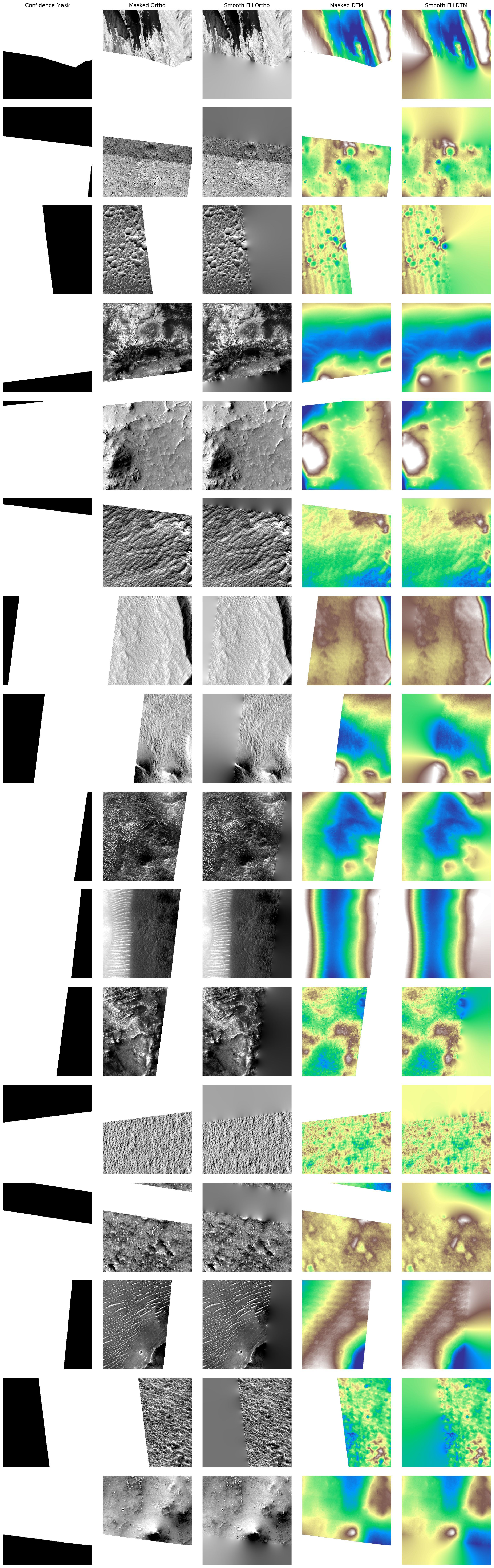}
\caption{Diagnostic visualization of the smooth diffusion infilling process prior to VAE optimization. Each row illustrates a single sample containing missing or invalid data, processed across five comparative panels: (1) \textbf{Confidence Mask}: A binary representation delineating valid observations from void regions; (2) \textbf{Masked Ortho}: The original RGB orthomosaic with invalid regions explicitly masked; (3) \textbf{Smooth Fill Ortho}: The resulting orthomosaic after applying a Gaussian Markov Random Field (GMRF) conditional solve, which smoothly diffuses boundary pixels to infill the voids; (4) \textbf{Masked DTM}: The corresponding Digital Terrain Model (DTM) with missing elevation data isolated; and (5) \textbf{Smooth Fill DTM}: The elevation map reconstructed via the identical GMRF sparse linear solve. The fill reduces zero-padding discontinuities, but the original mask still identifies unobserved values and the panels retain visible failure cases.}
\label{fig:smooth_fill_inspection}
\end{figure}
\FloatBarrier
\subsection{Code example}

Provided below is an example workflow of how to use the TorchGEO dataset in the code block \ref{lst:hirise_pipeline_dtm}.
\begin{figure}[htbp]
\begin{lstlisting}
global_coverage = True
viz_3d = True

bbox_tuple = (-120, -30, 150, 30)

dataset = MarsHiRISEDTM(
    target=None,
    bbox=bbox_tuple,
    include_ortho=True,
    ortho_type=["RED"],
    download=False,
    reuse_cache=True,
)

output_path = pathlib.Path("outputs/Figures_DTM")
output_path.mkdir(parents=True, exist_ok=True)

fig = dataset.plot_coverage()
fig.savefig(output_path / "dtm_coverage.png", bbox_inches='tight')
logger.info("Saved coverage figure.")

if global_coverage:
    dataset.plot_global_coverage(output_path / "global_coverage.pdf")
    logger.info("saved global coverage fig")

sampler = HiRISEGeoSampler(
    dataset, size=0.018,
    length=None,
    units=Units.CRS,
    center_mode="optimal"
)

# Saves memory
dataset._raw_index = None

logger.info("Number of samples: %d", len(sampler))

dataloader = DataLoader(
    dataset, sampler=sampler,
    num_workers=8, multiprocessing_context="fork", prefetch_factor=32,
    persistent_workers=True
)

output_path_3d = output_path / '3d'

for i, sample in enumerate(dataloader):
    if viz_3d:
        output_path_3d.mkdir(parents=True, exist_ok=True)
        fig = dataset.plot3d(sample)
        fig.savefig(output_path_3d / f"dtm_output{i}_3d.png")
        logger.info("Saved dtm_output%d_3d.png", i)
        plt.close(fig)

    output_path_flat = output_path / 'flat'
    output_path_flat.mkdir(parents=True, exist_ok=True)
    fig = dataset.plot(sample)
    fig.savefig(output_path_flat / f"dtm_output{i}.png")
    logger.info("Saved dtm_output%d.png", i)
    plt.close(fig)
\end{lstlisting}
\caption{Geospatial dataloading and patch extraction pipeline for the Mars HiRISE DTM dataset.}
\label{lst:hirise_pipeline_dtm}
\end{figure}
\newpage
\FloatBarrier
\section{Loss Formulation}
\label{sec:losses}

\FloatBarrier
\subsection{Overview and Design Rationale}
\label{sec:loss_overview}

Mars terrain supervision differs from terrestrial benchmarks such as KITTI~\cite{Geiger2012CVPR} and NYU-v2~\cite{Silberman:ECCV12} in acquisition geometry, scale, and reference construction. HiRISE DTMs are stereo-derived estimates with registration errors, shadow-related gaps, triangulation artifacts, and interpolated regions~\cite{Kirk2009UltrahighSites,Beyer2018TheData}. Their posting is also commonly coarser than the image grid: a representative pair has 0.25\,m imagery and 1\,m terrain posting~\cite{McEwen2007MarsHiRISE,Kirk2009UltrahighSites}. Neither posting alone nor visible image texture establishes terrain resolution, but this disparity motivates combining geometric references with a separate image constraint.

A reference-only objective cannot supply independent measurements of terrain detail absent from that reference. We therefore combine the following complementary objectives. They overlap in function and may conflict: matching an interpolated coarse reference can penalize image-derived variation even where the reference contains no independent fine-scale measurement.
\begin{enumerate}
\item \textbf{GT-matching losses} that enforce correctness at the GT resolution (flow matching, Section \ref{sec:fm_loss}, Huber, Section \ref{sec:huber_loss}, gradient, Section \ref{sec:grad_loss}, Laplacian, Section \ref{sec:laplacian_loss}, normals, Section \ref{sec:normal_loss}, ordinal ranking, Section \ref{sec:ordinal_loss}, focal-frequency, Section \ref{sec:ffl_loss}).
\item \textbf{Image-matching losses} that compare image information on the processed grid by rendering the predicted topography back into a shaded image and comparing it to the observation (photoclinometric/shape-from-shading loss, Section \ref{sec:photo_loss}).
\item \textbf{Noise-robust losses} that supervise on quantities which degrade gracefully under stereo artifacts (Huber, Section \ref{sec:huber_loss}, ordinal ranking, Section \ref{sec:ordinal_loss}, constant-shift-invariant gradient loss, Section \ref{sec:grad_loss}).
\end{enumerate}
The total training objective is a weighted sum:
\begin{equation}
\begin{aligned}
\mathcal{L}_{\text{total}} = \;&\lambda_{\text{FM}}\,\mathcal{L}_{\text{FM}}
+ \lambda_{\text{huber}}\,\mathcal{L}_{\text{huber}}
+ \lambda_{\text{norm}}\,\mathcal{L}_{\text{norm}}
+ \lambda_{\text{grad}}\,\mathcal{L}_{\text{grad}} \\
&+ \lambda_{\text{lap}}\,\mathcal{L}_{\text{lap}}
+ \lambda_{\text{FFL}}\,\mathcal{L}_{\text{FFL}}
+ \lambda_{\text{ord}}\,\mathcal{L}_{\text{ord}}
+ \lambda_{\text{photo}}\,\mathcal{L}_{\text{photo}}.
\end{aligned}
\label{eq:total_loss}
\end{equation}
The implementation exposes a configurable start step for each auxiliary term. In the available large-run configuration, all seven nonzero objectives are active from step zero; focal-frequency loss has zero weight. No successful staged curriculum is established by these settings. The formulation below retains all eight terms and distinguishes the encoded training reference from the raw evaluation reference.
\begin{figure}[htbp]
\centering
\includegraphics[height=0.80\textheight,width=\linewidth,keepaspectratio]{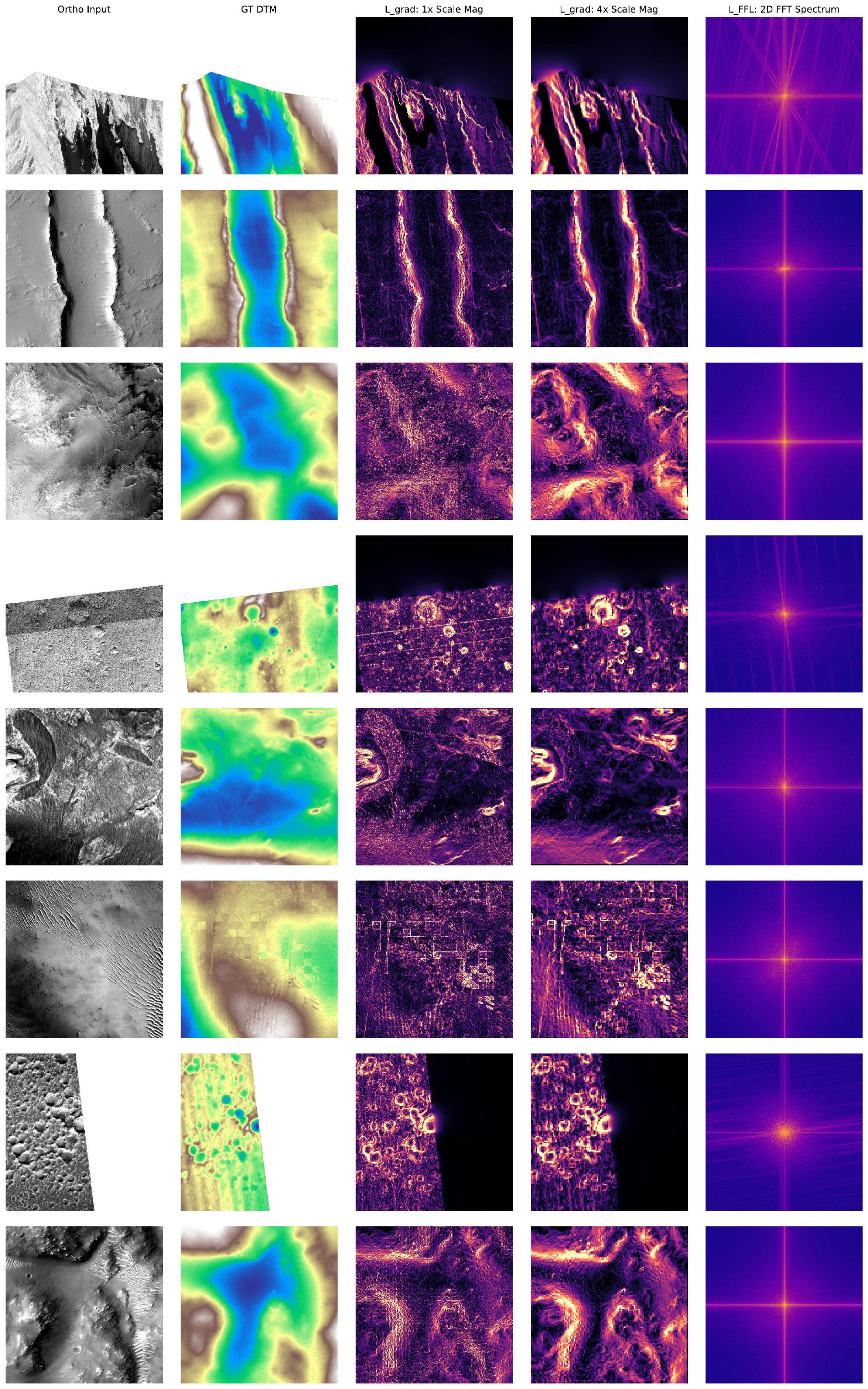}
\caption{\textbf{Intermediate fields for the pixel-space objectives.} Each row shows an archived sample: (a) input orthoimage; (b) stereo-reference DTM; (c) gradient magnitude at the processed grid scale; (d) gradient magnitude after $4\times$ pooling; and (e) log-magnitude Fourier spectrum. Pooling exposes broader spatial variation, while the spectrum summarizes frequency content without feature location. The optional FFL weights large spectral discrepancies rather than exclusively high frequencies; it is disabled in the reported run. These diagnostic fields document the operators, not measured improvements from each loss.}
\label{fig:loss_components}
\end{figure}
\FloatBarrier
\subsection{Primary Objective: Conditional Flow Matching Loss ($\mathcal{L}_{\mathrm{FM}}$)}
\label{sec:fm_loss}

We use Conditional Flow Matching (CFM)~\cite{Lipman2022FlowModeling,Liu2022FlowFlow,Tong2023ImprovingTransport,Gui2024DepthFM:Matching} as the primary latent-space objective. Let $\mathbf z_{\mathrm{img}}=\mathcal E(\mathbf I)$ and $\mathbf z_1=b\mathcal E(q)$, where $q$ is signed-log relief and the available configuration uses signal gain $b=1$. Encoding uses the frozen VAE posterior mode. If source noising is enabled, $\mathbf z_0=\sqrt{\bar\alpha_k}\mathbf z_{\mathrm{img}}+\sqrt{1-\bar\alpha_k}\boldsymbol\eta$ with $\boldsymbol\eta\sim\mathcal N(0,I)$ and the checkpoint's cosine-schedule noising step $k$; otherwise $\mathbf z_0=\mathbf z_{\mathrm{img}}$. The training path is $\mathbf z_t=(1-t)\mathbf z_0+t\mathbf z_1+\sigma_{\min}\boldsymbol\epsilon$. The available configuration uses $t=\operatorname{sigmoid}(u)$, $u\sim\mathcal N(0,1)$, numerically restricted to $[10^{-5},1-10^{-5}]$, and $\sigma_{\min}=10^{-4}$. Conditional on the sampled endpoints and fixed $\boldsymbol\epsilon$, the path velocity is $\mathbf z_1-\mathbf z_0$:

\begin{equation}
\mathcal{L}_{\text{FM}} = \mathbb{E}_{t,\mathbf I,q,\boldsymbol\eta,\boldsymbol\epsilon}
\Big[\,
\left\lVert v_\theta(\mathbf{z}_t, t;\,\mathbf{z}_{\text{img}})
- (\mathbf z_1-\mathbf z_0)\right\rVert_2^2
\,\Big].
\label{eq:cfm}
\end{equation}
The clean image latent remains the conditioning signal even when the path source is noised. The endpoint estimate used by pixel-space objectives is
\begin{equation}
\widetilde{\mathbf z}_1=\mathbf z_t+(1-t)v_\theta(\mathbf z_t,t;\mathbf z_{\mathrm{img}}),\qquad d=\mathcal D(\widetilde{\mathbf z}_1).
\label{eq:decoded_endpoint}
\end{equation}
The VAE weights are frozen, while gradients through $\mathcal D$ remain enabled for the prediction. This is a single-time endpoint estimate, not a training-time ODE solve. Even with the exact conditional target velocity, the estimate equals $\mathbf z_1+\sigma_{\min}\boldsymbol\epsilon$; the noise floor must not be omitted from an exact endpoint identity. At the configured gain $b=1$, geometric pixel losses compare the first decoded channel of $\mathcal D(\widetilde{\mathbf z}_1)$ with the first channel of $\mathcal D(\mathcal E(q))$. Their reference therefore includes VAE reconstruction error. In the current code, photoclinometry instead receives inverse-transformed predicted relief. The archived run at revision \texttt{6a4ea05} supplies normalized decoded relief to that branch; its results do not evaluate the later change of shading representation. At inference, an $N$-step Euler realization follows $\mathbf z_{j+1}=\mathbf z_j+N^{-1}v_\theta(\mathbf z_j,j/N;\mathbf z_{\mathrm{img}})$. Source noising and postprocessing must be specified separately from step count.

We optionally weight the per-cell residual by a confidence map obtained by area-interpolating the pixel-space validity mask, so that latent cells whose receptive fields are dominated by nodata or low-confidence pixels contribute proportionally less to the loss.

\FloatBarrier
\subsection{Absolute Depth Loss ($\mathcal{L}_{\text{huber}}$)}
\label{sec:huber_loss}

\paragraph{Formulation.}
Let $d$ denote the first channel of decoded predicted relief and $\hat d$ the first channel of the VAE-reconstructed signed-log reference. Neither decoder output is intrinsically bounded. We apply the Huber penalty \cite{Huber1964RobustParameter}:
\begin{equation}
\mathcal{L}_{\text{huber}} =
\frac{1}{|\Omega_v|}\sum_{\mathbf{x}\in\Omega_v}
\rho_\delta\!\big(d(\mathbf{x}) - \hat{d}(\mathbf{x})\big),
\quad
\rho_\delta(r) =
\begin{cases}
\frac{1}{2}r^2 & |r|\le \delta,\\[2pt]
\delta\big(|r|-\frac{1}{2}\delta\big) & |r|>\delta,
\end{cases}
\label{eq:huber}
\end{equation}
with $\Omega_v$ the valid-pixel set and $\delta=0.1$. This is 5\% of an illustrative $[-1,1]$ interval, but that interval is not a bound on an unclipped encoding or decoded output. Soft confidence weights replace the binary average when enabled.

\paragraph{Justification.}
The Huber term anchors decoded relief values to the reference in addition to the latent flow target. Differential losses remove constant offsets, but are generally sensitive to height scaling; flow matching already provides an absolute latent target and is not inherently affine-invariant. Huber regression is quadratic for small residuals and linear in the tails, reducing the influence of large errors relative to squared loss~\cite{Huber1964RobustParameter,Laina2016DeeperNetworks,Ranftl2022TowardsTransfer}. It does not make outliers harmless or correct systematic reference bias. The implementation uses the Huber convention above, not smooth-$\ell_1$ normalized by $\delta$. Figure~\ref{fig:huber_viz} retains the per-pixel diagnostic.
\begin{figure}[htbp]
\centering
\includegraphics[width=\linewidth,height=0.80\textheight,keepaspectratio]{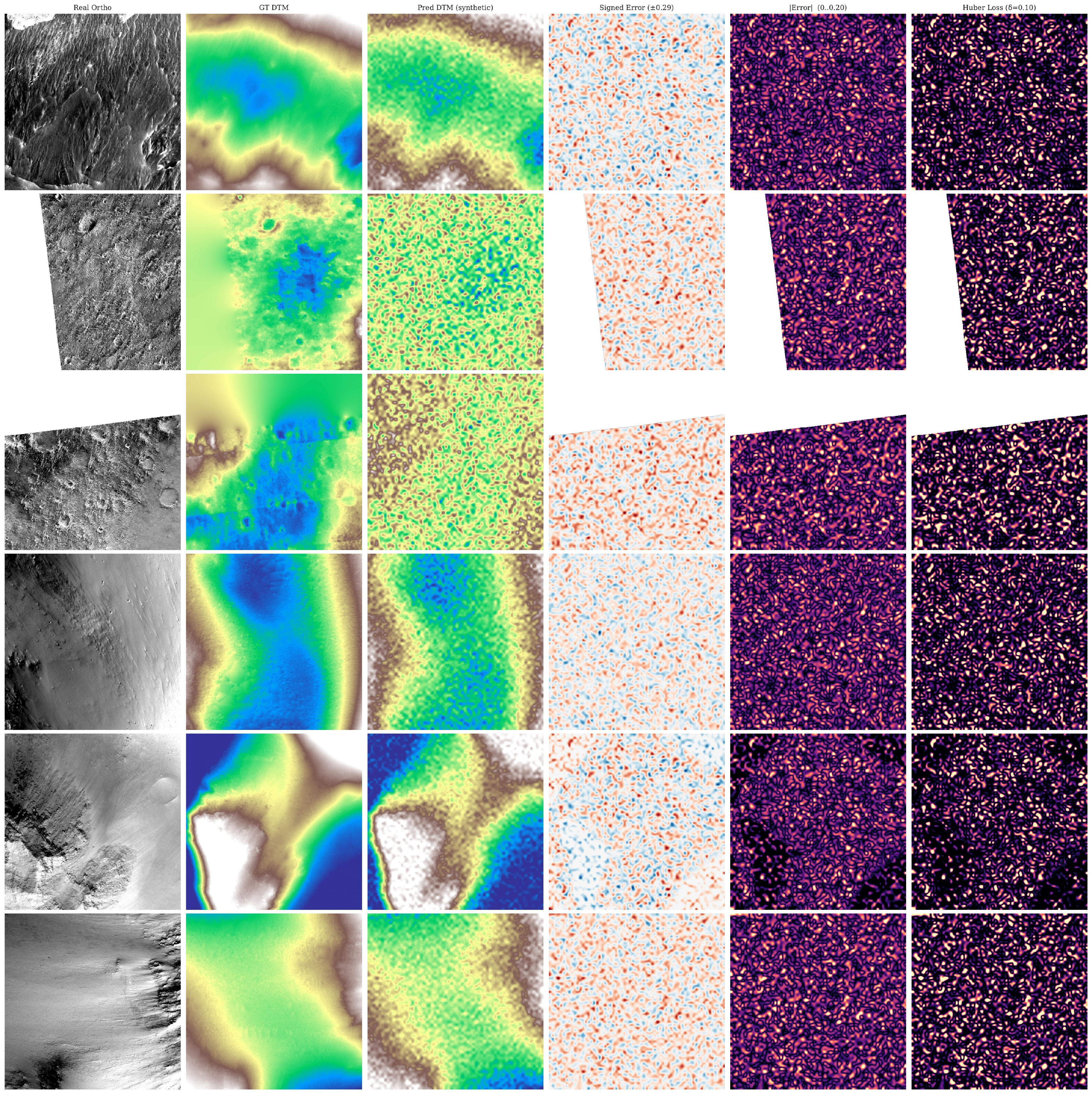}
\caption{%
\textbf{Huber loss diagnostic.} Columns, left to right: (1)~real orthoimage (percentile-stretched, masked); (2)~GT DTM (terrain colormap, $[-1,1]\!\to\![0,1]$); (3)~predicted DTM; (4)~signed error $d - \hat d$ on a symmetric RdBu\_r colormap, scale $\pm v_{\max}$ reported in the column title; (5)~absolute error $|d - \hat d|$ on a magma colormap, displayed up to $2\delta$; (6)~per-pixel Huber loss $\rho_\delta(\cdot)$, displayed up to $\delta^{2}$. The $L_2$-to-$L_1$ transition is visible by comparing columns (5) and (6): large residuals enter the linear branch of the penalty. Display saturation is separate from the loss, which continues to grow linearly.}
\label{fig:huber_viz}
\end{figure}
\FloatBarrier
\subsection{Surface-Normal Consistency Loss ($\mathcal{L}_{\text{norm}}$)}
\label{sec:normal_loss}

\paragraph{Formulation.}
The current normal-consistency operator uses the unnormalized $3\times3$ Sobel kernels (without the usual division by eight) and a spatial factor $s=\max(H,W)/2$. Applied to signed-log relief, this constructs normals in a normalized coordinate system; it does not by itself convert the derivatives into physical slopes. The represented normal is
\begin{equation}
\mathbf{n}(\mathbf{x}) = \frac{(-s\,\partial_x d,\,-s\,\partial_y d,\,1)^{\!\top}}
{\big\lVert(-s\,\partial_x d,\,-s\,\partial_y d,\,1)^{\!\top}\big\rVert_2},
\end{equation}
With $\hat{\mathbf n}$ computed from the reference under the same convention, the cosine loss is:
\begin{equation}
\mathcal{L}_{\text{norm}} = \frac{1}{|\Omega_v|}\sum_{\mathbf{x}\in\Omega_v}
\big(1 - \mathbf{n}(\mathbf{x})\!\cdot\!\hat{\mathbf n}(\mathbf{x})\big).
\end{equation}
The validity mask is morphologically eroded by one pixel to discard locations whose Sobel stencil crosses nodata. For physical interpretation, let $q$ be normalized relief, $r=S_{\mathrm{ref}}\operatorname{sgn}(q)(2^{|q|}-1)$ its inverse, and $X,Y$ horizontal coordinates in meters. Away from clipped tails,
\begin{equation}
\partial_X r=S_{\mathrm{ref}}\ln2\,2^{|q|}\partial_X q,\qquad
\mathbf n_{\mathrm{phys}}=\frac{(-\partial_X Z,-\partial_Y Z,1)}{\sqrt{1+(\partial_X Z)^2+(\partial_Y Z)^2}},
\end{equation}
where $Z=r+\text{plane}$ if the regional trend is independently available. Discrete derivatives must use the actual pixel spacing in meters. The logarithmic chain factor varies spatially, so a single factor $\max(H,W)/2$ cannot substitute for this conversion. This distinction applies to both angular evaluation and physical interpretation of shading.

\paragraph{Justification.}
Normal supervision compares local surface orientation and complements direct relief regression~\cite{Eigen2014DepthNetwork,Hu2018RevisitingBoundaries,Yin2019EnforcingPrediction,Fan2021Three-Filters-to-Normal:Estimator,Yang2024DepthData}. Here it operates on the graph of encoded relief, so its geometric meaning depends on both the height transform and derivative convention. The training operator uses raw Sobel kernels with replicate padding; evaluation uses unit-pixel finite differences. Neither supplies physical slope accuracy without horizontal spacing and height inversion.

\FloatBarrier
\subsection{Multi-Scale Gradient Matching Loss
($\mathcal{L}_{\text{grad}}$)}
\label{sec:grad_loss}

\paragraph{Formulation.}
Let $\mathcal P_s$ denote ordinary $s\times s$ average pooling, $m$ the confidence map, and $S=\{1,2,4\}$. For axis $a\in\{x,y\}$ define $\Delta_a$ as the forward difference and $w_{a,s}(u)=(\mathcal P_s m)(u)(\mathcal P_s m)(u+e_a)$. The implemented loss is
\begin{equation}
\mathcal L_{\mathrm{grad}}=\frac{1}{|S|}\sum_{s\in S}\frac1s
\sum_{a\in\{x,y\}}
\frac{\sum_u w_{a,s}(u)|\Delta_a\mathcal P_s d(u)-\Delta_a\mathcal P_s\hat d(u)|}{\sum_u w_{a,s}(u)+\varepsilon}.
\end{equation}
The two axes are normalized separately. At $s=1$, binary weights require both differenced pixels to be valid. At coarser scales, fractional pooled confidence downweights partially observed blocks; pooling is not masked, so filled values can influence their averages. The factor $1/s$ expresses a pooled-grid difference in original-pixel coordinates. This convention does not ensure equal contributions across scales.

\paragraph{Justification.}
Gradient matching ignores constant height offsets while preserving sensitivity to signed slope and edge contrast. It remains sensitive to tilt, local stereo artifacts, and height scale. Multi-scale supervision~\cite{eigen2015predictingdepthsurfacenormals,Li2018MegaDepth:Photos} compares sharp transitions and broader variations using the same basic operator. For Martian terrain, crater rims and scarp boundaries motivate the fine-scale term, while broad relief motivates pooled comparisons. The $1\times$ and $4\times$ fields in Figure~\ref{fig:loss_components} show what these terms measure; their individual benefit requires ablation.

\FloatBarrier
\subsection{Laplacian (Curvature) Consistency Loss
($\mathcal{L}_{\text{lap}}$)}
\label{sec:laplacian_loss}

\paragraph{Formulation.}
Let $\nabla^2$ denote the discrete five-point Laplacian. We use
\begin{equation}
\mathcal{L}_{\text{lap}} = \frac{1}{|\Omega_v^{\circ}|}
\sum_{\mathbf{x}\in\Omega_v^{\circ}}
\big|\nabla^2 d(\mathbf{x}) - \nabla^2 \hat{d}(\mathbf{x})\big|,
\label{eq:lap}
\end{equation}
where $\Omega_v^{\circ}$ is the valid set eroded by the $3{\times}3$ Laplacian footprint.

\paragraph{Justification.}
The Laplacian adds an explicit second-derivative penalty to first-order matching. Signed gradient agreement already constrains curvature indirectly; the extra term changes how local errors are weighted. With the convention $\Delta d=d_{x+1}+d_{x-1}+d_{y+1}+d_{y-1}-4d$, an ideal bowl has positive Laplacian and a peak negative Laplacian. This scalar is not a complete curvature tensor or a proof of geomorphic class, especially after nonlinear height normalization. Related differential operators are used in surface reconstruction and editing~\cite{Nehab2005EfficientlyGeometry,Sorkine2004LaplacianEditing,Atzmon2020SAL:Data}. Figure~\ref{fig:laplacian_viz} shows reference, prediction, and disagreement under the diagnostic's operator convention.
\begin{figure}[htbp]
\centering
\includegraphics[width=\linewidth,height=0.80\textheight,keepaspectratio]{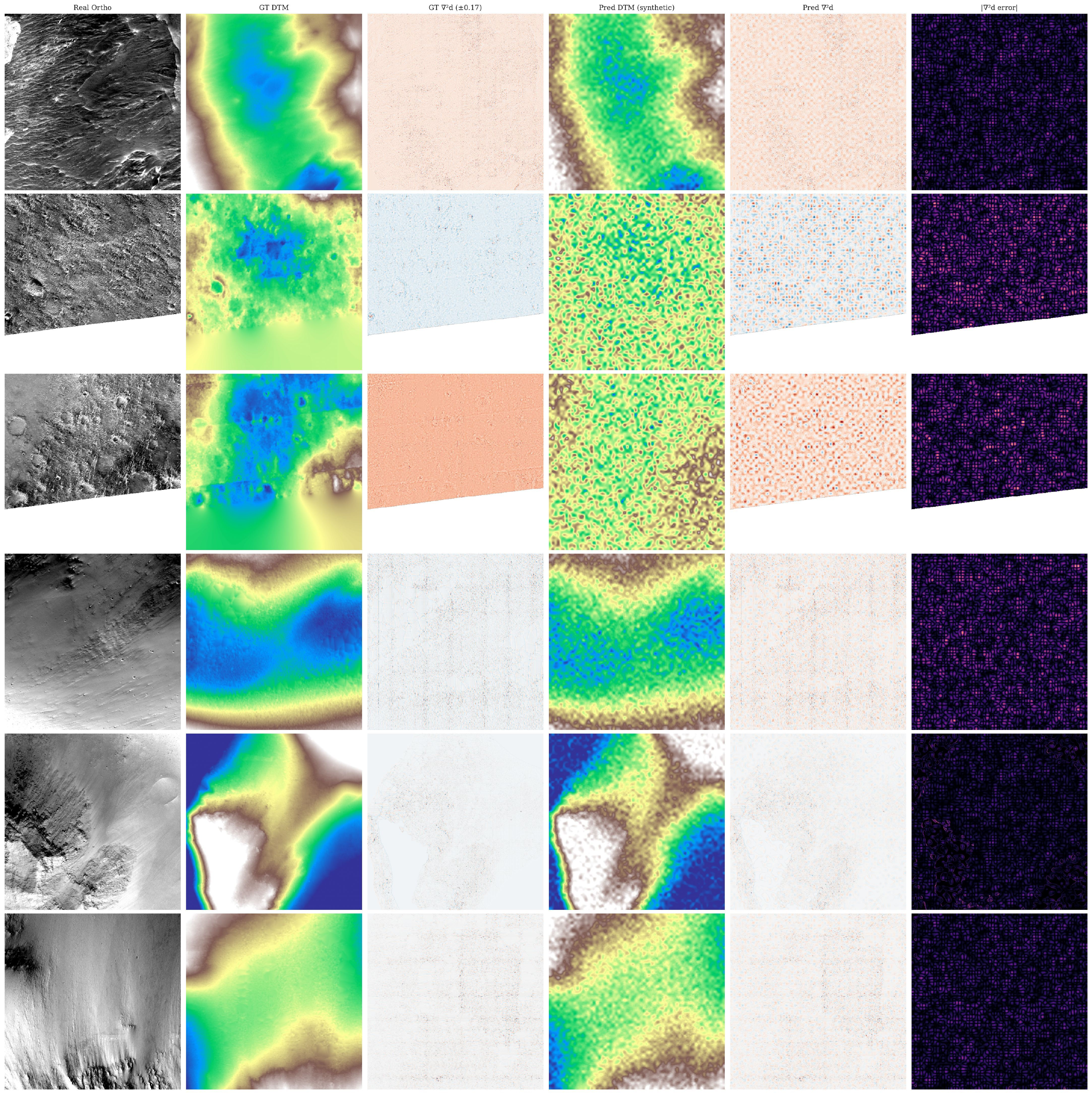}
\caption{\textbf{Laplacian diagnostic.} Columns show orthoimage, reference relief, reference Laplacian, predicted relief, predicted Laplacian, and absolute Laplacian error. Shared display limits permit spatial comparison of the derivative fields. Interpreting red/blue as bowl/peak requires checking the stencil sign; the standard positive-neighbor Laplacian is positive for a bowl. These are derivatives of the represented relief field, not independently calibrated physical curvatures.}
\label{fig:laplacian_viz}
\end{figure}
\FloatBarrier
\subsection{Focal Frequency Loss ($\mathcal{L}_{\text{FFL}}$)}
\label{sec:ffl_loss}

\paragraph{Formulation.}
Let $\mathcal{F}$ denote the 2-D discrete Fourier transform. Writing $F = \mathcal{F}(d)$ and $\hat F = \mathcal{F}(\hat d)$, the focal frequency loss of Jiang \emph{et al.}~\cite{Jiang2020FocalSynthesis} computes:
\begin{equation}
\mathcal{L}_{\text{FFL}} =
\frac{1}{HW}\sum_{\mathbf u} w(\mathbf u)\,
\big|F(\mathbf u) - \hat F(\mathbf u)\big|^2,
\quad
w(\mathbf u) = \frac{|F(\mathbf u) - \hat F(\mathbf u)|^\alpha}
{\max_{\mathbf u}|F(\mathbf u) - \hat F(\mathbf u)|^\alpha},
\end{equation}
The weight emphasizes the largest current frequency discrepancies, which need not occur at high frequency. Define $w=0$ when all discrepancies vanish, or use a documented positive denominator guard. Before the FFT, GT injection sets $\tilde d=c\,d+(1-c)\hat d$ for confidence $c\in[0,1]$. For a binary mask, the residual is zero outside the valid region. This suppresses differences caused by arbitrary padding but does not eliminate spectral leakage at mask boundaries; the FFT still analyzes a finite, masked domain.

\paragraph{Justification.}
Spectral bias motivates examining frequency-dependent errors~\cite{Rahaman2018OnNetworks,Xu2024FrequencyNetworks,Basri2020FrequencyDensity}. FFL adaptively weights difficult Fourier coefficients and has been studied for reconstruction and related synthesis tasks~\cite{Jiang2020FocalSynthesis,Gao2023ImplicitSuper-Resolution}. It can only compare with the spectrum present in the reference, and cannot verify fine terrain absent from that reference. We retain this optional formulation and its diagnostic spectrum, but $\lambda_{\mathrm{FFL}}=0$ in the reported training configuration, so no experimental gain is attributed to it.

\FloatBarrier
\subsection{Ordinal Ranking Loss ($\mathcal{L}_{\text{ord}}$)}
\label{sec:ordinal_loss}

\paragraph{Formulation.}
For each training iteration, we sample $N_p=10^4$ pixel pairs $(\mathbf{x}_i,\mathbf{x}_j)$ per image, optionally weighted by the confidence map. Only pairs with unambiguous GT ordering $|\hat d_i - \hat d_j| > \tau$ (we use $\tau=0.02$) contribute:
\begin{equation}
\mathcal{L}_{\text{ord}} =
\frac{1}{N_{\text{ord}}}
\sum_{(i,j):|\hat d_i - \hat d_j|>\tau}
\max\!\big(0,\;-\operatorname{sgn}(\hat d_i - \hat d_j)\,(d_i - d_j)\big),
\end{equation}
Here $N_{\mathrm{ord}}$ counts the retained pairs; define the loss as zero if that count is zero. This is a zero-margin hinge: an exactly tied prediction also has zero penalty. Thus it discourages reversed ordering but does not enforce a strictly positive separation. Adding such a margin would change the implemented objective and requires a separate experiment.

\paragraph{Justification.}
Stereo artifacts can corrupt elevations and their ordering~\cite{Beyer2018TheData,Kirk2009UltrahighSites}. Ordinal supervision is useful when an error changes magnitudes without reversing the selected pair order; the threshold removes pairs with small reference differences but does not guarantee correct labels. The approach draws on relative-depth learning~\cite{ChenSingle-ImageWild,Xian2018MonocularSupervision,Ranftl2022TowardsTransfer}, while the displayed hinge is the formulation used here rather than a claim of exact equivalence to all cited objectives. Sampling with zero probability on invalid pixels excludes nodata; confidence weighting favors trusted regions only to the extent that the confidence map is reliable. Figure~\ref{fig:ordinal_viz} retains pair outcomes and local violation rates for inspection.
\begin{figure}[htbp]
\centering
\includegraphics[width=\linewidth,height=0.80\textheight,keepaspectratio]{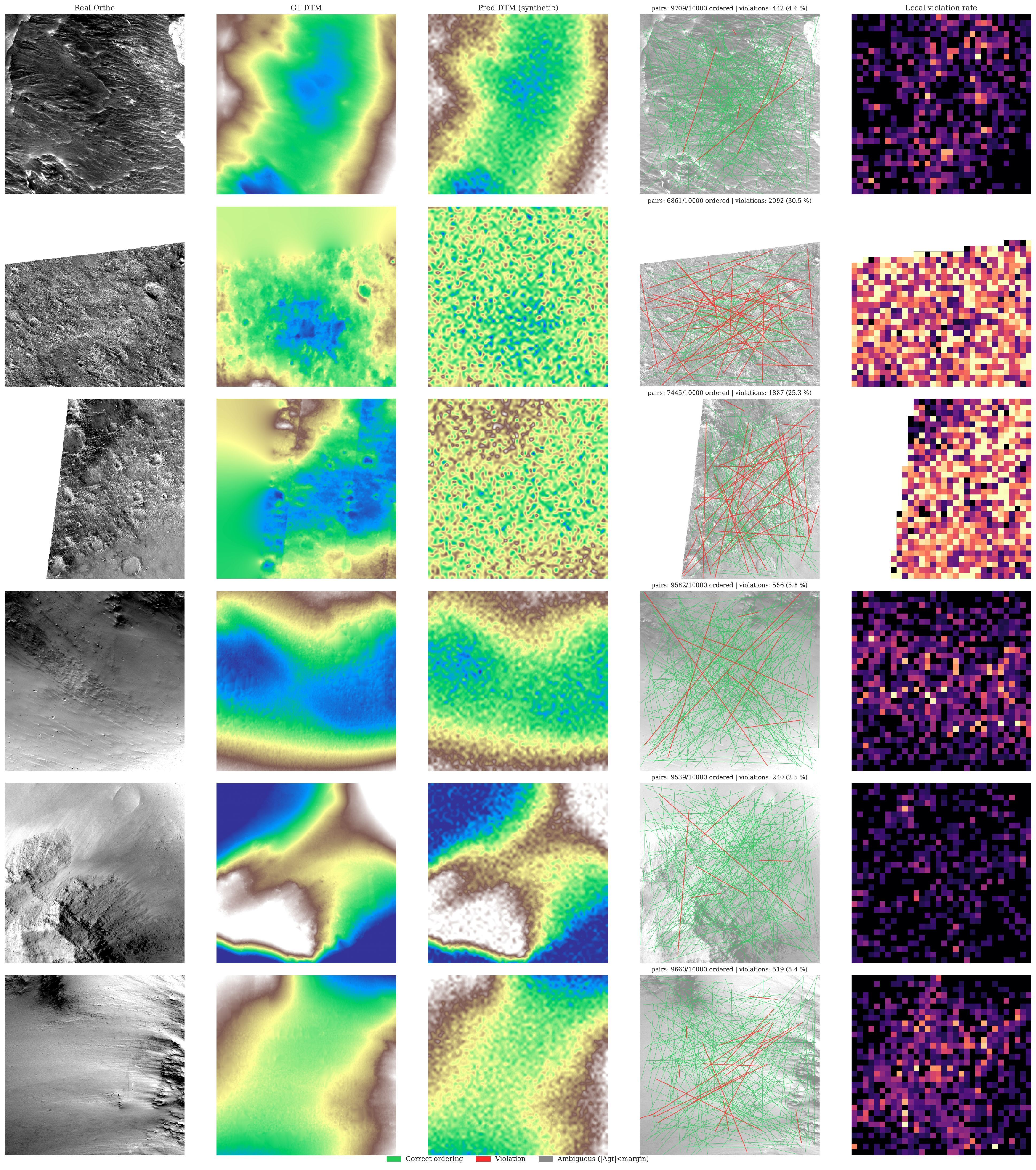}
\caption{%
\textbf{Ordinal ranking diagnostic.} Columns: (1)~orthoimage; (2)~GT DTM; (3)~predicted DTM; (4)~\emph{pair-sample panel} --- 250 randomly-drawn pairs overlaid on a dimmed ortho. Green segments mark pairs whose GT ordering is unambiguous ($|\Delta\hat d|>\tau$) and whose prediction \emph{agrees}; red segments mark ordering \emph{violations}; grey segments are ambiguous pairs ($|\Delta\hat d|\le\tau$) which are excluded from the loss. The per-row caption reports the ordered/total pair count and the violation rate; (5)~\emph{local violation rate} on a coarse $32\!\times\!32$ grid (magma colormap). Hot spots reveal where the model systematically predicts the wrong relative ordering --- typically crater rims and shadowed regions --- and are directly actionable for dataset reweighting or active-learning sample selection. Legend at the foot of the figure.}
\label{fig:ordinal_viz}
\end{figure}
%
\FloatBarrier
\subsection{Photoclinometric (Shape-from-Shading) Loss ($\mathcal{L}_{\text{photo}}$)}
\label{sec:photo_loss}

\FloatBarrier
\subsubsection{Motivation}

The preceding objectives compare with the stereo reference, directly or in latent space. Where the image grid has finer sampling than the terrain grid~\cite{McEwen2007MarsHiRISE,Kirk2009UltrahighSites}, interpolation does not create missing geometric measurements. As in other super-resolution settings~\cite{Dong2014ImageNetworks,Ledig2016Photo-RealisticNetwork}, plausible fine-scale output must be distinguished from verified recovery. Image-based shading offers an additional constraint whose information and ambiguities differ from those of stereo.

Photoclinometry relates image irradiance to surface orientation under assumptions about illumination and reflectance~\cite{SHAPEBooks,Horn1989ShapeShading,Rigneault2008NumericalBenchmarks}. Planetary applications have used these constraints to refine terrain estimates~\cite{Kirk2003High-resolutionImages,Kirk2009UltrahighSites,Beyer2018TheData,Alexandrov2018MultiviewImages}. We transfer this idea into a differentiable training objective: predicted relief is rendered and compared with the observed image. The image contains potentially useful fine-scale variation, but some of that variation comes from albedo, cast shadows, registration, or sensor effects. Consequently, shading consistency is evidence of compatibility with a surrogate image model, not a unique solution for terrain.

The photoclinometric term is the only objective here that compares the prediction to image intensity rather than solely to stereo-derived relief. Its spatial support is the processed image grid. Establishing a gain beyond the reference requires both a known resampling scale and independent geometric validation.

\FloatBarrier
\subsubsection{Formulation}

In the current implementation, the training caller first applies the inverse signed-log transform to decoded predicted relief. The recorded experimental revision instead renders the normalized decoded field. The reflectance equations below apply to either graph, but the input representation changes its normals and must accompany an experimental result. The shading branch uses Sobel kernels divided by eight, unlike the normal-consistency branch, and multiplies derivatives by $\max(H,W)/2$, with a numerical clamp at $10^4$. This defines a graph on normalized horizontal coordinates; even with meter-valued residual heights, it is not a calibrated metric surface unless horizontal spacing is supplied. We render that graph under a Lunar--Lambert surrogate and compare it to the grayscale orthoimage $I_o$ using a positive-affine-invariant structural comparison:
\begin{equation}
\mathcal{L}_{\text{photo}} = \frac{1}{|S|}\sum_{s\in S}
\Big[(1-\alpha)\,\big(1-\operatorname{Pearson}_s(\mathcal{R}_s,\,I_{o,s})\big)
+ \alpha\,\big(1-\operatorname{SSIM}^{z}_s(\mathcal{R}_s,\,I_{o,s})\big)\Big],
\label{eq:photo}
\end{equation}
with $S=\{1,2,4\}$, $\alpha=0.5$, and $\mathcal{R}_s$ the render at scale $s$ defined below. We now detail each component.

\paragraph{Lunar--Lambert reflectance model.}
Lambertian reflectance is a useful baseline~\cite{SHAPEBooks}, but planetary regolith can exhibit more complex photometric behavior~\cite{Hapke2012TheorySpectroscopy}. We use an empirical Lunar--Lambert blend~\cite{McEwen1991PhotometricApplications,mcewen1996reflectance,Kirk2003High-resolutionImages}. In our coefficient convention, $L=1$ is Lambertian and $L=0$ is the unscaled Lommel--Seeliger term:
\begin{equation}
\mathcal{R}_{\text{LL}}(\mathbf{n};L,\ell) =
\underbrace{L\,\mu_i}_{\text{Lambertian}}
+ \underbrace{(1-L)\,\frac{\mu_i}{\mu_i+\mu_e+10^{-6}}}_{\text{Lommel--Seeliger}},
\label{eq:ll_render}
\end{equation}
Here $\mu_i=\max(\mathbf n\cdot\boldsymbol\ell,0)$ and $\mu_e=\max(n_z,10^{-4})$ implement nonnegative incidence and a safe denominator under a nadir-view surrogate. HiRISE stereo acquisitions can be off nadir, so this is a modeling approximation rather than an instrument guarantee. The scalar blend is $L=\sigma(\ell_{LL})$, initialized by $\ell_{LL}=0$ to give $L=0.5$. Coefficient conventions in the photometric literature can reverse the mixture weight or include a factor of two; the equation above defines the convention used in this manuscript. It is an empirical blend, not a fitted complete phase function. The per-pixel render is
\begin{equation}
\mathcal{R}(\mathbf{x}) =
\mathcal{R}_{\text{LL}}\big(\mathbf n(\mathbf x);L,\boldsymbol\ell\big)\,I_{\text{sun}}
+ I_{\text{amb}},
\end{equation}
The per-image gain and offset $(I_{\mathrm{sun}},I_{\mathrm{amb}})$ are estimated using the robust Lambertian fit in Section~\ref{sec:sun_irls}. Positive global exposure changes are removed by the structural comparison below. A change from Lambertian to Lunar--Lambert reflectance can also alter spatial contrast and is not generally reducible to a single affine exposure correction.

\paragraph{Multi-scale structural comparison.}
At each scale $s\in\{1,2,4\}$, the implementation average-pools relief and the orthoimage, then recomputes normals and renders at that scale. Rendering pooled relief differs from pooling a full-resolution render. Pooled confidence must exceed 0.99, and samples with less than 5\% valid support at a scale receive a zero contribution. Two complementary terms are computed:
\begin{itemize}
\item A \textbf{Pearson correlation} term $\operatorname{Pearson}_s = \frac{\operatorname{Cov}(\mathcal{R}_s,I_{o,s})}
{\sqrt{\operatorname{Var}(\mathcal{R}_s)\operatorname{Var}(I_{o,s})}}$. For nonconstant signals on a fixed mask, Pearson is invariant to positive affine changes of either signal; negative gain reverses correlation. It measures correspondence of spatial intensity variation. \item A \textbf{$ z$-normalized SSIM}~\cite{Wang2004ImageSimilarity} term $\operatorname{SSIM}^{z}_s$, computed after each signal is per-image z-scored inside the valid mask. Z-scoring removes \emph{global} luminance and contrast, but SSIM's local luminance/contrast/structure decomposition preserves local structural comparison using a $7\!\times\!7$ Gaussian kernel.
\end{itemize}
Equation~\eqref{eq:photo} removes positive global gain and additive offset under fixed masks and nonzero signal variance. It remains sensitive to local pattern changes, including those caused by geometry, the blend $L$, albedo, or an incorrect sun direction. Numerical variance floors, clipping, and gain-dependent masks can weaken exact invariance. These distinctions matter when interpreting both the loss and the illumination sensitivity experiments.

The $s=1$ comparison retains variation on the processed image grid, while $s=2,4$ average local fluctuations and compare broader patterns. This is motivated by multi-scale structural and photoconsistency objectives~\cite{Wang2003Multi-scaleAssessment,Yao2018MVSNet:Stereo}. It provides neither a sub-pixel accuracy guarantee nor proof that a particular scale stabilizes training; those are experimental questions.

The implementation does not clip the final rendered intensity to the orthoimage's input range. It does clamp local incidence to zero and stabilizes the emission denominator as specified in Equation~\eqref{eq:ll_render}. These operations do not model cast shadows. Global exposure invariance holds for the ideal structural comparison under fixed masks; variance floors make it approximate numerically. Deep shadows can be poorly informative or violate the surrogate~\cite{SHAPEBooks}.

\FloatBarrier
\subsubsection{Per-Scene Illumination Estimation with a Decoupled Offset}
\label{sec:sun_irls}

The render uses a unit direction $\boldsymbol\ell\in\mathbb S^2$, positive directional gain $I_{\mathrm{sun}}$, and ambient offset $I_{\mathrm{amb}}$. PDS products contain acquisition geometry, but a tile-level image-model fit additionally requires consistent coordinate conventions and radiometric assumptions~\cite{McEwen2007MarsHiRISE,Kirk2009UltrahighSites}. Broader image mosaics similarly combine acquisition and processing choices~\cite{Dickson2018AEditing}. We estimate nuisance parameters from each ortho/reference-DTM/mask triple during data preparation and cache them. These are reference-conditioned fit parameters, not independent measurements of the true sun direction. Their role in training and diagnostic scoring should be distinguished from deployment without a reference DTM.

\paragraph{Conditioning of the joint direction--offset fit.}
A joint Lambertian fit uses $I_o(\mathbf x)=\mathbf k^\top\mathbf n(\mathbf x)+b$ with design matrix $A=[\mathbf N_x,\mathbf N_y,\mathbf N_z,\mathbf1]$. On gently varying terrain, $n_z\approx1$, making the last two columns nearly collinear. The system can therefore be ill-conditioned, while exactly flat terrain is rank-deficient. Directional vertical intensity and offset are then difficult to identify separately, and noise or model mismatch can produce unstable coefficients, including a negative $\hat k_z$. The issue is conditioning and identifiability, not a universal failure of ordinary least squares.

A negative vertical direction is inconsistent with the chosen upper-hemisphere illumination prior. Depending on clipping and the reflectance continuation, it can produce unphysical or weakly varying renders. An unclamped formula does not necessarily become a uniformly black image, but near-constant renders do make normalized structural losses poorly conditioned. This motivates an explicit estimation and validity policy.

\paragraph{Implemented estimator.}
The cached illumination parameters were generated with \texttt{estimate\_sun\_vector\_irls}. The adapter supplies inverse-transformed residual relief, patch-normalized imagery, and a validity mask. The routine averages RGB channels when present and constructs unit normals from replicate-padded Sobel derivatives divided by eight and multiplied by $\max(H,W)/2$. These are graph normals in normalized horizontal coordinates. Let $\Omega_v$ contain valid pixels with normalized image intensity greater than $10^{-4}$. For at least 100 selected pixels, the offset and regression target are
\begin{equation}
I_{\mathrm{amb}}=Q_{0.01}\{I_o(\mathbf x):\mathbf x\in\Omega_v\},\qquad
Y_i=\max(I_o(\mathbf x_i)-I_{\mathrm{amb}},0).
\end{equation}
The percentile estimates an offset of the selected normalized intensities, not physical ambient radiance. Removing the intercept avoids its near-collinearity with $n_z$, although weak normal diversity can still make direction poorly identified. With rows $H_i=\mathbf n_i^\top$, the initial coefficient is $\mathbf k^{(0)}=\arg\min_{\mathbf k}\|H\mathbf k-Y\|_2^2$.

The routine forms Huber-weighted candidates~\cite{Huber1964RobustParameter,Holland1977RobustLeast-squares} using a MAD scale~\cite{Rousseeuw1993AlternativesDeviation}:
\begin{equation}
\begin{aligned}
r_i^{(t)}&=Y_i-H_i\mathbf k^{(t)},\qquad
\sigma^{(t)}=\operatorname{median}_i|r_i^{(t)}-\operatorname{median}_j r_j^{(t)}|/0.67449+10^{-6},\\
w_i^{(t)}&=\min\left(1,\frac{1.345}{|r_i^{(t)}/\sigma^{(t)}|+10^{-8}}\right),\qquad
\widetilde{\mathbf k}^{(t+1)}=\arg\min_{\mathbf k}\sum_i w_i^{(t)}(Y_i-H_i\mathbf k)^2.
\end{aligned}
\label{eq:irls_update}
\end{equation}
The actual assignment rule differs from an ordinary IRLS recurrence: the candidate replaces the current coefficient only if $\|\widetilde{\mathbf k}^{(t+1)}-\mathbf k^{(t)}\|_2<10^{-4}$, after which the loop terminates. Otherwise the coefficient remains unchanged, for at most 15 attempts. Thus a nonconvergent call returns its initial least-squares coefficient; the implementation does not progressively reweight successive updated fits. The cache's use of this routine establishes its provenance, but does not establish the robustness of a completed IRLS optimization.

Finally, the routine reflects a negative vertical coefficient, $k_z\leftarrow|k_z|$, returns $I_{\mathrm{sun}}=\max(\|\mathbf k\|_2,10^{-4})$, and applies PyTorch vector normalization to obtain $\boldsymbol\ell$. Reflection is an upper-hemisphere heuristic, not a constrained least-squares optimum~\cite{Belhumeur1999Bas-reliefAmbiguity,Shi2019AStereo}. A zero coefficient vector remains a degenerate direction. With fewer than 100 selected pixels, the defaults are $\boldsymbol\ell=(0.5,-0.5,1)/\sqrt{1.5}$, $I_{\mathrm{sun}}=1$, and $I_{\mathrm{amb}}=0.05$. Figure~\ref{fig:solar_dist} describes the fitted distribution; comparison with acquisition geometry would provide an independent directional check.

%
\begin{figure}[htbp]
\centering
\begin{subfigure}[t]{0.31\linewidth}
\centering
\includegraphics[width=\linewidth,height=0.80\textheight,keepaspectratio]{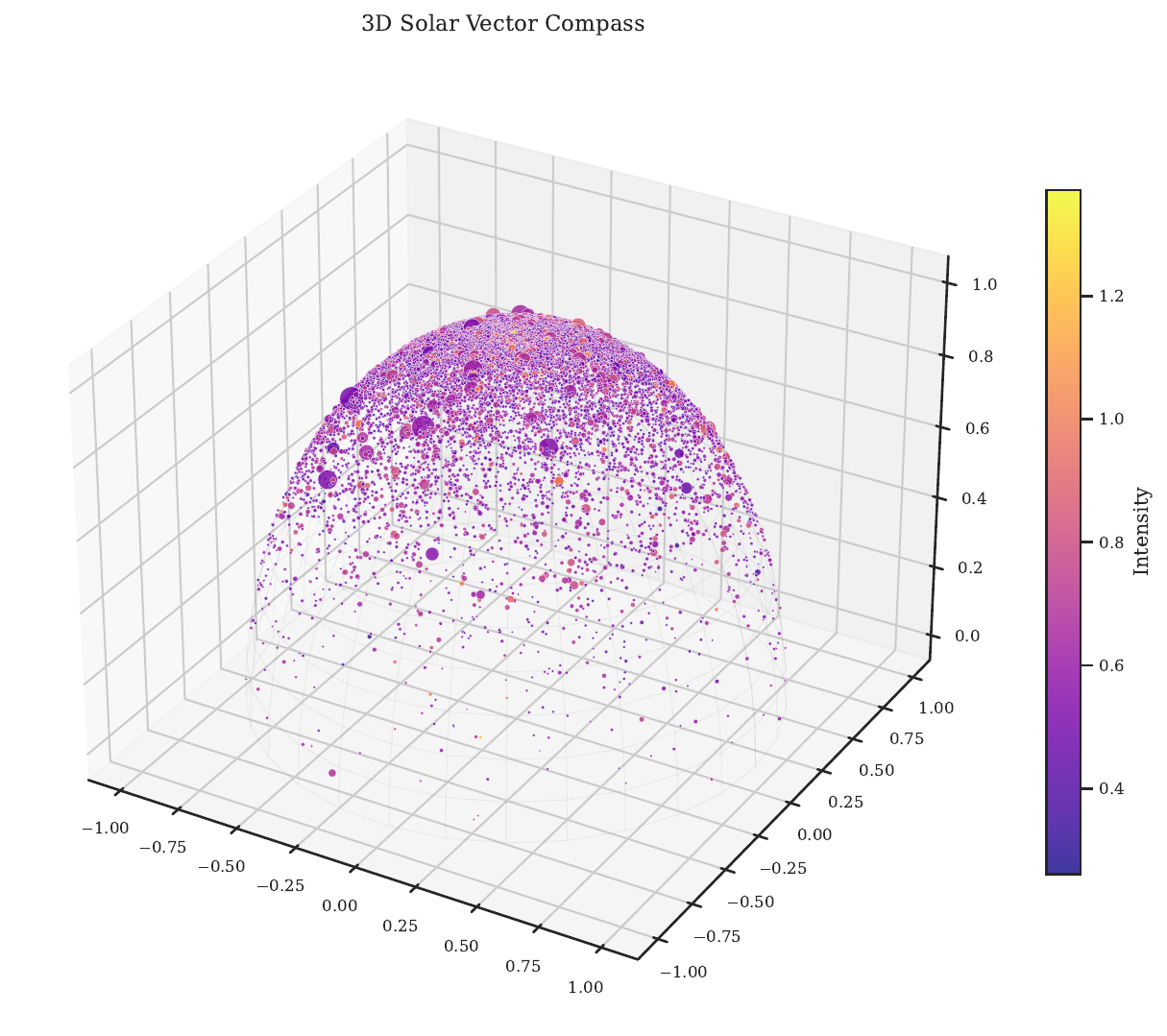}
\caption{3D solar vector compass.}
\label{fig:solar_compass_3d}
\end{subfigure}
\hfill
\begin{subfigure}[t]{0.31\linewidth}
\centering
\includegraphics[width=\linewidth,height=0.80\textheight,keepaspectratio]{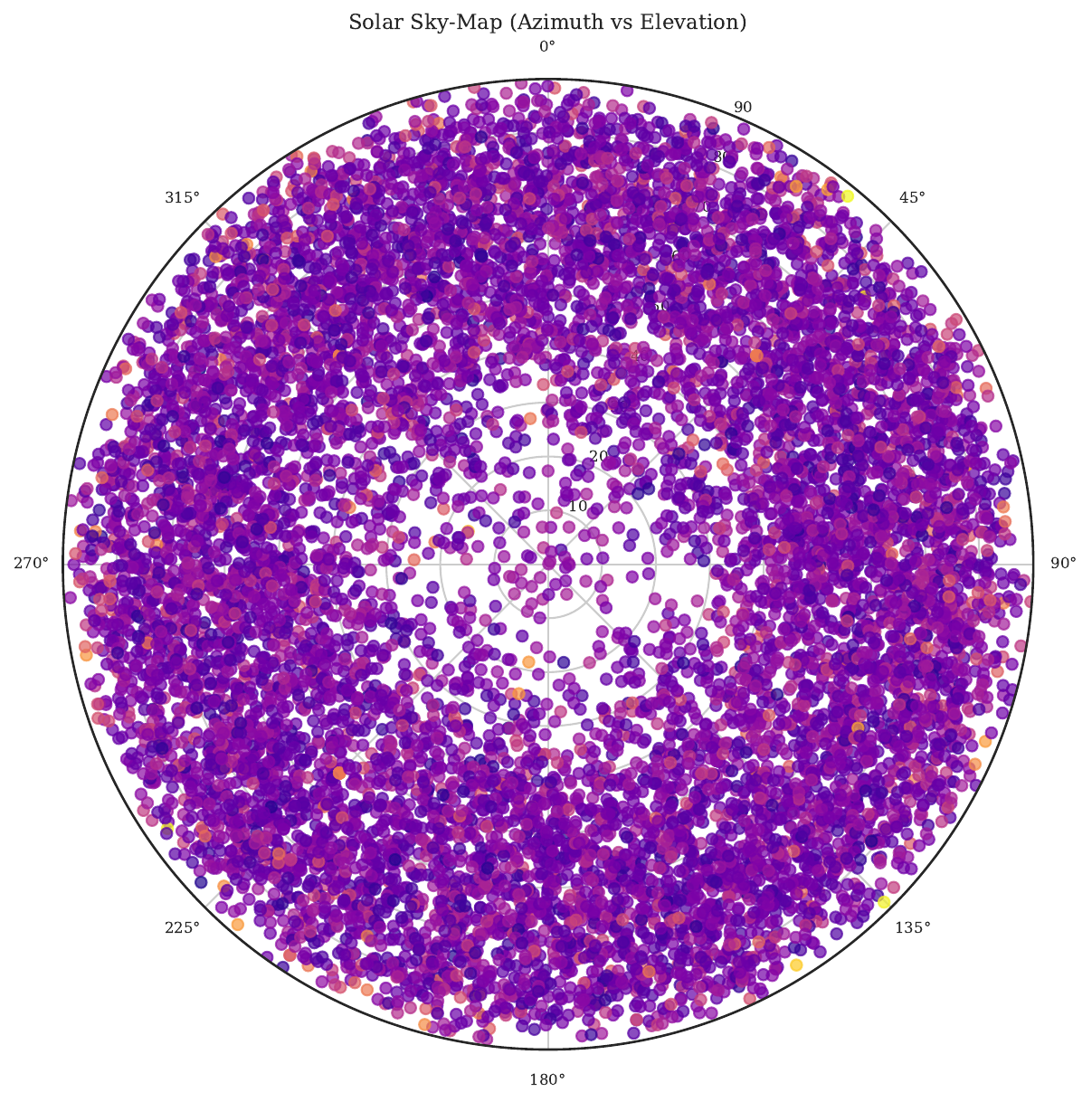}
\caption{Polar sky map (azimuth vs.\ elevation).}
\label{fig:solar_sky_polar}
\end{subfigure}
\hfill
\begin{subfigure}[t]{0.31\linewidth}
\centering
\includegraphics[width=\linewidth,height=0.80\textheight,keepaspectratio]{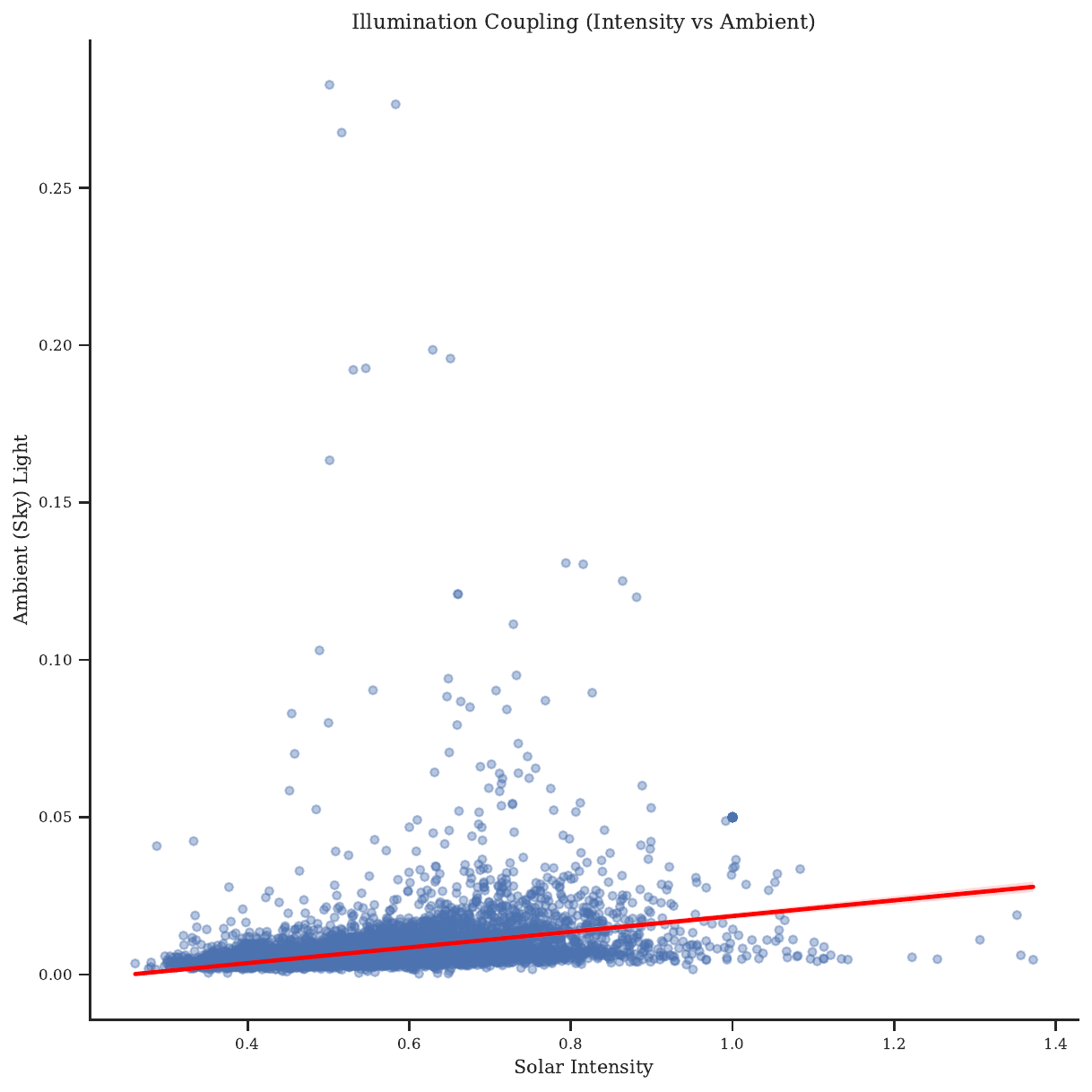}
\caption{Exposure coupling (intensity vs.\ ambient).}
\label{fig:solar_coupling}
\end{subfigure}
\caption{\textbf{Distribution of fitted illumination parameters.} (a) Unit directions after the upper-hemisphere correction; color encodes fitted directional gain and marker size ambient offset. (b) Polar view of the direction distribution. (c) Gain versus offset with a fitted trend. The hemisphere support is imposed by the estimator, and direction coverage is not proof of unbiased recovery. Positive gain--offset association can arise from exposure or the fitting procedure as well as atmospheric effects~\cite{Hapke2012TheorySpectroscopy,Pollack1979PropertiesAtmosphere}. Independent metadata comparison is needed to validate the fitted directions and physical interpretation.}
\label{fig:solar_dist}
\end{figure}
\FloatBarrier
\subsection{Sun vector Photoclinometric loss sensitivity}
Changing the sun direction changes the spatial shading pattern and can therefore affect the training constraint. For fixed geometry and support, positive global gain and additive offset are removed by ideal z-normalization and Pearson comparison. Local illumination or reflectance changes are not removed by this invariance.

Figure~\ref{fig:component_ablation} reports mean loss sensitivity to angular perturbations around fitted reference directions; Figure~\ref{fig:qualitative_illumination_physics} shows two $+45^{\circ}$ perturbations on one patch. Figures~\ref{fig:prove_and_visualize_local_convexity} and~\ref{fig:saddle_contour} examine local curvature, while Figure~\ref{fig:umap_global_disentanglement} reports exploratory predictive probes. Together they characterize the shading objective under the sampled conditions. They are not trained-model loss ablations or a proof of unique geometric recovery.

\begin{figure}[htbp]
    \centering
    \includegraphics[width=\linewidth,height=0.80\textheight,keepaspectratio]{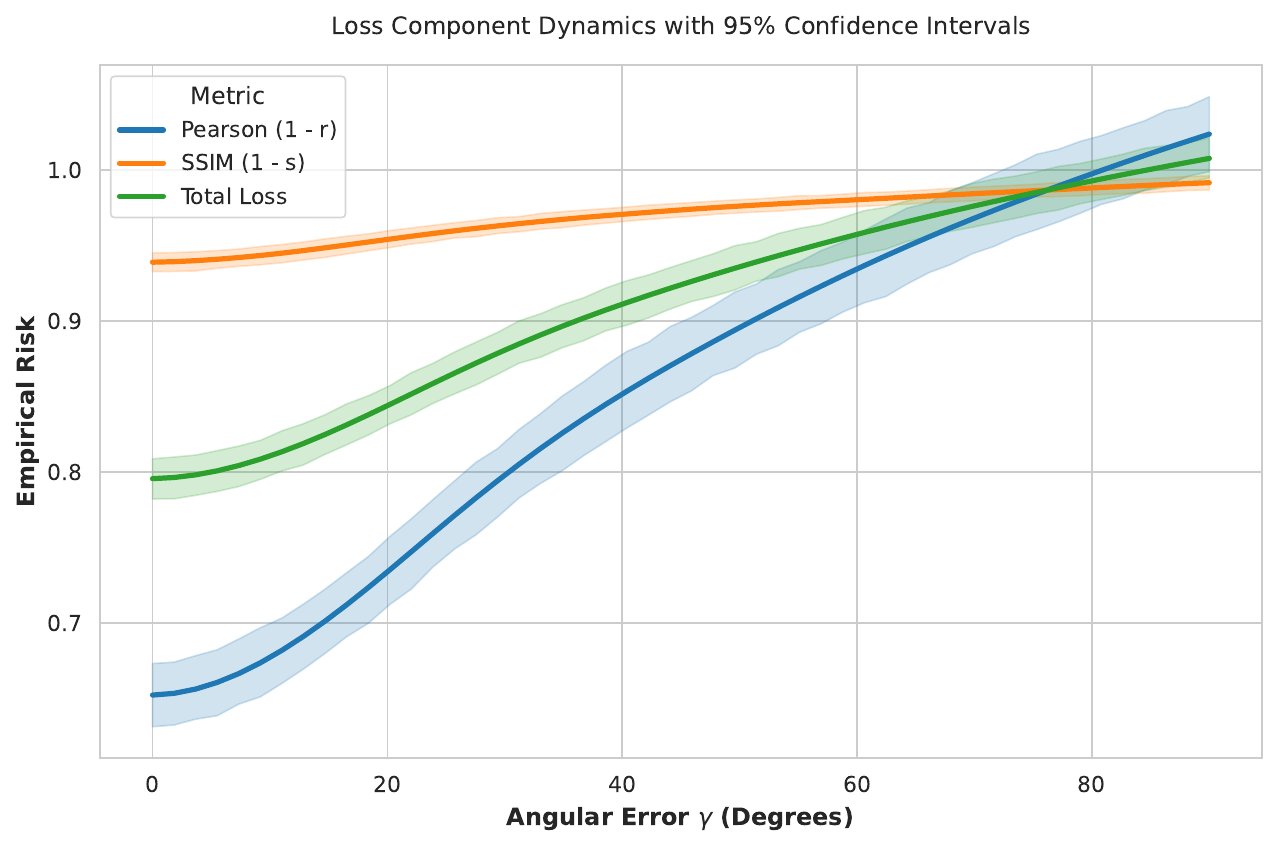}
    \caption{\textbf{Sensitivity of image-loss components to direction perturbation.} The mean Pearson, SSIM, and combined losses increase over the plotted angular range. Bands are labeled 95\% intervals in the archived export; the resampling unit and procedure must accompany the experiment. Pearson varies more strongly than SSIM in this plot. The comparison characterizes the two components under fixed geometry; it does not prove that both are necessary for reconstruction accuracy.}
    \label{fig:component_ablation}
\end{figure}

The photoclinometric forward model couples three physical parameters to the rendered intensity: the sun direction $\mathbf{s}\!\in\!\mathbb{S}^2$, a multiplicative gain $I$, and an additive ambient term $A$. By construction (Sec.~\ref{sec:photo_loss}), the loss is built on per-image $z$-score normalization followed by Pearson correlation and SSIM, with positive-affine invariance supplied by normalization (and by Pearson itself). Since $I$ and $A$ enter the render only through such a transformation, the loss analytically annihilates them; among these three illumination parameters, only $\mathbf{s}$ changes the spatial response after ideal normalization. Geometry and reflectance parameters can also change that response. A formal derivation is given in App.~\ref{app:invariance_proof}. The following diagnostics examine this distinction at pixel, sample, terrain-aggregate, and dataset levels, subject to the assumptions of the invariance proof.

\textbf{Pixel-level.} Figure~\ref{fig:qualitative_illumination_physics} compares a fitted reference direction with $+45^{\circ}$ azimuth and elevation perturbations on one patch. The z-residual RMSE changes from 1.13 to 1.30 for azimuth and to 1.14 for elevation, with corresponding MAE values of approximately 0.89, 1.08, and 0.91. The different responses show that angular sensitivity depends on direction and terrain. The unchanged input image and retained residual distributions make that comparison inspectable.

\begin{figure}[htbp]
    \centering
    \includegraphics[width=0.78\linewidth,height=0.80\textheight,keepaspectratio]{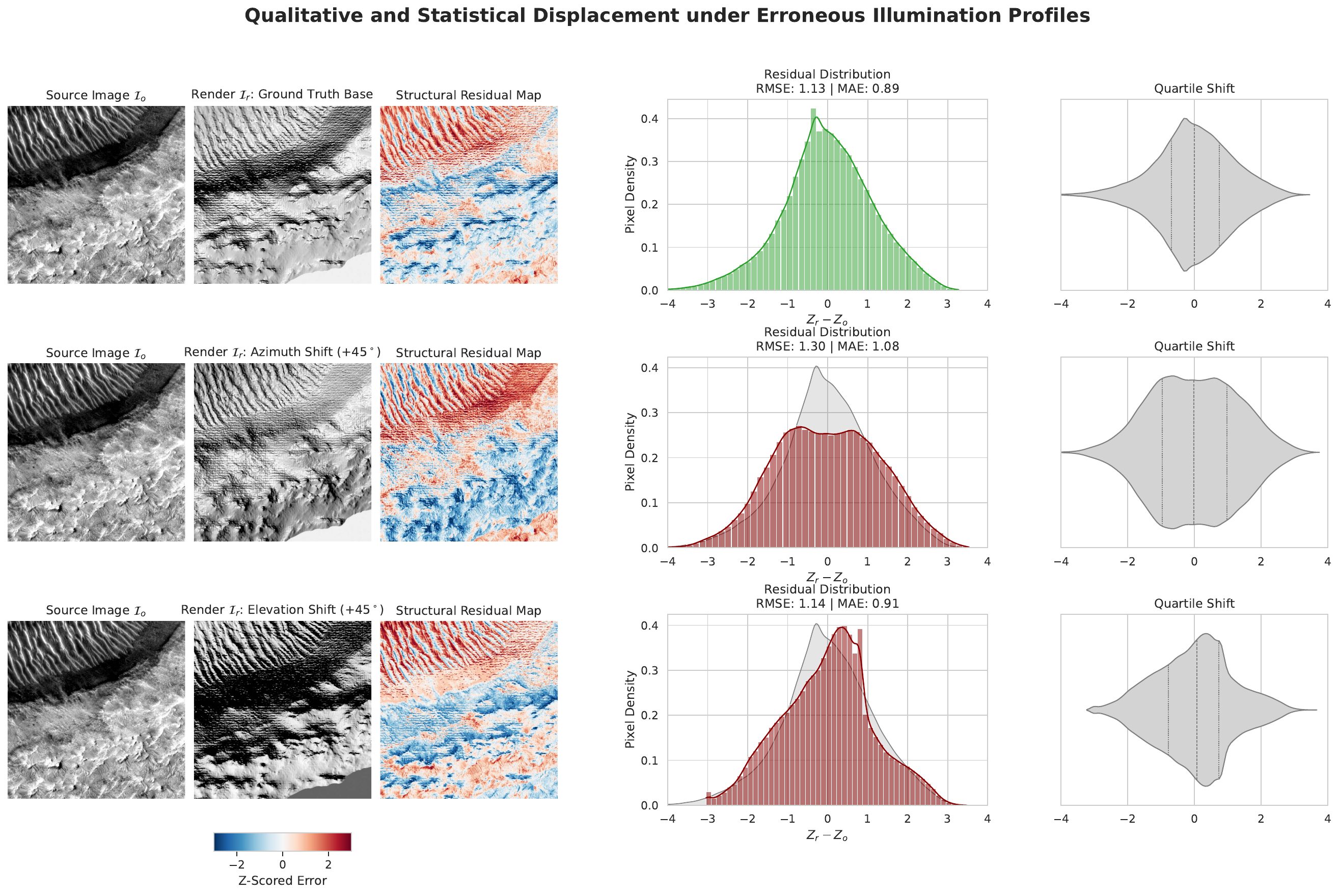}
    \caption{\textbf{Effect of illumination perturbations on one patch.} Columns show the observed image, surrogate render, signed z-residual, residual density, and a distribution summary. Rows use the fitted reference direction, a $+45^{\circ}$ azimuth shift, and a $+45^{\circ}$ elevation shift. The baseline is not independently measured solar ground truth. The retained RMSE/MAE annotations quantify image residuals; they are not terrain errors or proof of a unique inverse.}
    \label{fig:qualitative_illumination_physics}
\end{figure}

\textbf{Sample-level.} Figure~\ref{fig:prove_and_visualize_local_convexity} reports tangent-plane diagnostics for 794 fitted patches. Its classifier labels 793 cases as minima and one as a saddle, with a displayed bootstrap interval of $[99.6,100.0]\%$ for the former proportion. Positive Hessian eigenvalues characterize local curvature; a strict-minimum claim also requires stationarity and numerical reliability. The plotted median gradient norm is approximately $1.5\times10^{-3}$ and the condition-number distribution peaks near 1.8. The quarter-circle pattern in the direction-alignment panel follows from projecting a unit tangent vector onto an orthonormal two-dimensional basis, and does not establish alignment with either physical axis.

\begin{figure}[htbp]
    \centering
    \includegraphics[width=\linewidth,height=0.80\textheight,keepaspectratio]{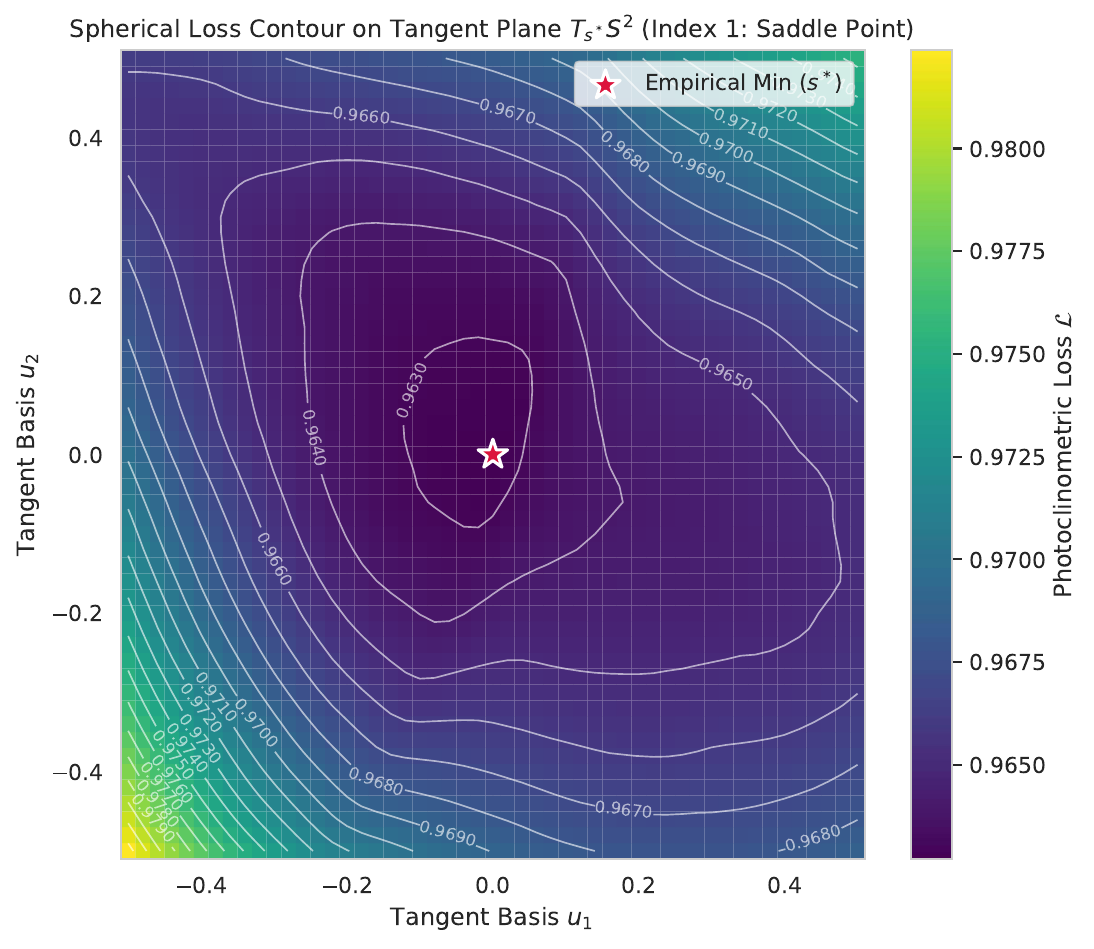}
    \caption{\textbf{Retained negative-curvature case.} This archived tangent-plane contour accompanies the sample classified as a saddle. The reported negative eigenvalue, $-6.21\times10^{-3}$, cannot be dismissed as numerical error merely because it is small. A reliable classification requires stationarity checks, finite-difference or autodifferentiation consistency, and sensitivity to numerical tolerances. The example qualifies the broader local-curvature result.}
    \label{fig:saddle_contour}
\end{figure}

\begin{figure}[htbp]
    \centering
    \includegraphics[width=\linewidth,height=0.80\textheight,keepaspectratio]{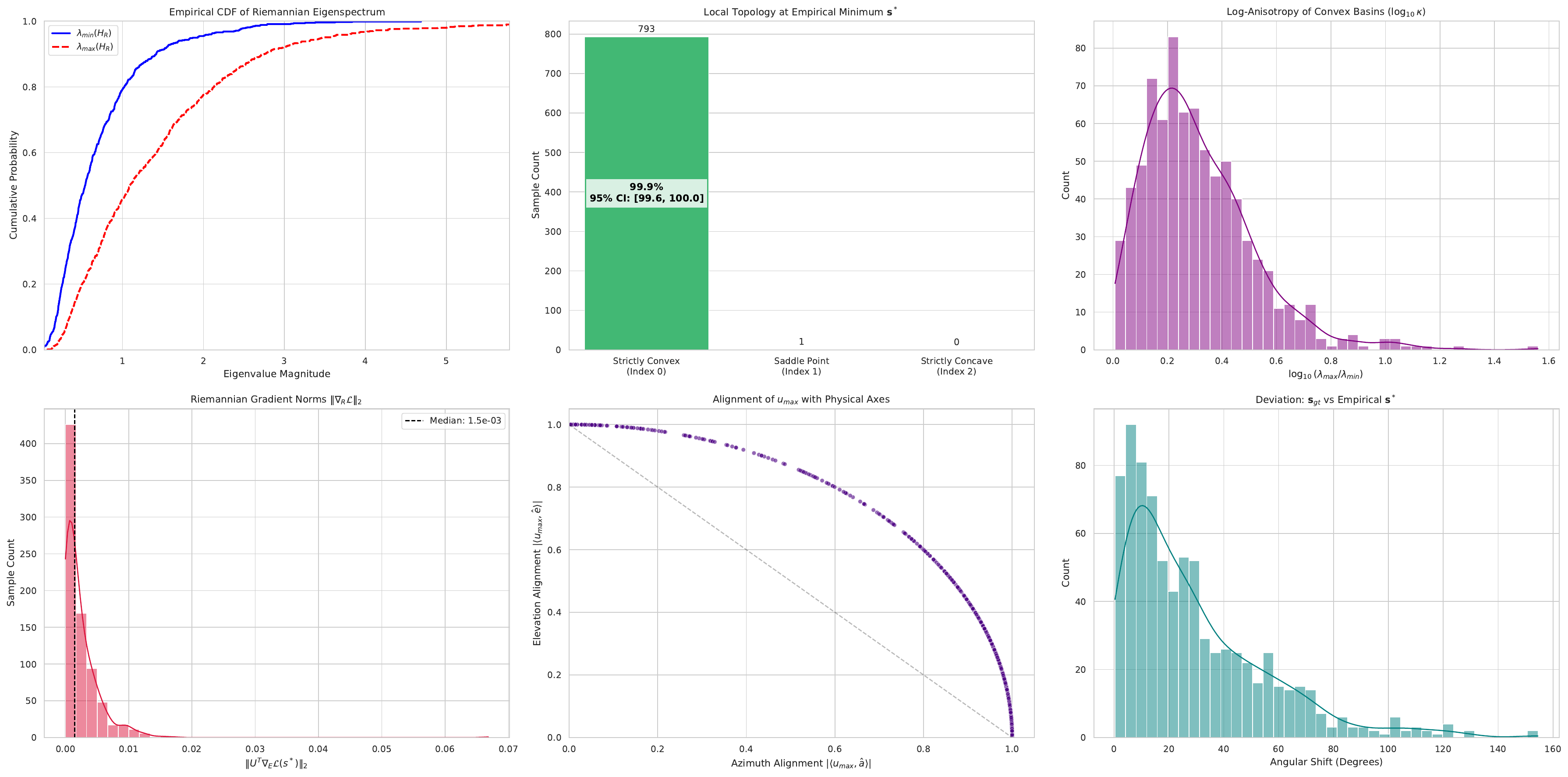}
\caption{\textbf{Local tangent-plane diagnostics for 794 patches.} Top row: Hessian-eigenvalue empirical CDFs, tolerance-based classification (793 minima, one saddle), and log condition numbers for positive-curvature cases. Bottom row: gradient norms, projections of the dominant eigenvector on the azimuth/elevation basis, and angular differences between fitted reference directions and empirical optima. The original export reports a 99.9\% minimum classification with bootstrap interval $[99.6,100.0]\%$, median gradient norm near $1.5\times10^{-3}$, and condition-number mode near 1.8. These describe the inspected sample and numerical criteria. The unit quarter-arc in the projection panel is a basis identity. Angular disagreement, with a mode near $10^{\circ}$, compares two fitting objectives, not an estimate against independent solar ground truth.}
    \label{fig:prove_and_visualize_local_convexity}
\end{figure}

\textbf{Terrain-level.} Figure~\ref{fig:spherical_loss_landscape} shows a smooth basin in the mean sampled loss around the reference illumination direction, consistent with the angular sensitivity curves. Averaging across terrain can hide secondary minima or flat regions in individual scenes. The aggregate surface therefore does not prove global convexity, uniqueness, or convergence from every initial direction.

\begin{figure}[htbp]
    \centering
    \includegraphics[width=\linewidth,height=0.80\textheight,keepaspectratio]{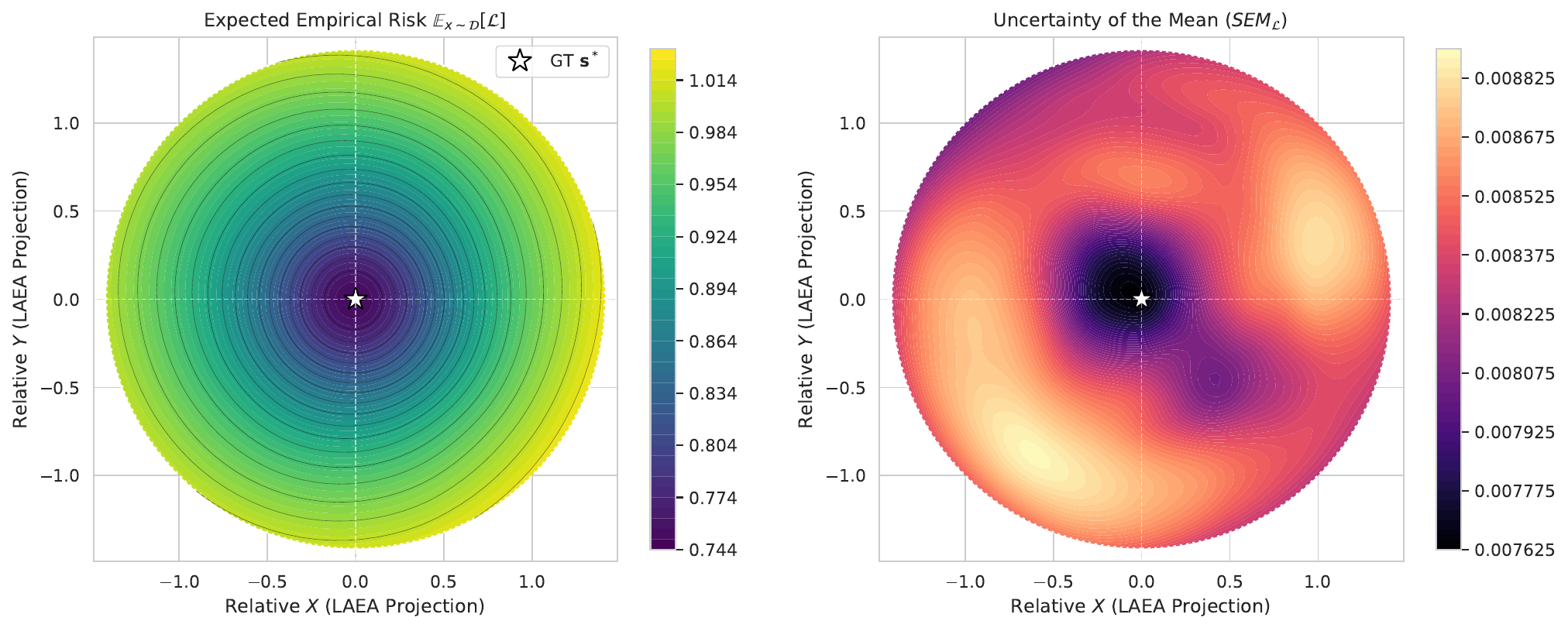}
\caption{\textbf{Mean illumination-loss landscape and uncertainty of the mean.} The left panel uses a Lambert azimuthal equal-area projection centered on the fitted reference direction. Its displayed loss scale spans approximately 0.744--1.014. The right panel's SEM scale spans approximately 0.007625--0.008825. The smooth average basin characterizes the sampled terrains and directions; SEM measures uncertainty of an average, not scene-to-scene ambiguity. Neither panel guarantees a unique per-scene optimum or global optimizer convergence.}
    \label{fig:spherical_loss_landscape}
\end{figure}

\textbf{Dataset-level.} Figure~\ref{fig:umap_global_disentanglement} is labeled 100 terrains and 50,000 samples. Its angular-error probe reports $R^2=0.917$ and distance correlation 0.496; nuisance probes report $R_I^2=-0.034$, $R_A^2=-0.115$, and distance correlation 0.059. Negative held-out $R^2$ means the tested regressor underperformed a mean predictor on that split, not that no dependence can exist. The feature definition, whether 50,000 is a total or per-terrain count, and the terrain grouping of probe splits need to accompany the export. The plot is retained as exploratory evidence; exact invariance follows from the algebra below under its stated assumptions.

\begin{figure}[htbp]
    \centering
    \includegraphics[width=\linewidth,height=0.80\textheight,keepaspectratio]{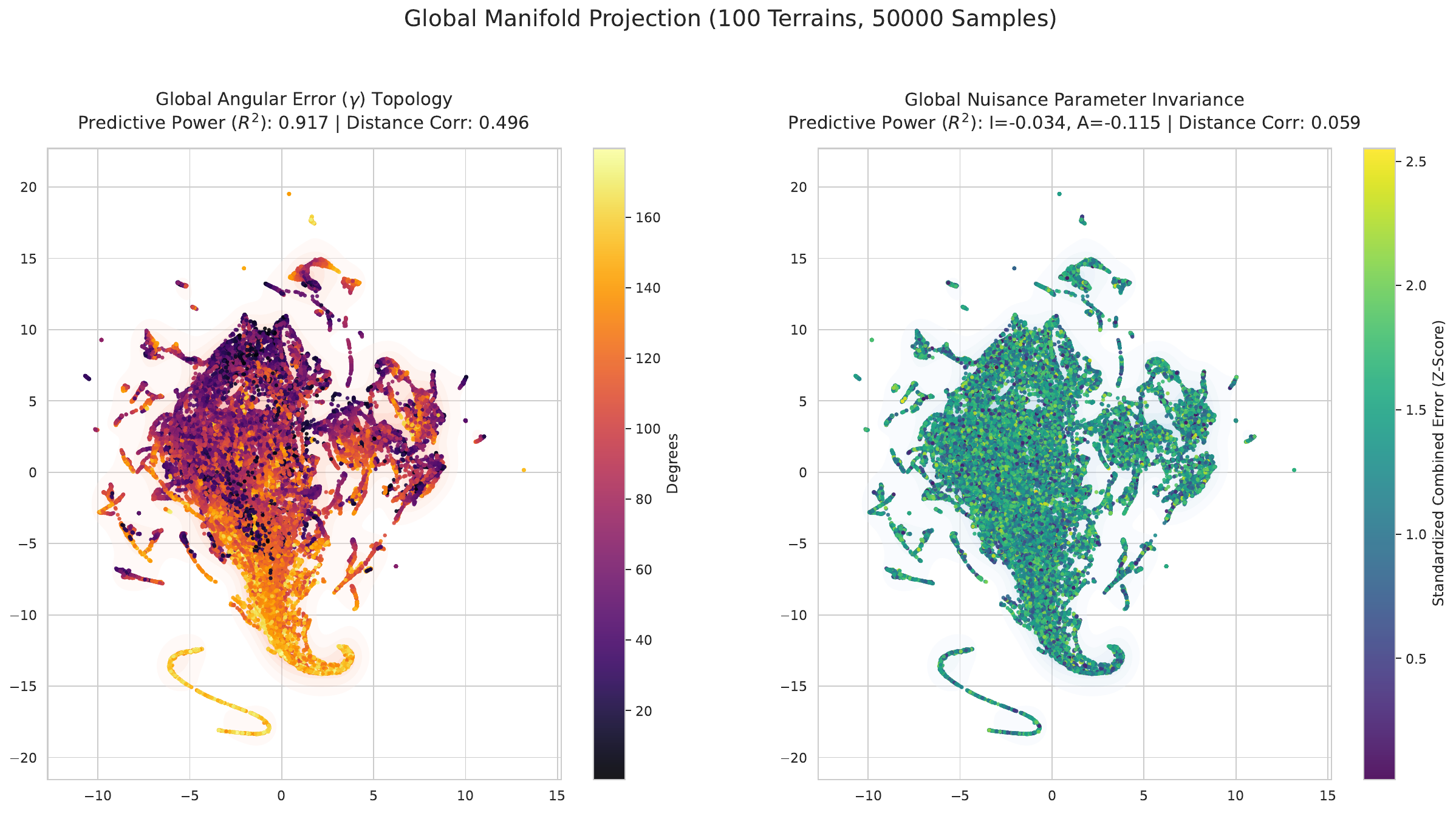}
    \caption{\textbf{Exploratory embedding and parameter probes.} The archived embedding is colored by angular error (left) and standardized gain/offset variation (right). The displayed angular probe has $R^2=0.917$; nuisance probes have $R^2=-0.034$ and $-0.115$. These scores summarize the specified probes, not a general proof of statistical independence. Reproduction requires the input feature construction and terrain-disjoint probe protocol; an embedding of parameters and an embedding of residual features are different experiments.}
    \label{fig:umap_global_disentanglement}
\end{figure}

These diagnostics support studying angular sensitivity separately from positive global exposure changes. They do not establish that illumination is the only source of variation in the reconstruction loss: terrain geometry, reflectance, registration, masking, and numerical stabilization remain relevant. The useful mathematical guarantee is exposure invariance under explicit assumptions, complemented by empirical local-curvature measurements.

\FloatBarrier
\subsection{Invariance of the Photoclinometric Loss to \texorpdfstring{$(I,A)$}{(I,A)}} \label{app:invariance_proof}

\paragraph{Forward model} Let $\mathbf{N}(p)\in\mathbb{S}^2$ denote the unit surface normal at pixel $p$, computed from the digital terrain model via finite differences. For fixed geometry and reflectance, write the render as $I g(p;\mathbf s)+A$. A Lambertian example is
\begin{equation} \label{eq:render}
I_r(p\,;\,\mathbf{s}, I, A) \;=\; I\,\bigl[\mathbf{N}(p)\!\cdot\!\mathbf{s}\bigr]_+ \;+\;A, \qquad \mathbf{s}\in\mathbb{S}^2,\; I>0,\; A\in\mathbb{R}, \end{equation}
where $[\,\cdot\,]_+ \!=\! \max(\cdot, 0)$ enforces nonnegative local incidence in this illustrative model; it does not model cast shadows. The same proof applies to a fixed Lunar--Lambert spatial response without output clipping. Crucially, $(I, A)$ enter Eq.~\eqref{eq:render} only as a positive affine map of the spatial shading pattern $g(p; \mathbf{s}) \!\triangleq\! [\mathbf{N}(p)\!\cdot\!\mathbf{s}]_+$.  \paragraph{Loss.} Given an observed orthorectified image $I_o$ and a binary validity mask $M$, the photoclinometric loss is
\begin{equation} \label{eq:loss}
\mathcal{L}(\mathbf{s}, I, A; \mathcal{D}) \;=\; (1-\alpha)\bigl[1 - r\!\left(\tilde z_r,\,\tilde z_o\right)\bigr] \;+\; \alpha\bigl[1 - \mathrm{SSIM}\!\left(\tilde z_r,\,\tilde z_o\right)\bigr],
\end{equation} where $\tilde z_x = z_M(x)$ denotes the masked $z$-score $z_M(x) = (x - \mu_M(x)) / \sigma_M(x)$, and $r$ is the masked Pearson correlation.

\paragraph{Invariance to positive affine intensity transformations} For any $a>0$ and $b\in\mathbb{R}$ and any field $x$ with $\sigma_M(x)>0$, \begin{equation} \label{eq:zscore_invariance} z_M(a x + b) \;=\; \frac{(a x + b) - (a \mu_M(x) + b)}{a\,\sigma_M(x)} \;=\; z_M(x). \end{equation}

\paragraph{Theorem (Invariance to $(I,A)$).} Fix geometry, reflectance, and a common validity mask, and assume nonzero variances of the rendered and observed fields. For gains $I,I'>0$ and offsets $A,A'\in\mathbb R$,
\begin{equation}
\label{eq:invariance_thm}
\mathcal{L}(\mathbf{s}, I', A'; \mathcal{D})
\;=\;
\mathcal{L}(\mathbf{s}, I, A; \mathcal{D}).
\end{equation}

\noindent\textit{Proof} From Eq.~\eqref{eq:render}, the rendered field under the alternative photometry satisfies
\begin{equation}
I_r(p; \mathbf{s}, I', A')
\;=\;
\frac{I'}{I}\,\bigl[I_r(p; \mathbf{s}, I, A) - A\bigr] + A'
\;=\;
a\,I_r(p; \mathbf{s}, I, A) + b,
\end{equation}
with $a = I'/I > 0$ and $b = A' - aA$. Applying Eq.~\eqref{eq:zscore_invariance} pixel-wise on the support of $M$ yields $z_M(I_r(\cdot; \mathbf{s}, I', A')) = z_M(I_r(\cdot; \mathbf{s}, I, A))$, and the observed $z$-score $z_M(I_o)$ is independent of $(\mathbf{s}, I, A)$. Since both terms of Eq.~\eqref{eq:loss} are functions solely of $(z_M(I_r), z_M(I_o))$, the claim follows. $\square$

\paragraph{Remark (local SSIM and practical scope).} Global z-scoring does not force every local window to have zero mean and unit variance. Local SSIM therefore retains its luminance, contrast, and structure factors:
\begin{equation}
\mathrm{SSIM}_{W}(z_1,z_2)=
\frac{2\mu_{1,W}\mu_{2,W}+C_1}{\mu_{1,W}^2+\mu_{2,W}^2+C_1}
\frac{2\sigma_{12,W}+C_2}{\sigma_{1,W}^2+\sigma_{2,W}^2+C_2}.
\end{equation}
The proof holds because the \emph{entire normalized input fields} are unchanged, so every local window receives identical values after a positive affine transformation. Average pooling also commutes with such a transformation on fixed support, extending the argument to each pyramid scale. Negative gains, gain-dependent masks, post-gain clipping, and degenerate variance are excluded. A fixed additive variance floor generally makes scale invariance approximate rather than exact. Changes in geometry, $L$, or spatially varying albedo are outside the theorem.

\FloatBarrier
\subsection{Riemannian Geometry of $\mathcal{L}$ on $\mathbb{S}^2$} \label{app:riemannian_geometry}

The unit-norm constraint $\|\mathbf{s}\|=1$ confines optimisation to the sphere $\mathbb{S}^2 \subset \mathbb{R}^3$. The tangent space at $\mathbf{s}^*$ is the codimension-one subspace
\begin{equation}
T_{\mathbf{s}^*}\mathbb{S}^2
\;=\;
\{\,v \in \mathbb{R}^3 : v^\top \mathbf{s}^* = 0\,\}.
\end{equation}
Let $U \in \mathbb{R}^{3\times 2}$ be an orthonormal basis of $T_{\mathbf{s}^*}\mathbb{S}^2$, so that $U^\top U = I_2$ and $U U^\top = I_3 - \mathbf{s}^* \mathbf{s}^{*\top}$.

\paragraph{Riemannian gradient} The Euclidean gradient $\nabla_E \mathcal{L} \in \mathbb{R}^3$ projects orthogonally onto the tangent space:
\begin{equation}
\nabla_R \mathcal{L}(\mathbf{s}^*)
\;=\;
(I_3 - \mathbf{s}^* \mathbf{s}^{*\top})\,\nabla_E \mathcal{L},
\qquad
\widetilde{\nabla}_R \mathcal{L}
\;\triangleq\;
U^\top \nabla_E \mathcal{L}\;\in\;\mathbb{R}^2.
\end{equation}

\paragraph{Riemannian Hessian} For $\mathbb{S}^2$ embedded in $\mathbb{R}^3$ with the round metric, the Riemannian Hessian in coordinates $U$ is~\cite{AbsMahSep2008}
\begin{equation}
\label{eq:riemannian_hessian}
H_R(\mathbf{s}^*)
\;=\;
\underbrace{U^\top H_E(\mathbf{s}^*)\, U}_{\text{embedded Hessian}}
\;-\;
\underbrace{\bigl\langle \nabla_E \mathcal{L}(\mathbf{s}^*),\,
\mathbf{s}^* \bigr\rangle\, I_2}_{\text{Weingarten correction}}
\;\in\;\mathbb{R}^{2\times 2},
\end{equation}
where the second term arises from the second fundamental form of $\mathbb{S}^2$ in $\mathbb{R}^3$ and is required to obtain the intrinsic curvature.

\paragraph{Stationarity and topological classification} A point $\mathbf{s}^*\!\in\!\mathbb{S}^2$ is a strict local minimum of a twice continuously differentiable $\mathcal{L}$ if
\begin{equation}
\widetilde{\nabla}_R \mathcal{L}(\mathbf{s}^*) = 0
\quad\text{and}\quad
\lambda_{\min}(H_R(\mathbf{s}^*)) > 0.
\end{equation}
These are sufficient second-order conditions; they are not necessary for a strict minimum, which can have a singular Hessian. The numerical diagnostic declares a candidate approximately stationary when $\|\widetilde\nabla_R\mathcal L\|_2/\max_i|\lambda_i(H_R)|<10^{-2}$, provided the denominator is nonzero, and uses spectral tolerance $10^{-4}|\lambda_{\max}|$. A denominator guard and treatment of nonsmooth model operations must be specified. Under these criteria the archived classifier reports 793 of 794 candidates as minima. This finite-sample proportion and its bootstrap interval characterize the inspected fits; they do not constitute a population guarantee for illumination inversion, still less for joint terrain reconstruction.

\paragraph{Justification}
The shading response depends on the local derivatives of the predicted relief. Comparing it with an image can therefore supply orientation-related information beyond an interpolated stereo target, an idea established in planetary photoclinometry~\cite{Kirk2003High-resolutionImages,Beyer2018TheData,Alexandrov2018MultiviewImages}. The benefit depends on actual grid spacing, metric normals, reflectance assumptions, and visibility. Our formulation makes this constraint differentiable; identifying it as the most important loss, or claiming a particular resolution gain, requires ablation and independent validation. Figure~\ref{fig:photo_physics} retains the intermediate fields used to audit the render and comparison.
\begin{figure}[htbp]
\centering
\includegraphics[width=\linewidth,height=0.80\textheight,keepaspectratio]{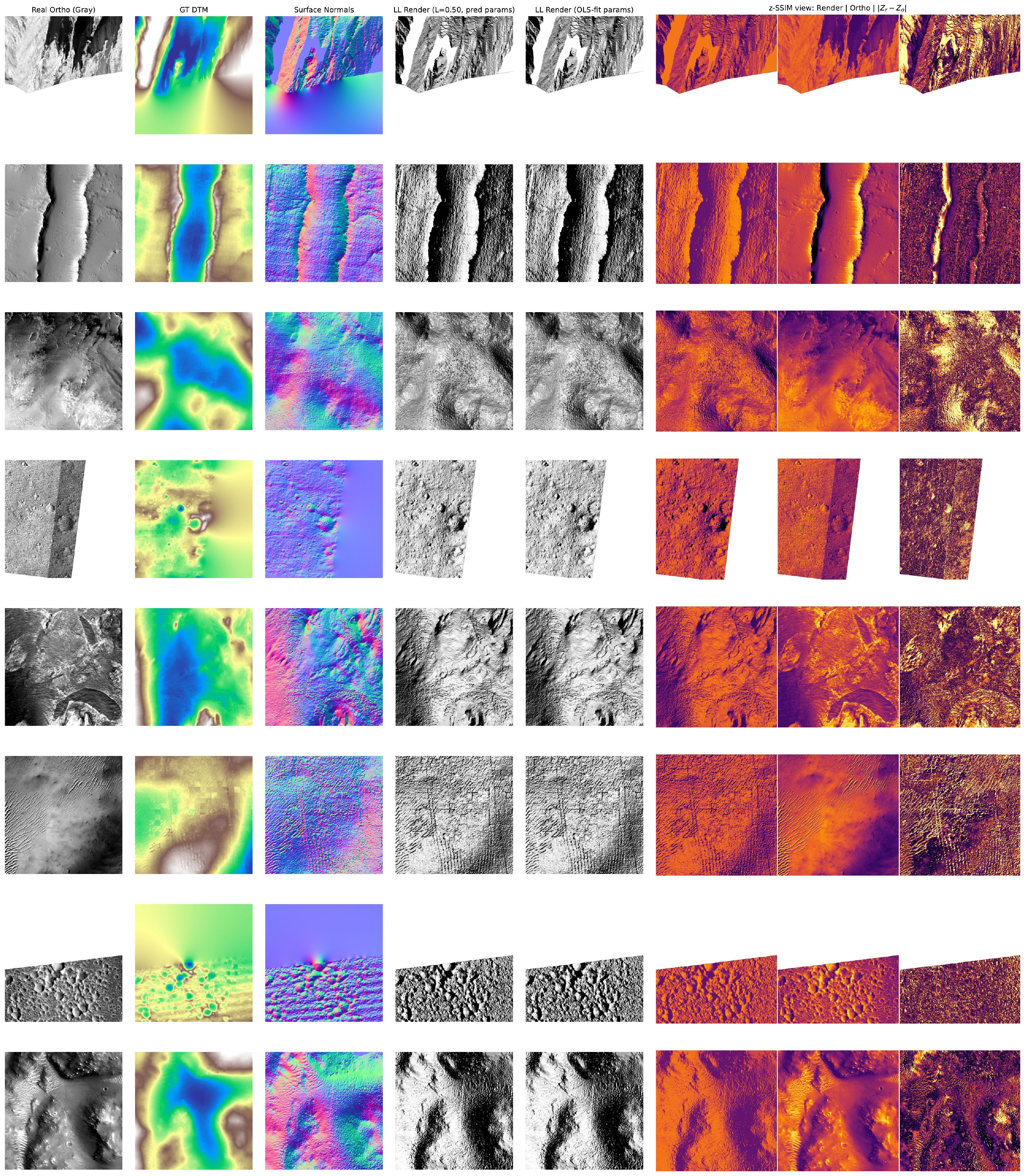}
\caption{\textbf{Intermediate shading and comparison fields.} The complete diagnostic grid retains the input orthoimage, reference DTM, normal field, two render computations, and the z-normalized render/image/residual views. Reference normals are used for this inspection; prediction normals supply the training objective. The absolute z-residual is a diagnostic image, not the local SSIM map itself. Bright residuals indicate spatial disagreement that may arise from terrain, reflectance, shadows, registration, or illumination; global normalization removes only positive affine exposure changes under the stated assumptions.}
\label{fig:photo_physics}
\end{figure}

\FloatBarrier
\subsection{Summary and Loss Interaction}
\label{sec:loss_summary}

Table~\ref{tab:loss_summary} summarises the role of each term, including whether it targets GT-bounded or full-orthoimage-resolution information, and its sensitivity to stereo GT noise.
\begin{table}[t]
\centering
\caption{Role of each loss term. ``Scale'' denotes the spatial frequency each term primarily supervises.}
\label{tab:loss_summary}
\small
\begin{tabular}{lllll}
\toprule
Loss & Supervises & Scale & Noise-robust? & Section \\
\midrule
$\mathcal{L}_{\text{FM}}$    & latent velocity        & all latent    & partial  & \ref{sec:fm_loss} \\
$\mathcal{L}_{\text{huber}}$ & absolute depth         & values & reduced tail influence  & \ref{sec:huber_loss} \\
$\mathcal{L}_{\text{norm}}$  & local slope            & mid--high     & partial  & \ref{sec:normal_loss} \\
$\mathcal{L}_{\text{grad}}$  & slope multi-scale      & low--high     & offset-invariant  & \ref{sec:grad_loss} \\
$\mathcal{L}_{\text{lap}}$   & curvature              & high          & partial  & \ref{sec:laplacian_loss} \\
$\mathcal{L}_{\text{FFL}}$   & frequency bands        & adaptive bands & no  & \ref{sec:ffl_loss} \\
$\mathcal{L}_{\text{ord}}$   & relative ordering      & all           & conditional  & \ref{sec:ordinal_loss} \\
$\mathcal{L}_{\text{photo}}$ & \textbf{ortho shading} & image grid & model-dependent  & \ref{sec:photo_loss} \\
\bottomrule
\end{tabular}
\end{table}
The objective separates complementary constraints. Flow matching and Huber regression anchor latent and decoded values; normals and multi-scale gradients compare first-order structure; the Laplacian weights second derivatives; ordinal supervision penalizes selected order reversals; and the optional FFL emphasizes difficult spectral coefficients. Photoclinometry adds image-pattern agreement under an approximate reflectance model. These roles are design motivations, not independently verified benefits. In particular, shift invariance is not general noise robustness, and image-based comparison still depends on reference-derived illumination in the reported preparation pipeline. The retained diagnostics make these interactions and limitations inspectable.

The full pipeline of loss culminates in the Figure \ref{fig:flow_and_losses}.
\begin{figure}[htbp]
    \centering
    \includegraphics[width=1\textwidth,height=0.80\textheight,keepaspectratio]{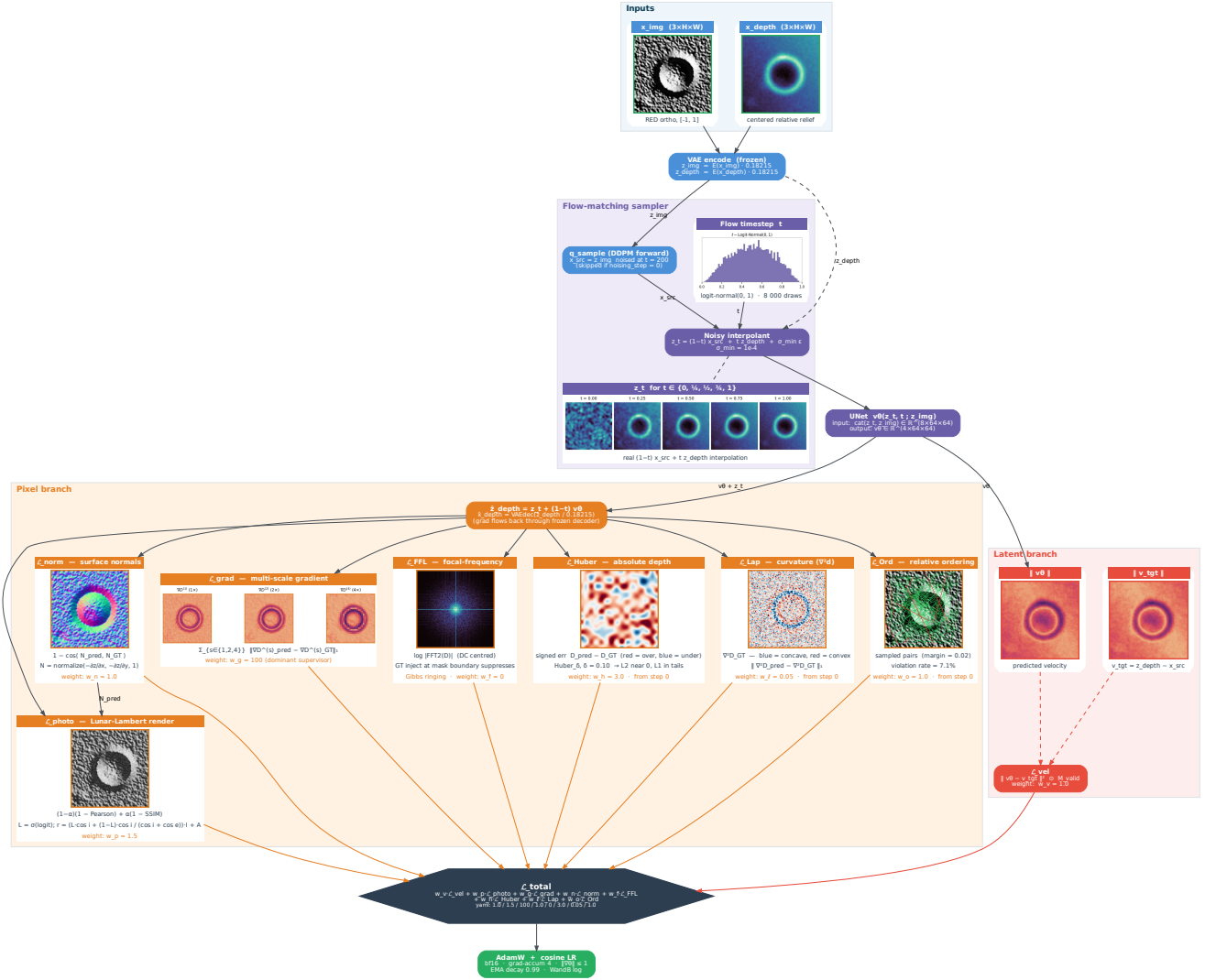}
    \caption{\textbf{Archived flow-matching and loss architecture.} The complete diagram preserves the latent encoding, interpolant, velocity prediction, endpoint decoding, eight objective branches, and weighted optimization. VAE parameters are frozen while the decoder remains differentiable with respect to its input. FFL is optional and has zero reported weight. The revised Equation~\eqref{eq:cfm} includes source augmentation and logit-normal time sampling. Geometric losses compare VAE-decoded fields. The archived run shades normalized decoded relief; the current code first inverts the signed logarithm. The recovered run configuration governs the reported results, while the diagram retains historical schematic annotations.}
\label{fig:flow_and_losses}
\end{figure}

\newpage
%
%

\FloatBarrier
\section{Visual Results and Diagnostic Figures}
\label{sec:visual_results}

This appendix examines the spatial behavior behind the aggregate metrics: terrain correspondence, amplitude compression, smoothing, boundaries, derivative structure, shading, spectra, flow evolution, and sampling variability. The primary evaluation in Table~\ref{tab:metrics_main} contains 2024 gathered records per step count. The additional test export and individual diagnostic panels below provide complementary evidence under their stated protocols.

\textbf{Evaluation provenance and units.} Run \texttt{88c8} records source revision \texttt{6a4ea05}. The primary step sweep gathers 2024 records per count and compares VAE-decoded normalized fields. Table~\ref{tab:metrics_legacy} preserves the separate 1012-record rank-zero export at checkpoint 9476, using 16 Euler steps, one realization, and the normalized input reference. We computed its photo-consistency score separately; the primary sweep does not compute that metric. Eleven individual panels match the run's step-9827 validation PDFs, including the $N=8$ sampling-variance example. Relief in these panels is normalized, so their historical meter labels do not denote physical height. The figures remain unchanged to preserve the underlying evidence.

\textbf{Configuration and implementation version.} The launch command is \texttt{bash scripts/training/launch\_train.sh --config configs/train\_hirise\_large.yaml --n\_runs 1}. The script forwards the chosen configuration and respects externally selected visible GPUs; the recorded run used two devices. The current configuration specifies $S_{\mathrm{ref}}=45.9075$\,m and disables clipping, whereas the recorded experiment uses a percentile scale and clipping. The current evaluator also inverts the height transform. These later changes are not retroactively applied to the reported scores. The cached illumination parameters use the newer decoupled routine described in Appendix~\ref{sec:sun_irls}; we state its update rule there. A separate scale-selection study reports approximately 76.8\,m under a mean clipping-error budget of 0.1\,m, with a 76.4\,m tabulated export. Both are retained as preprocessing diagnostics; this discrepancy is not resolved by choosing the current configuration's scale.

\begin{table}[t]
\centering
\caption{Recovered MarsFM test export from run \texttt{88c8}, checkpoint 9476: mean $\pm$ sample standard deviation over 1012 rank-zero patch records, with 16 Euler steps and one realization. Relief errors follow affine alignment in normalized signed-log space. The default positive-reference mask restricts scoring support. Angular values are grid-coordinate diagnostics, not physical-slope errors. Numerical rounding follows the recovered CSV.}
\label{tab:metrics_legacy}
\small
\setlength{\tabcolsep}{2pt} 
\begin{tabular}{@{}lcc@{}} 
\toprule
Metric & Direction & Value \\
\midrule
\multicolumn{3}{@{}l@{}}{\textit{Differential geometry diagnostics}} \\
Normal angular error ($\degree$)   & $\downarrow$ & $0.380 \pm 0.206$ \\
Slope RMSE ($\degree$)             & $\downarrow$ & $0.489 \pm 0.303$ \\
Curvature RMSE ($\times 10^{-3}$)  & $\downarrow$ & $2.93 \pm 1.79$ \\
Depth-Boundary F-score             & $\uparrow$   & $0.075 \pm 0.147$ \\
\midrule
\multicolumn{3}{@{}l@{}}{\textit{Structural, image, and spectral diagnostics}} \\
MS-SSIM-topo                       & $\uparrow$   & $0.580 \pm 0.109$ \\
Photo-consistency (SSIM)           & $\uparrow$   & $0.220 \pm 0.130$ \\
PSD slope ratio (1.0 = ideal)      & $\to 1$      & $1.074 \pm 0.132$ \\
Patch SWD                          & $\downarrow$ & $0.0295 \pm 0.0660$ \\
\midrule
\multicolumn{3}{@{}l@{}}{\textit{Normalized relief and depth scores}} \\
RMSE                                & $\downarrow$ & $0.098 \pm 0.085$ \\
RMSE-log                            & $\downarrow$ & $1.131 \pm 0.202$ \\
$\delta_1$ (\%)                     & $\uparrow$   & $23.9 \pm 8.4$ \\
$\delta_2$ (\%)                     & $\uparrow$   & $45.0 \pm 12.2$ \\
$\delta_3$ (\%)                     & $\uparrow$   & $60.7 \pm 11.9$ \\
\bottomrule
\end{tabular}
\end{table}

\FloatBarrier
\subsection{Per-Patch Qualitative Predictions}
\label{sec:vis_qualitative}

\paragraph{Triptychs.} Figure~\ref{fig:triptych_gallery} contains one input--prediction--reference triptych, labeled step 9827. It shows a cratered patch, with a shared relief color scale permitting comparison of broad structure and local variation. The prediction captures some spatial organization but has a different amplitude distribution and smoother detail. A stratified gallery covering dunes, cratered plains, channelized scarps, and blocky ejecta would be useful for a future terrain-class evaluation; the present panel alone does not establish performance across those regimes. Exporter: \texttt{plot\_prediction\_triptych}.

The useful evidence in this example is the coexistence of broad morphological correspondence and visible reconstruction error. Image texture alone cannot identify the correct relief, and extra texture in a prediction could reflect either geometry or appearance leakage. The triptych should therefore be read alongside the cross-sections, derivative maps, and pixel scatter, which expose amplitude and spatial disagreement that a color rendering can hide. A matched no-photoclinometry run is needed to attribute an improvement to that objective.

\begin{figure}[htbp]
\centering
\includegraphics[width=\linewidth,height=0.80\textheight,keepaspectratio]{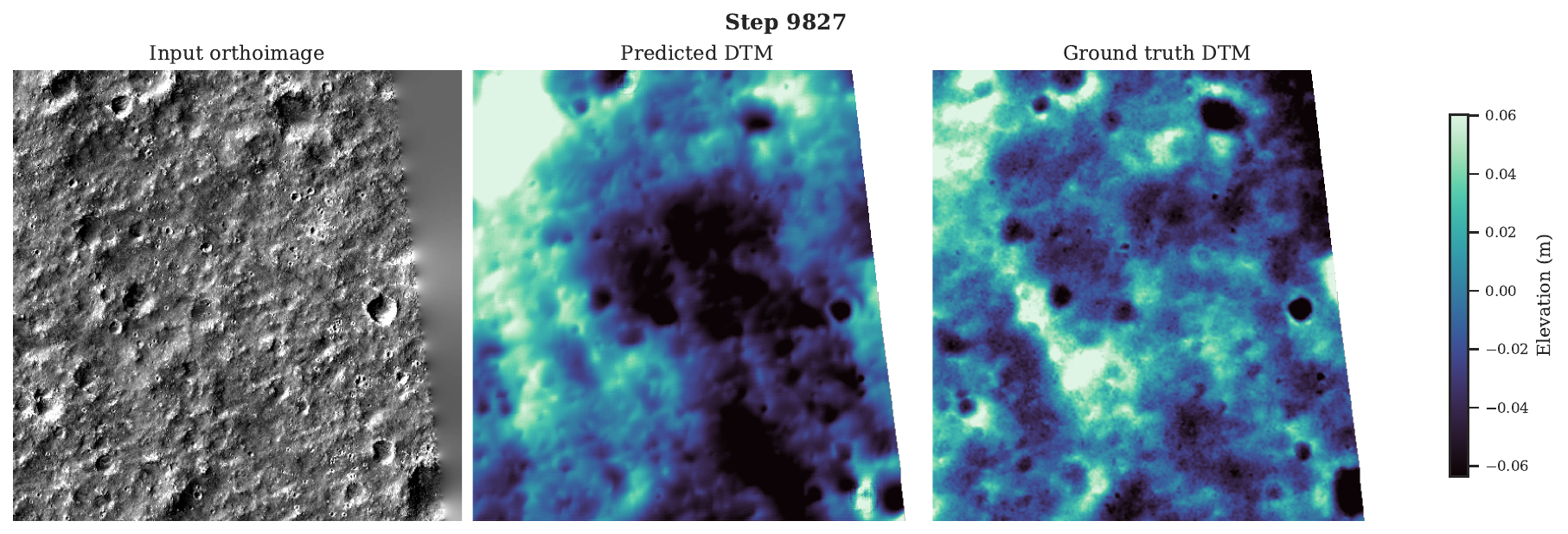}
\caption{\textbf{Prediction triptych, validation step 9827.} Left: HiRISE RED input; center: predicted relief; right: stereo-derived reference. The common color scale reveals broad terrain correspondence, amplitude differences, and smoothing. The historical meter annotation represents normalized relief in this export, not physical height. Exporter: \texttt{plot\_prediction\_triptych}.}
\label{fig:triptych_gallery}
\end{figure}

\paragraph{Cross-sectional profiles.} Figure~\ref{fig:cross_sections} compares horizontal, vertical, NW--SE, and NE--SW slices through the same archived example. Prediction and reference agree in parts of the broad envelope, but their separation extends across substantial segments and sharp transitions. The shaded residual therefore exposes amplitude and offset discrepancies, rather than verified sub-meter edge recovery. Four slices from one patch do not test directional independence or rule out illumination bias. Exporter: \texttt{plot\_cross\_sections}.

\begin{figure}[htbp]
\centering
\includegraphics[width=\linewidth,height=0.80\textheight,keepaspectratio]{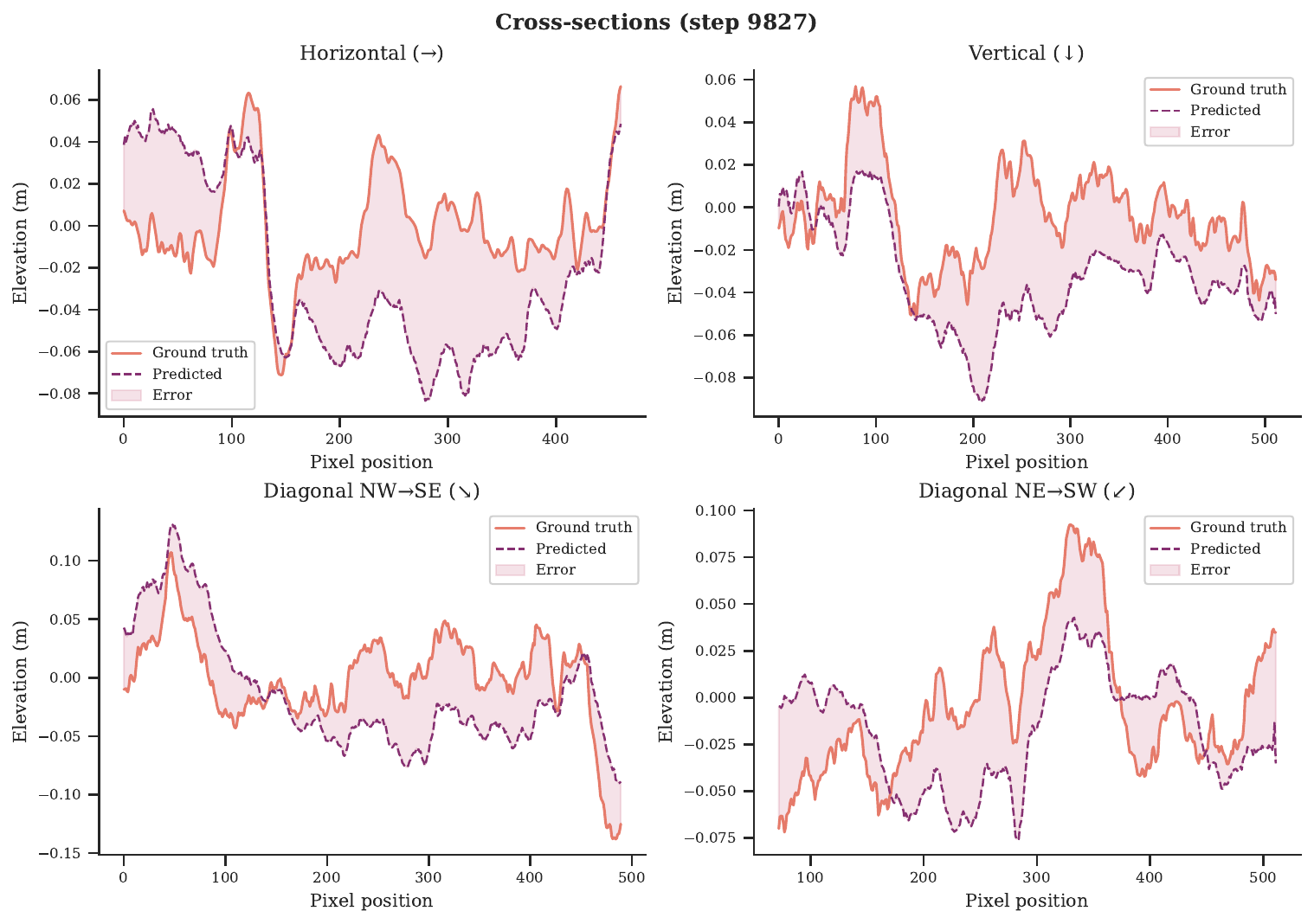}
\caption{\textbf{Directional cross-sections, step 9827.} Solid reference, dashed prediction, and shaded difference along four central slices. Broad residual regions remain in addition to errors at sharp transitions. The retained elevation-unit labels require export verification. Exporter: \texttt{plot\_cross\_sections}.}
\label{fig:cross_sections}
\end{figure}

\FloatBarrier
\subsection{Geometric Fidelity Diagnostics}
\label{sec:vis_geometry}

The supporting differential diagnostics in Table~\ref{tab:metrics_legacy} report normal angular error $0.38^{\circ}$, slope RMSE $0.49^{\circ}$, and represented-field Laplacian RMSE $2.93\times10^{-3}$. The example panels below show their spatial counterparts, but use the checkpoint identified in each title. Because derivatives depend on height representation and grid spacing, these diagnostics require the metric-coordinate conversion in Section~\ref{sec:normal_loss} before interpretation as physical slope accuracy.

\paragraph{Surface normal maps.} Figure~\ref{fig:normal_maps} maps $(n_x,n_y,n_z)$ to RGB. The prediction is visibly smoother and less varied than the reference over much of the patch, consistent with attenuated small-scale relief. Differences extend beyond obvious shadow boundaries. The spatial factor used by \texttt{compute\_surface\_normals} must be combined with height inversion and metric spacing before these colors represent physical surface orientations. Exporter: \texttt{plot\_normal\_maps}.

\begin{figure}[htbp]
\centering
\includegraphics[width=\linewidth,height=0.80\textheight,keepaspectratio]{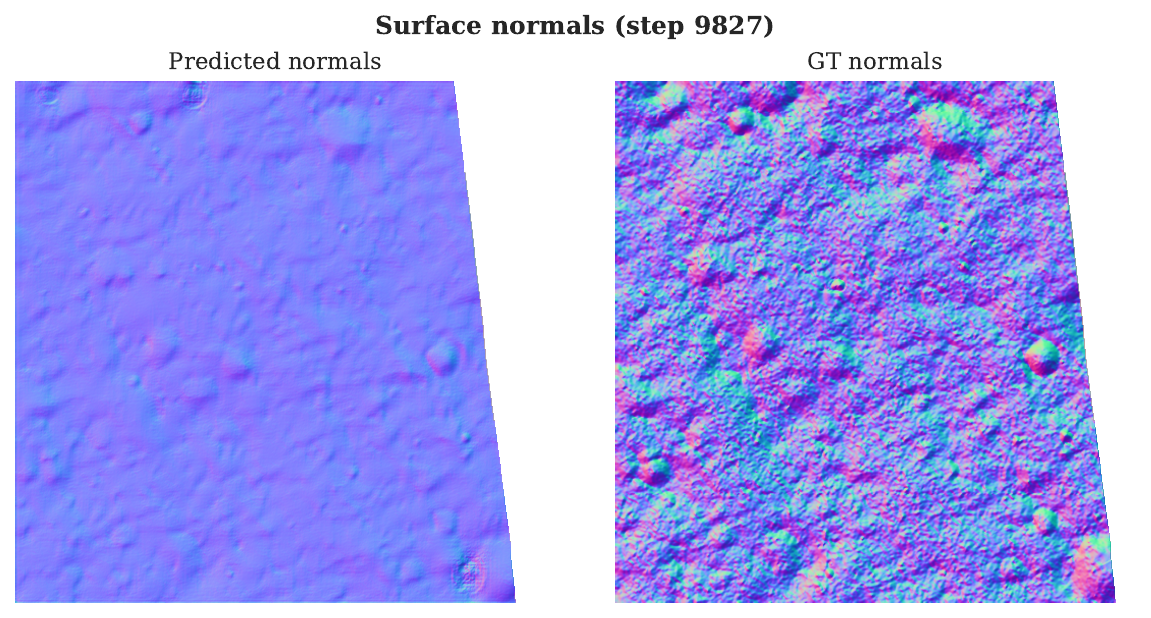}
\caption{\textbf{RGB normal fields, step 9827.} Prediction (left) and reference (right), with channels encoding $(n_x,n_y,n_z)$. The prediction contains less fine-scale variation. Angles refer to the operator's coordinate convention until metric conversion is verified. Exporter: \texttt{plot\_normal\_maps}.}
\label{fig:normal_maps}
\end{figure}

\paragraph{Slope error maps.} Figure~\ref{fig:slope_error} shows the orthoimage, reference slope, and absolute prediction--reference slope difference. Errors occur around crater structures and across textured terrain; the panel does not demonstrate that all errors are confined to unresolved boundaries. The displayed angular values depend on the gradient convention. A percentile statement requires direct extraction from the scored mask and cannot be inferred from the dataset RMSE. Exporter: \texttt{plot\_slope\_error\_map}.

\begin{figure}[htbp]
\centering
\includegraphics[width=\linewidth,height=0.80\textheight,keepaspectratio]{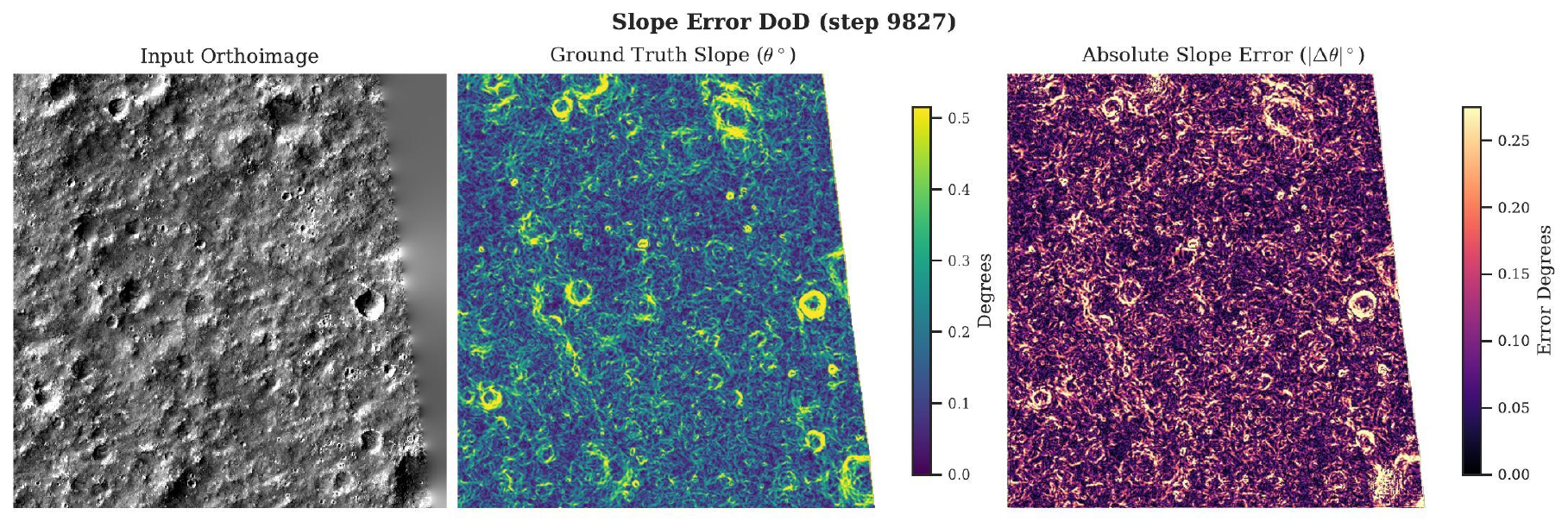}
\caption{\textbf{Slope diagnostic, step 9827.} Input image, reference slope, and absolute slope difference. The image shows spatially distributed disagreement; the retained degree labels describe the exported operator, not independently verified physical slopes. Exporter: \texttt{plot\_slope\_error\_map}.}
\label{fig:slope_error}
\end{figure}

\FloatBarrier
\subsection{Photoclinometric Loss Diagnostics}
\label{sec:vis_photo}

The photoclinometric loss compares a predicted render with image intensity on the processed grid. Its value depends on geometry, fitted illumination, reflectance assumptions, and valid support. Visualizing its inputs is therefore useful even when reference-relief metrics improve: a low-variance render or an incorrect spatial illumination pattern can undermine the intended supervision without being obvious from an aggregate depth score.

Figure~\ref{fig:lunar_lambert} retains four fields: the input orthoimage, the prediction render, the reference-DTM render under the same illumination, and an absolute z-score residual. The reference render is a comparator, not an upper bound on achievable image agreement. The residual $|Z_{\mathrm{pred}}-Z_{\mathrm{ortho}}|$ is also not the local SSIM penalty: SSIM uses window statistics. The figure reports a blend near $L=0.49$. Exporter: \texttt{plot\_lunar\_lambert\_comparison}.

\begin{figure}[htbp]
\centering
\includegraphics[width=\linewidth,height=0.80\textheight,keepaspectratio]{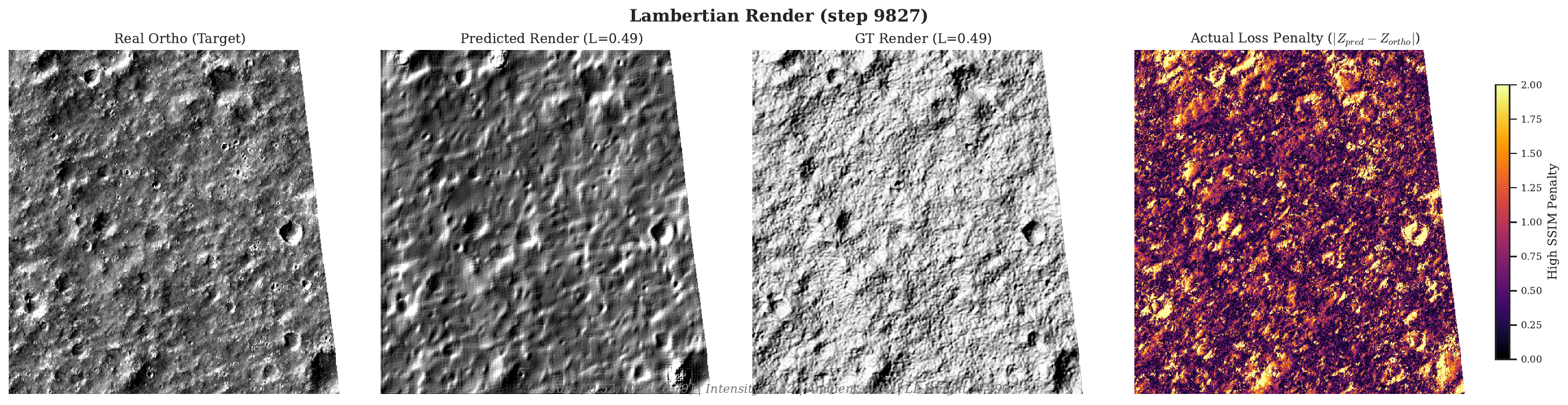}
\caption{\textbf{Archived shading comparison, step 9827.} Input orthoimage, predicted-relief render, reference-relief render, and absolute z-normalized residual. The last panel is labeled as a loss view in the original asset, but it is not itself an SSIM map. Reference rendering provides a comparison under the same surrogate; it is not an accuracy bound. Exporter: \texttt{plot\_lunar\_lambert\_comparison}.}
\label{fig:lunar_lambert}
\end{figure}

The predicted and reference renders show different texture and contrast patterns, and the residual remains spatially structured. These discrepancies are scientifically useful: they identify where the shading model and the image disagree. They cannot be assigned uniquely to incorrect topography, because albedo, shadows, registration, and illumination errors also contribute. A visually closer render can coexist with worse geometric accuracy; independent terrain measurements are needed to distinguish those outcomes.

\FloatBarrier
\subsection{Frequency-Domain and Geomorphometric Analysis}
\label{sec:vis_frequency}

\paragraph{Radial PSD.} Figure~\ref{fig:radial_psd} compares radial power spectra for one archived prediction and reference. The dotted $f^{-2.5}$ curve is an illustrative slope comparator, not a universal law of planetary terrain. The aggregate slope ratio in Table~\ref{tab:metrics_main} summarizes decay-rate similarity; it does not compare spectral phase or feature location. The horizontal axis extends above one despite its cycles/pixel label, indicating that the exporter must be checked for radial-bin or index coordinates before converting it to a physical wavelength.

The prediction lies below the reference over much of the intermediate and high radial range, while exceeding it at the very lowest frequencies and in the extreme tail. This pattern is consistent with smoothing plus redistribution of spectral power. The supporting 16-step export has a slope ratio of 1.074; this is not a bandwidth ratio, and excess tail power alone cannot distinguish recovered detail from noise or artifacts. The full spectrum is retained because it reveals behavior hidden by a single fitted slope. Exporter: \texttt{plot\_radial\_psd\_curves}.

\begin{figure}[htbp]
\centering
\includegraphics[width=\linewidth,height=0.80\textheight,keepaspectratio]{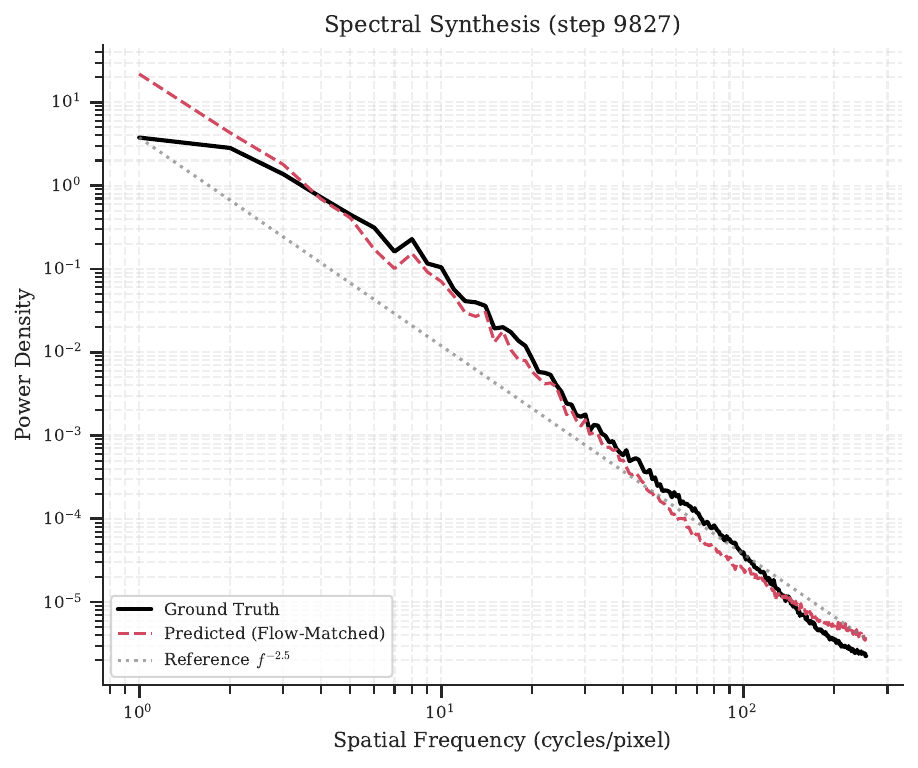}
\caption{\textbf{Radial power spectra, step 9827.} Solid reference, dashed prediction, and dotted $f^{-2.5}$ comparator. The prediction has less power over much of the plotted range and a small excess at the extreme tail. The original frequency-axis units require correction from the FFT metadata before making physical-scale claims. Exporter: \texttt{plot\_radial\_psd\_curves}.}
\label{fig:radial_psd}
\end{figure}

\paragraph{Edges and curvature.} Figure~\ref{fig:geomorph} compares thresholded edge sets, predicted Laplacian, and absolute Laplacian disagreement: blue marks reference-only edges, red prediction-only edges, and purple overlap. Much of the example is blue, exposing missed or weakened reference structure. A red-only edge could be an additional terrain feature or an error; the overlay does not distinguish them. The Laplacian sign depends on the stencil convention, and a normalized-field Laplacian is not physical surface curvature. Exporter: \texttt{plot\_geomorphometric\_analysis}.

\begin{figure}[htbp]
\centering
\includegraphics[width=\linewidth,height=0.80\textheight,keepaspectratio]{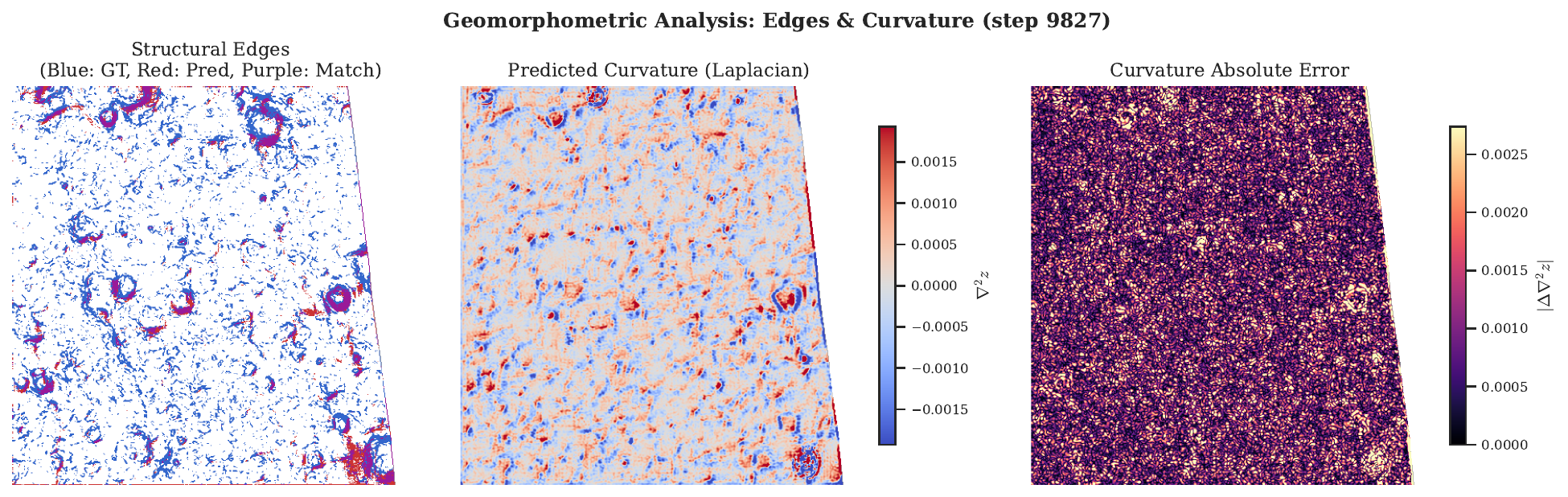}
\caption{\textbf{Edge and Laplacian diagnostics, step 9827.} Blue: reference-only edges; red: prediction-only; purple: overlap. The other panels show the predicted Laplacian and its absolute disagreement with the reference. Unmatched edges require investigation and are not automatically evidence of super-resolution. Exporter: \texttt{plot\_geomorphometric\_analysis}.}
\label{fig:geomorph}
\end{figure}

\FloatBarrier
\subsection{Per-Pixel Error Structure}
\label{sec:vis_error}

\paragraph{Error heatmap with marginal statistics.} Figure~\ref{fig:error_heatmap} retains the absolute pixel residual and its row/column means. The displayed example reports a mean of 0.021, a median of 0.017, a standard deviation of 0.017, and a 95th percentile of 0.053 in the exported scale. The matched historical exporter operates on normalized relief, so its meter labels are not physical units. Marginals expose broad spatial trends but do not identify their cause: terrain error, registration, mask boundaries, and reference artifacts can all contribute. Exporter: \texttt{plot\_error\_heatmap}.

\begin{figure}[htbp]
\centering
\includegraphics[width=\linewidth,height=0.80\textheight,keepaspectratio]{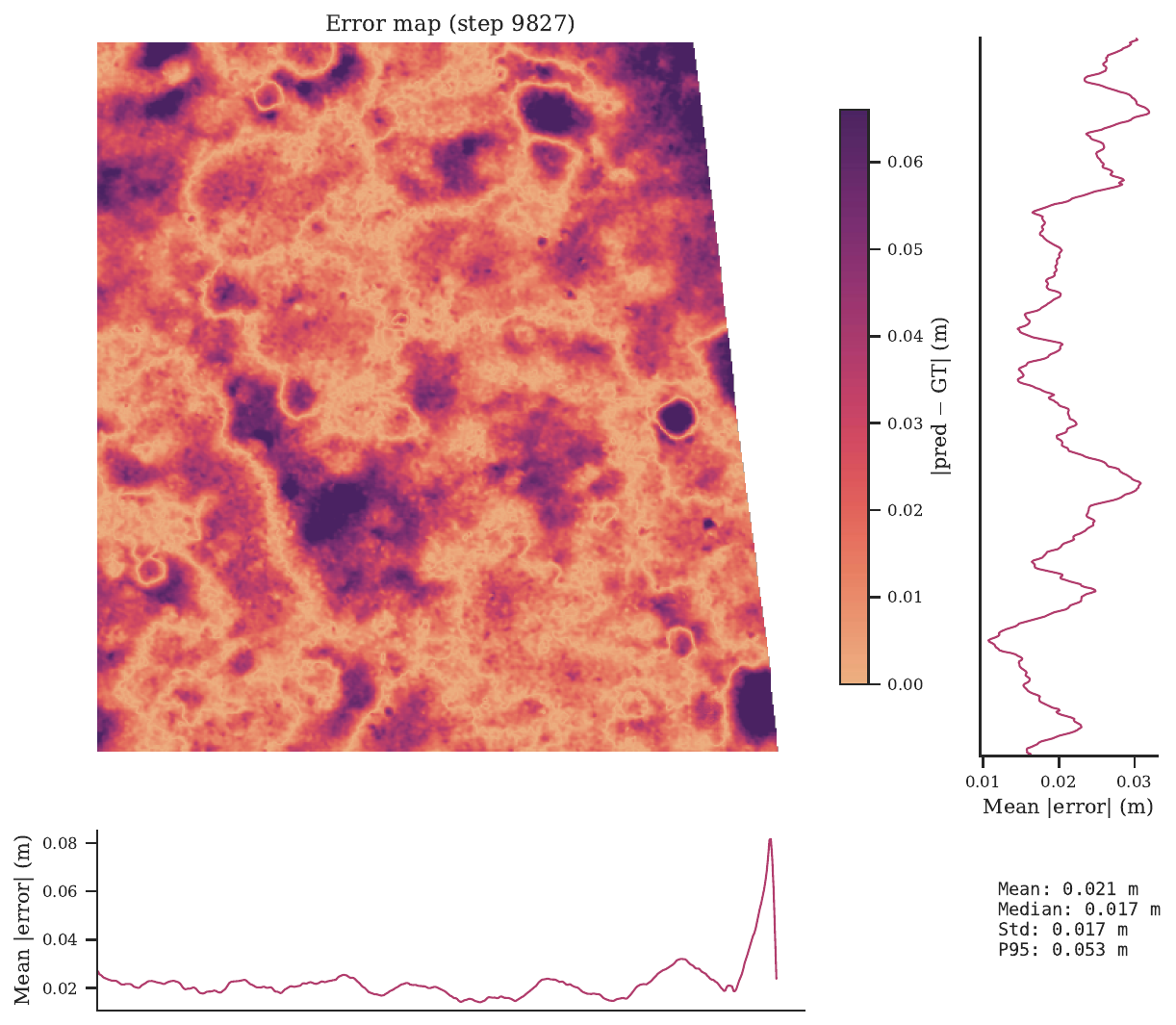}
\caption{\textbf{Absolute residual and marginal means, step 9827.} Main panel: pixelwise prediction--reference error. Right and bottom: row and column means. The summary values are retained, but the meter labels conflict with the recorded normalized-relief export. Exporter: \texttt{plot\_error\_heatmap}.}
\label{fig:error_heatmap}
\end{figure}

\paragraph{Pixel-level scatter.} Figure~\ref{fig:elevation_scatter} reports $R^2=0.2832$ for its displayed sample. The predicted range is compressed relative to the reference, and the cloud departs substantially from the identity line. This directly shows an amplitude mismatch in the example, not the previously stated $R^2\approx0.96$--$0.99$. Scatter around identity measures prediction error, not calibrated uncertainty. The stated plotting workflow subsamples approximately $5\times10^3$ valid pixels; exact sampling and alignment should be recovered from the export configuration. Exporter: \texttt{plot\_elevation\_scatter}.

\begin{figure}[htbp]
\centering
\includegraphics[width=0.7\linewidth,height=0.80\textheight,keepaspectratio]{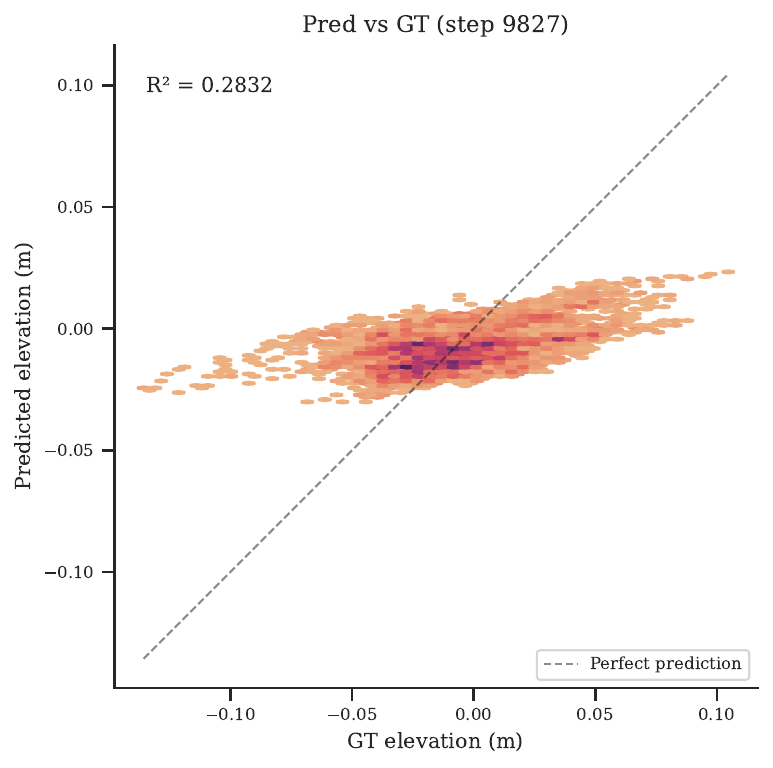}
\caption{\textbf{Prediction--reference hexbin scatter, step 9827.} The figure reports $R^2=0.2832$ and shows a compressed predicted range relative to the dashed identity line. Verify axis units and alignment upon export. Exporter: \texttt{plot\_elevation\_scatter}.}
\label{fig:elevation_scatter}
\end{figure}

\FloatBarrier
\subsection{Distributional Analysis Across the Test Set}
\label{sec:vis_distribution}

\paragraph{Per-metric violins.} Figure~\ref{fig:violins} contains four distributions: RMSE, absolute relative error, $\delta_1$, and normal angular error. It does not contain a boundary-F-score panel. The broad RMSE distribution supports reporting variation across patches rather than relying on the mean alone. Absolute-relative and ratio-based scores require a documented valid positive domain; near-zero signed relief can make them unstable or uninterpretable. We retain these legacy distributions for audit, without assigning the F-score's low mean to an unobserved bimodal mechanism. Exporter: \texttt{plot\_metric\_distributions}.

\begin{figure}[htbp]
\centering
\includegraphics[width=\linewidth,height=0.80\textheight,keepaspectratio]{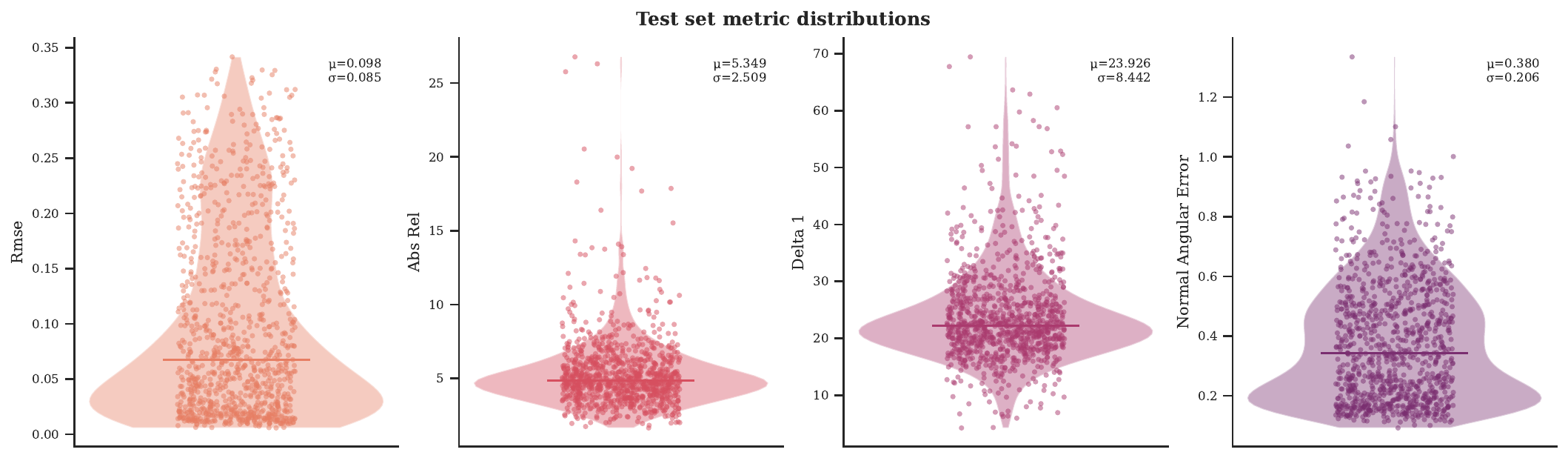}
\caption{\textbf{Four archived metric distributions.} Panels: RMSE, AbsRel, $\delta_1$, and normal angular error, with per-patch points and reported summary values. Ratio metrics require clarification for signed relief; the panel is not evidence of boundary-F-score bimodality. Exporter: \texttt{plot\_metric\_distributions}.}
\label{fig:violins}
\end{figure}

\FloatBarrier
\subsection{Flow-Matching Dynamics}
\label{sec:vis_flow}

\paragraph{Flow evolution.} Figure~\ref{fig:flow_evolution} decodes the state at $t\in\{0,0.25,0.5,0.75,1\}$ and places the reference alongside it. Early panels retain substantial image-like texture; later panels show a smoother relief representation. These decoded intermediate fields are not all calibrated terrain estimates. Because the decoder is nonlinear, the montage cannot establish straightness of the latent path. Measuring trajectory curvature would require the latent states themselves. Exporter: \texttt{plot\_flow\_evolution}.

\begin{figure}[htbp]
\centering
\includegraphics[width=\linewidth,height=0.80\textheight,keepaspectratio]{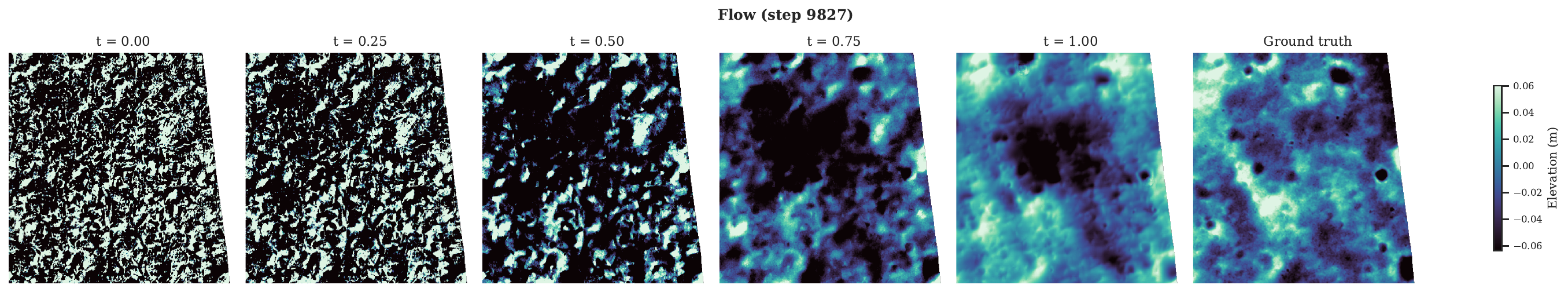}
\caption{\textbf{Decoded flow evolution, step 9827.} Five states at $t=0,0.25,0.5,0.75,1$ followed by the reference DTM. The changing appearance illustrates the image-to-relief transition, not a direct measurement of latent linearity. Exporter: \texttt{plot\_flow\_evolution}.}
\label{fig:flow_evolution}
\end{figure}

\paragraph{Compute--accuracy frontier.} Figure~\ref{fig:pareto} plots the step sweep on a linear horizontal axis and shows RMSE only. One step gives the lowest reported value, and the range across all counts is approximately 2.35\% of that value. The figure therefore supports low-step inference under the sweep conditions. It does not show a secondary $\delta_1$ axis or a need for distillation. NFE equals the step count for one Euler realization; an ensemble multiplies this cost. The primary sweep uses one realization and requires one velocity evaluation per Euler step. The supporting 16-step export therefore requires 16 evaluations per image. The shaded band's definition must be recovered before interpreting it as uncertainty. Exporters: \texttt{plot\_timestep\_ablation} / \texttt{plot\_pareto\_frontier}.

\begin{figure}[htbp]
\centering
\includegraphics[width=\linewidth,height=0.80\textheight,keepaspectratio]{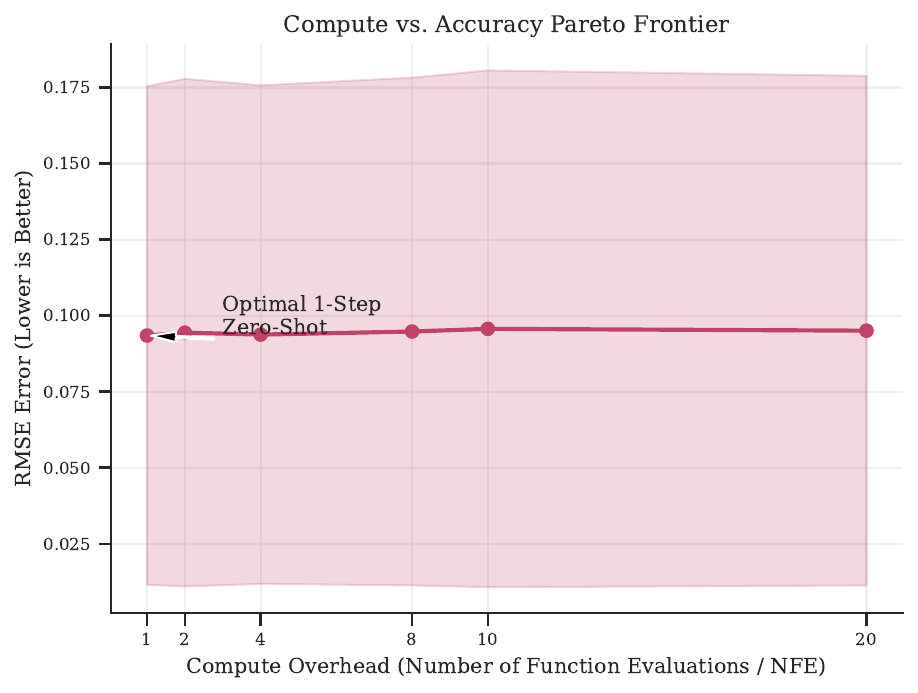}
\caption{\textbf{Archived step-count versus RMSE study.} The horizontal scale is linear, and one step has the lowest RMSE. The original ``zero-shot'' annotation is inconsistent with the reported fine-tuned model and must not be read as a zero-shot experiment. The figure is preserved pending regeneration with corrected terminology, NFE convention, and band definition. Exporter: \texttt{plot\_timestep\_ablation}.}
\label{fig:pareto}
\end{figure}

\FloatBarrier
\subsection{Training Dynamics}
\label{sec:vis_training}

\paragraph{Validation trajectories.} Figure~\ref{fig:convergence} shows three validation metrics over training: RMSE, photo-consistency, and $\delta_1$. It does not show the individual training losses. At the marked step 9476, the displayed values are approximately 0.0737, 0.191, and 21.6\%, respectively. These validation values differ from the final table and should retain their own split/protocol identity. The single-run trajectories show improvement in the monitored quantities, but cannot establish which auxiliary term caused it or provide across-run confidence. The intended logging interface includes \texttt{self.log}, \texttt{val/rmse\_mean}, and \texttt{val/photo\_consistency\_mean}; individual \texttt{train/loss\_*} traces would be needed for a component-level analysis. Exporter: \texttt{plot\_convergence\_curves}.

\begin{figure}[htbp]
\centering
\includegraphics[width=\linewidth,height=0.80\textheight,keepaspectratio]{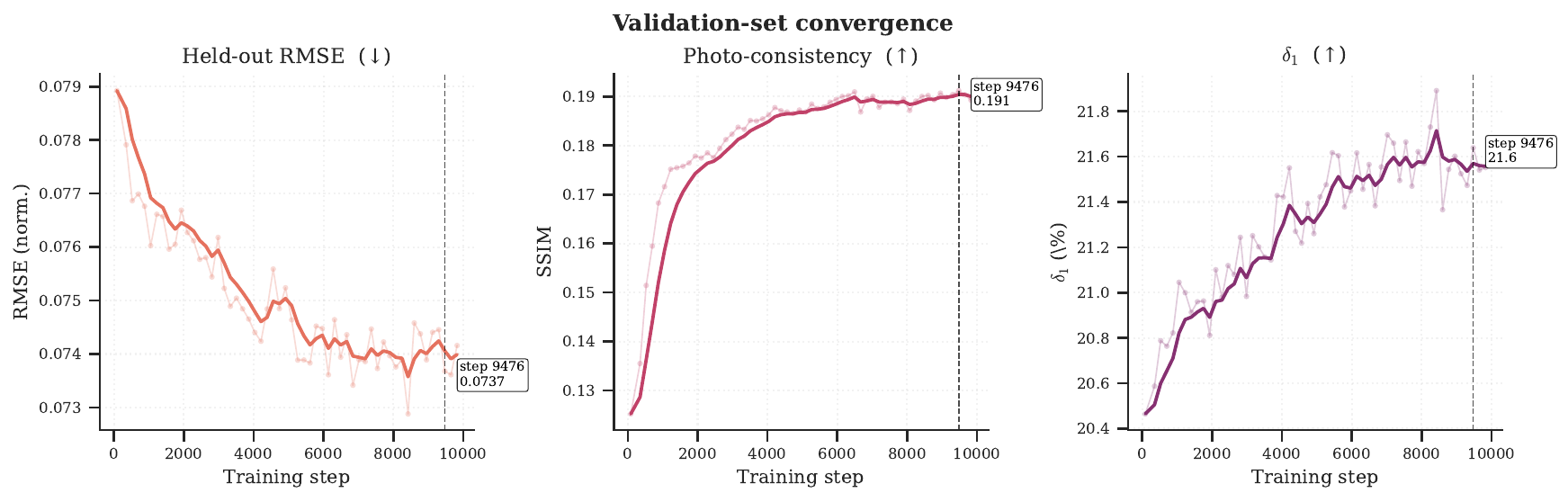}
\caption{\textbf{Validation metrics during one training run.} RMSE, photo-consistency, and $\delta_1$ trajectories, with checkpoint 9476 marked. These are validation metrics, not eight component-loss curves. Their improvement does not isolate a curriculum or individual objective contribution. Exporter: \texttt{plot\_convergence\_curves}.}
\label{fig:convergence}
\end{figure}

\paragraph{Learnable Lunar--Lambert weight.} The blend $L=\sigma(\ell_{LL})$ starts at 0.5 and remains close to that value, reaching approximately 0.490 at the marked checkpoint in Figure~\ref{fig:ll_weight}. The plot shows a small learned adjustment, not a drift to 0.3--0.4 or a measured regolith property. Although empirical mixture models are motivated by planetary photometry~\cite{McEwen1991PhotometricApplications,mcewen1996reflectance,Hapke2012TheorySpectroscopy}, a shared learned coefficient also reflects geometry errors, normalization, and the chosen objective. Its trajectory alone cannot establish loss identifiability or good conditioning. Logged as \texttt{train/lunar\_lambert\_weight}.

\begin{figure}[htbp]
\centering
\includegraphics[width=0.85\linewidth,height=0.80\textheight,keepaspectratio]{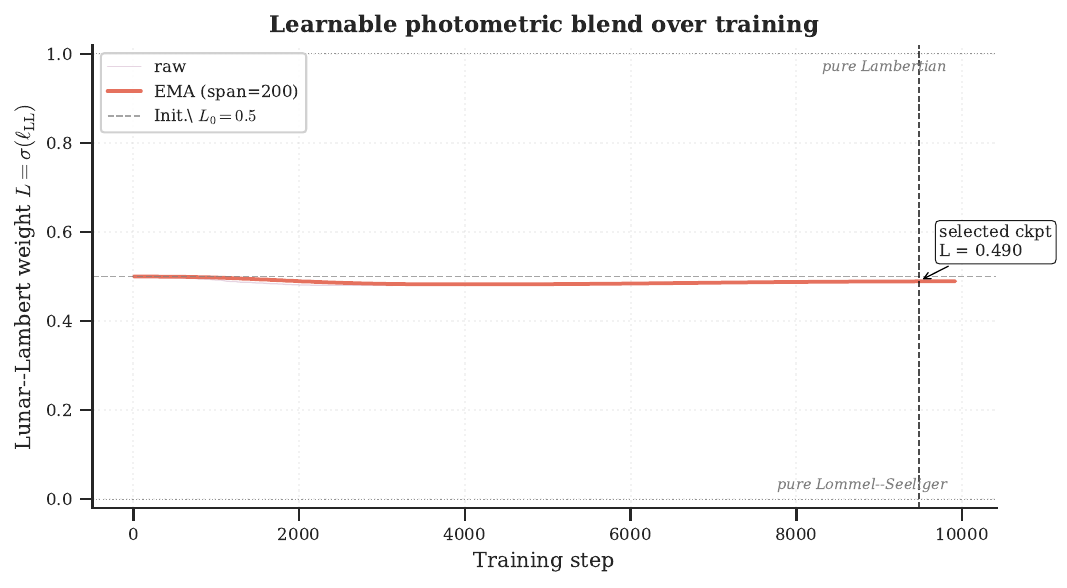}
\caption{\textbf{Learned blend trajectory.} $L$ remains near its 0.5 initialization and is approximately 0.490 at the selected checkpoint. The trace shows modest optimization change; it does not establish a physical reflectance measurement or a global conditioning guarantee.}
\label{fig:ll_weight}
\end{figure}

\FloatBarrier
\subsection{Sampling Variability and Uncertainty}
\label{sec:vis_uncertainty}

Noise-augmented inference produces multiple predictions from fixed network weights. Figure~\ref{fig:uncertainty} is an archived eight-realization example at step 9827, distinct from the single-realization main evaluation. Its variance is elevated in broad border regions and lower in much of the center; it is not confined to shadow boundaries. This measures sampling variability under the chosen source perturbations and averaging convention. It is not automatically epistemic uncertainty or calibrated error probability. Calibration requires direct error--variance analysis or coverage tests on held-out reference data. Exporter: \texttt{plot\_uncertainty\_map}.

\begin{figure}[htbp]
\centering
\includegraphics[width=\linewidth,height=0.80\textheight,keepaspectratio]{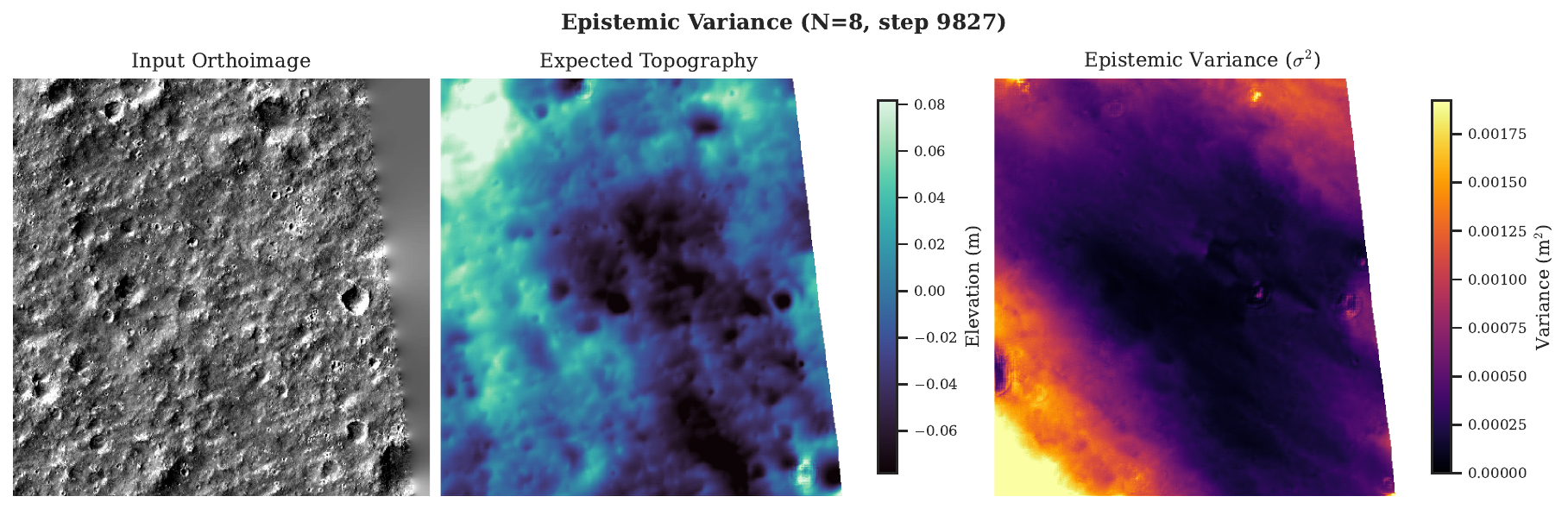}
\caption{\textbf{Sampling variability, $N=8$, step 9827.} Input orthoimage, ensemble-mean relief, and pixelwise variance. The original title uses ``epistemic'' terminology, but fixed-weight sample variance alone does not establish that interpretation or calibration. The recorded normalized-relief pipeline does not justify the displayed physical variance units. Exporter: \texttt{plot\_uncertainty\_map}.}
\label{fig:uncertainty}
\end{figure}

\FloatBarrier
\subsection{Summary}
\label{sec:vis_summary}

The preserved figure suite supports four complementary observations:

\begin{itemize}
\item \textbf{Relief correspondence is partial} (Sections~\ref{sec:vis_qualitative} and~\ref{sec:vis_geometry}): broad structure is present, while profiles, normals, and scatter expose smoothing and amplitude mismatch.
\item \textbf{Shading and spectra reveal additional behavior} (Sections~\ref{sec:vis_frequency} and~\ref{sec:vis_photo}): neither a similar spectral slope nor an image-like render independently proves recovered fine-scale terrain.
\item \textbf{The sweep supports low-step inference} (Section~\ref{sec:vis_flow}): endpoint RMSE changes little from one to twenty steps, with one step lowest in the archived comparison.
\item \textbf{The diagnostics identify validation priorities} (Sections~\ref{sec:vis_distribution} and~\ref{sec:vis_uncertainty}): patch variability, unmatched edges, and ensemble spread motivate grouped evaluation, failure analysis, and uncertainty calibration.

\end{itemize}

Together with Table~\ref{tab:metrics_main}, these panels define the evidence available for MarsFM's current behavior and retain examples that challenge a uniformly positive interpretation. Their value is reproducibility and diagnosis: they show what must improve, what requires corrected units or provenance, and which future experiments could test the proposed benefit of shading regularization.

\end{document}